\documentclass{article}

\usepackage{selectp}

    \usepackage[final]{neurips_2021}

\usepackage[utf8]{inputenc} % allow utf-8 input
\usepackage[T1]{fontenc}    % use 8-bit T1 fonts
\usepackage{hyperref}       % hyperlinks
\usepackage{url}            % simple URL typesetting
\usepackage{booktabs}       % professional-quality tables
\usepackage{amsfonts}       % blackboard math symbols
\usepackage{nicefrac}       % compact symbols for 1/2, etc.
\usepackage{microtype}      % microtypography
\usepackage{color}         % colors

\hypersetup{colorlinks=true, linkcolor=blue, breaklinks=true, urlcolor=blue}

\usepackage{amsmath}
\usepackage{amssymb}
\usepackage{amsthm}
\usepackage[svgnames]{xcolor}
\usepackage{amsfonts}
\usepackage{footnote}
\usepackage{indentfirst}
\usepackage{enumerate}
\usepackage{bbm}
\usepackage{graphicx}
\usepackage{array}
\usepackage{adjustbox}
\usepackage{doi,color}
\usepackage{array,colortbl}
\usepackage{multicol,gensymb,multirow}
\usepackage{threeparttable,sidecap,xspace}
\usepackage{color, colortbl}
\usepackage[font=small]{caption}
\usepackage{subfigure}
\usepackage{stmaryrd}
\usepackage{ulem,here}

\newcommand*{\mydoi}[1]{\href{http://dx.doi.org/#1}{\includegraphics[width=.75em]{doi.png}}}
\newcommand{\KM}{Krasnosel’skii–Mann\xspace}

\newcommand{\suchthat}{\;\ifnum\currentgrouptype=16 \middle\fi|\;}
\newcommand{\until}[1]{\{1,\dots, #1\}}
\newcommand{\subscr}[2]{#1_{\textup{#2}}}

\newcommand{\setdef}[2]{\{#1 \; | \; #2\}}

\newcommand{\map}[3]{#1: #2 \rightarrow #3}

\newcommand{\real}{\mathbb{R}}
\newcommand{\realpositive}{\mathbb{R}_{>0}}
\newcommand{\realnonnegative}{\mathbb{R}_{\geq0}}

\newcommand{\diag}{\mathrm{diag}}

\newcommand{\scirc}{\raise1pt\hbox{$\,\scriptstyle\circ\,$}}

\newcommand\oprocendsymbol{\hbox{$\square$}}
\newcommand\oprocend{\relax\ifmmode\else\unskip\hfill\fi\oprocendsymbol}

\definecolor{Gray}{gray}{0.9}

\newcommand{\bigO}{\mathcal{O}}

\newcommand{\diagL}{\operatorname{diagL}}
\newcommand{\myclearpage}{\clearpage}
\renewcommand{\myclearpage}{}

\newcommand{\Iinfty}{I_{\infty}}

\newcommand{\sign}{\operatorname{sign}}

\DeclareSymbolFont{bbold}{U}{bbold}{m}{n}
\DeclareSymbolFontAlphabet{\mathbbold}{bbold}
\newcommand{\vect}[1]{\mathbbold{#1}}
\newcommand{\vectorones}[1][]{\vect{1}_{#1}}
\newcommand{\vectorzeros}[1][]{\vect{0}_{#1}}

\newcommand{\ds}{\displaystyle}

\newtheorem{theorem}{Theorem}
\newtheorem{lemma}[theorem]{Lemma}
\newtheorem{corollary}[theorem]{Corollary}

\newtheorem{example}[theorem]{Example}

\graphicspath{{./images/}{./fig}}

\newcommand{\jac}[1]{D\mkern-0.75mu{#1}}

\newcommand{\WP}[2]{\left\llbracket{#1}, {#2}\right\rrbracket}

\newcommand{\seminorm}[1]{{\left\vert\kern-0.25ex\left\vert\kern-0.25ex\left\vert #1
		\right\vert\kern-0.25ex\right\vert\kern-0.25ex\right\vert}}

\newcommand{\semimeasure}[1]{\mu_{\seminorm{\cdot}}\kern-0.5ex\left(#1\right)}

\newcommand{\osL}{\operatorname{osL}}
\newcommand{\Lip}{\operatorname{Lip}}
\newcommand{\realpart}{\operatorname{{Re}}}

\renewcommand{\realpart}{\Re}

\newcommand{\spectrum}{\operatorname{spec}}

\newcommand{\norm}[2]{\|#1\|_{#2}}

\DeclareMathOperator{\proj}{Proj}

\DeclareSymbolFont{bbold}{U}{bbold}{m}{n}
\DeclareSymbolFontAlphabet{\mathbbold}{bbold}

\newcommand{\OF}{\mathsf{F}}
\newcommand{\OG}{\mathsf{G}}
\newcommand{\OH}{\mathsf{H}}
\newcommand{\OI}{\mathsf{I}}
\newcommand{\ON}{\mathsf{N}}

\newcommand{\mcX}{\mathcal{X}}
\newcommand{\mcU}{\mathcal{U}}
\newcommand{\mcY}{\mathcal{Y}}

\title{Implicit Neural Networks via Contraction Theory}
\title{Robust Implicit Networks \\ via Non-Euclidean Contractions}

\author{Saber Jafarpour$^{1,*}$, Alexander Davydov$^{1,}$\thanks{These
  authors contributed equally.}$\;\;$, Anton V. Proskurnikov$^{2,3}$, and Francesco Bullo$^1$ \\ \\
$^1$ Center for Control,
Dynamical Systems and Computation, University of
California, \\ Santa Barbara, 93106-5070, USA, \texttt{\{saber, davydov, bullo\}@ucsb.edu}. \\
$^2$ Department of Electronics and
Telecommunications, Politecnico di Torino, Turin, Italy;\\
$^3$ Institute for Problems in Mechanical Engineering, Russian Academy of Sciences,\\
St. Petersburg, Russia, \texttt{anton.p.1982@ieee.org}
  }

\begin{document}
\maketitle

\begin{abstract}
  Implicit neural networks, a.k.a., deep equilibrium networks, are a
  class of implicit-depth learning models where function evaluation is
  performed by solving a fixed point equation.
  They generalize classic feedforward models and are equivalent to
  infinite-depth weight-tied feedforward networks.
  While implicit models show improved accuracy and significant
  reduction in memory consumption, they can suffer from ill-posedness
  and convergence instability.

  This paper provides a new framework, {\color{black}which we call Non-Euclidean
  Monotone Operator Network (NEMON),} to design well-posed and robust
  implicit neural networks based upon contraction theory for the
  non-Euclidean norm $\ell_\infty$.
  Our framework includes (i) a novel condition for well-posedness
  based on one-sided Lipschitz constants, (ii) an average iteration
  for computing fixed-points, and (iii) explicit estimates on
  input-output Lipschitz constants.
  Additionally, we design a training problem with the well-posedness
  condition and the average iteration as constraints and, to achieve
  robust models, with the input-output Lipschitz constant as a
  regularizer.
  Our $\ell_\infty$ well-posedness condition leads to a larger
  polytopic training search space than existing conditions and our
  average iteration enjoys accelerated convergence.
  {\color{black}Finally, we evaluate our framework in
  image classification through the MNIST and the CIFAR-10 datasets.} Our numerical
  results demonstrate improved accuracy and robustness of the implicit
  models with smaller input-output Lipschitz bounds. Code is available
  at \url{https://github.com/davydovalexander/Non-Euclidean_Mon_Op_Net}.
\end{abstract}

\section{Introduction}\label{sec:introduction}

% paragraph about inn in general

Implicit neural networks are infinite-depth learning models with
layers defined implicitly through a fixed-point equation. Examples of
implicit neural networks include deep equilibrium
models~\citep{SB-JZK-VK:19} and implicit deep learning
models~\citep{LEG-FG-BT-AA-AYT:21}. % \sashamargin{What is the difference between these two? I thought DEQ was a framework and then El Ghaoui and others proposed a sufficient condition for the well-posedness of DEQ.}
Implicit networks can be considered as generalizations of feedforward
neural networks with input-injected weight tying, i.e., training
parameters are transferable between layers. Indeed, in implicit networks, function
evaluation is executed by solving a fixed-point equation and
backpropagation is implemented by computing gradients using
implicit differentiation. Due to these unique features, implicit
models enjoy more flexibility and improved memory efficiency compared to
traditional neural networks. 
At the same time, implicit networks can suffer from instability in their
training due to the nonlinear nature of their fixed-point
equations and can show brittle input-output behaviors due to their model flexibility.

% These
% instabilities can be categorized as instabilities arising from solvability of the fixed-point
% equations and instability in input-output behavior of these networks.

%Other references on implicit deep learning and implicit neural networks include
% \citep{LEG-FG-BT-AA-AYT:19,FG-HC-WZ-SS-LEG:20,SB-JZK-VK:19}.
% Theory of monotone system is used to study convergence of implicit
% methods in~\citep{EW-JZK:20}

% paragraph about solvability and stability of convergence of inn

It is known that implicit neural networks require careful tuning and
initialization to avoid ill-posed training procedures. Indeed, without
additional assumptions, their fixed-point equation may not have a
unique solution and the numerical algorithms for finding their
solutions might not converge. Several recent works in the literature
have focused on studying well-posedness and convergence of the
fixed-point equations of implicit networks using frameworks such as
monotone operator theory~\citep{EW-JZK:20}, contraction
theory~\citep{LEG-FG-BT-AA-AYT:21}, and a mixture of
both~\citep{MR-RW-IRM:20}. Despite several insightful results,
important questions about conditions for well-posedness of implicit
networks and efficient algorithms that converge to their solutions are
still
open. % To avoid fixed-point instability in the training process and to improve the
% expressiveness of the implicit neural networks, it is important to
% establish generally-applicable explicit conditions for their well-posedness and
% develop efficient algorithms that converge to their fixed-points.

% \apmargin{1) To me, the red sentence sounds too aggressive. We could say that the tightness of their estimates remained an open problem and, as we show, their estimates can be improved, considering also a broader class of activation functions. 2) See also my previous comment about sharpness.}
% \textcolor{red}{However, these frameworks either
% provide conservative estimates for well-posedness of the fixed-point
% equation or are only applicable to networks with a special class of
% activation functions}.

% One of the main tasks to understand the capacity and expressiveness of
% implicit neural networks is to study solvability and find
% computationally efficient to solve their associated fixed-point
% equation.

% paragraph about robustness of inn

% It is well-known that there is a tradeoff betwreen the expressiveness
% of neural networks and their robustness.

One of the key features of implicit neural networks is their
flexibility, which might come at the cost of low input-output
robustness. As first noted in~\citep{CZ-WZ-IS-JB-DE-IG-RF:13}, the
input-output behavior of deep neural networks can be vulnerable to
perturbations; close enough input data can lead to completely
different outputs. This lack of robustness can lead to unreliable
performance of neural networks in safety-critical applications. Among
several notions of robustness, the Lipschitz constant of a neural
network is a coarse but rigorous measure which can be used to estimate
input-output sensitivity of the
network~\citep{CZ-WZ-IS-JB-DE-IG-RF:13}. For this reason, there has
been a growing interest in estimating the input-output Lipschitz
constant of deep neural networks with respect to the
$\ell_2$-norm~\citep{MF-AR-HH-MM-GJP:19,PLC-JCP:20}. However, it turns out
that in some applications, the input-output Lipschitz constants with
respect to non-Euclidean norms are more informative measures for
studying robustness. One such application appears in the robustness
analysis of neural networks with large-scale inputs under
widely-distributed adversarial perturbations (examples of these
adversarial perturbations can be found
in~\citep{CZ-WZ-IS-JB-DE-IG-RF:13}). For these examples, the
input-output $\ell_2$-Lipschitz constant does not provide complete
information about robustness of the network; a neural network with
small input-output $\ell_2$-Lipschitz constant can be very sensitive
to widespread entrywise-small perturbations of the input signal. On
the other hand, the input-output $\ell_{\infty}$-Lipschitz constant
provides a different metric which appears to be well-suited for the
analysis of widespread distributed perturbations. Another application
is the estimation of input signal confidence intervals from
output deviations, where the input-output
$\ell_{\infty}$-Lipschitz constant of the network provides
more scalable bounds than its $\ell_{2}$ counterpart. 

\subsection*{Related works}

\paragraph*{Implicit learning models.}
Numerous works in learning theory have shown the power of deep
learning models with implicit layers. In these learning models, the
notion of layers are replaced by a composition rule, which can be
either a fixed-point iteration or a solution to a differential
equation. Well-known frameworks for deep learning using implicit
infinite-depth layers include deep equilibrium
networks~\citep{SB-JZK-VK:19}, implicit deep
learning~\citep{LEG-FG-BT-AA-AYT:21}, and Neural
ODEs~\citep{RTQC-YR-JB-DD:18}. 
%In~\citep{BT-LEG-SM:20} a class of
%implicit learning models has been proposed that includes deep
%learning, nonlinear control, and mixed-integer programming as special
%cases.
In~\citep{AK-ZZ-VS:20}, a class of implicit recurrent neural networks 
is considered and it is demonstrated that, with this architecture, 
the models do not suffer from vanishing nor exploding gradients. 
Implicit layers have also been used to study convex
optimization problems~\citep{AA-BA-SB-SB-SD-JZK:19} and to design
control strategies~\citep{BA-IJ-JS-BB-JZK:18}. Convergence to
global minima of certain classes of implicit networks is studied in~\citep{KK:21}.

\paragraph{Well-posedness and numerical algorithms for fixed-point
  equations.}
There has been a recent interest in studying well-posedness and
numerical stability of implicit-depth learning models.
\citep{LEG-FG-BT-AA-AYT:21} proposes a sufficient spectral condition for
well-posedness and for convergence of the Picard
iterations associated with the fixed-point
equation of implicit networks. In~\citep{EW-JZK:20,MR-RW-IRM:20}, using monotone operator
theory, a suitable parametrization of the weight matrix is proposed
which guarantees the stable convergence of suitable fixed-point iterations. A recent influential survey on monotone operators is
\citep{EKR-SB:16}. A recent survey on fixed point strategies in data science is
given by~\citep{PLC-JCP:21}.

\paragraph{Robustness of learning models}
It is known that neural networks can be vulnerable to adversarial
input perturbations~\citep{CZ-WZ-IS-JB-DE-IG-RF:13}.  A large body of
literature is devoted to improve robustness of neural networks using
various defense strategies against adversarial
examples~\citep{IJG-JS-CZ:15,NP-PM-XW-SJ-AS:16}. While these
strategies are effective in many scenarios, they do not provide formal
guarantees for robustness~\citep{NC-DW:17}. {\color{black} However, there has been a
  recent interest in designing classifiers that are provably robust with respect to adversarial
  perturbations~\citep{AM-AM-LS-DT-AV:17,EW-ZK:18}.} The input-output
Lipschitz constant of a neural network is a rigorous metric for its
worst-case sensitivity with respect to input perturbations. Several
recent works have focused on estimating the Lipschitz constant and
enforcing its boundedness. For
example,~\citep{MF-AR-HH-MM-GJP:19,MF-MM-GJP:20} propose a convex
optimization framework using quadratic constraints and semidefinite
programming to obtain upper bounds on Lipschitz constants of deep
neural networks. In~\citep{PP-AK-JB-PK-FA:21}, a training algorithm is
designed to ensure boundedness of the Lipschitz constant of the neural
network via a semidefinite program. Other methods for estimating the
Lipschitz constant of deep neural networks
include~\citep{VK-AAAM-FP:20,MR-RW-IRM:21,PLC-JCP:20}. {\color{black}For
  implicit neural networks, a sensitivity-based robustness analysis is
  proposed in~\citep{LEG-FG-BT-AA-AYT:21}. Lipschitz constants of deep
  equilibrium networks have also been studied
  in~\citep{CP-EW-JZK:21,MR-RW-IRM:20} using monotone operator
  theory.}

\subsection*{Contributions}

{\color{black}In this paper, using non-Euclidean contraction theory
  with respect to the $\ell_{\infty}$-norm, we
  propose our novel framework, Non-Euclidean Monotone Operator
  Network (NEMON), to design implicit neural networks and study their
  well-posedness, stability, and robustness.} First, we develop
elements of a novel non-Euclidean monotone operator theory akin to the
frameworks in~\citep{HHB-PLC:17,EKR-SB:16}. Using the concept of
matrix measure, we introduce the essential notion of one-sided
Lipschitz constant of a map. Based upon this notion, we prove a
general fixed-point theorem with weaker requirements than classical
results on Picard and \KM iterations. For maps with one-sided
Lipschitz constant less than unity, we show that an average iteration
converges for sufficiently small step sizes
and optimize its rate of convergence. For the special case of the
weighted $\ell_{\infty}$-norm, we show that this average iteration can
be accelerated by choosing a larger step size. Additionally, we study
perturbed fixed-point equations and establish a bound on the distance
between perturbed and nominal equilibrium points as a function of
one-sided Lipschitz condition.
Second, for implicit neural networks, we use our new
fixed-point theorem to (i) establish $\ell_{\infty}$-norm conditions for
their well-posedness, (ii) design accelerated numerical algorithms for
computing their solutions, and (iii) provide upper bounds on their
input-output $\ell_{\infty}$-Lipschitz constants.
{\color{black}Third, we propose a parametrization for matrices
with appropriate bound on their one-sided Lipschitz constants and use this parametrization with the average iteration to design a training
optimization problem.}
%Additionally, by adding the input-to-output
%$\ell_{\infty}$-Lipschitz bound as a regularizer in the %training
%problem, we study the tradeoff between the accuracy and %the robustness
%of the trained neural network.
%optimization problem.
%
{\color{black}Finally, we perform several numerical experiments illustrating
improved performance of NEMON in image classification
on the MNIST and the CIFAR-10 datasets compared to the state-of-the-art models
in~\citep{LEG-FG-BT-AA-AYT:21,EW-JZK:20}.} Additionally, by adding the
input-output Lipschitz constant as regularizer in the training
problem, we observe improved robustness to some classes of adversarial
perturbations. We include all relevant proofs in Appendix~\ref{app:proof}.

\myclearpage

\section{Review material}\label{sec:reviewmaterial}

\paragraph*{Matrix measures}
Let $\|\cdot\|$ be a norm on $\real^n$ and its induced norm on
$\real^{n\times{n}}$.  The matrix measure of $A \in \real^{n \times n}$
with respect to $\norm{\cdot}{}$ is defined by $ \mu(A) := \lim_{h \to 0^+}
\frac{\|I_n + hA\| - 1}{h}$, that is, the one-sided directional derivative
of the induced norm in direction of $A$, evaluated at $I_n$. Remarkably,
the matrix measure is a tighter upper bound on the spectral abscissa of $A$
than $\|A\|$ and the set of matrices $A\in \real^{n\times n}$ satisfying $\mu(A)\le 1$ is an unbounded subset of $\real^{n\times
  n}$ strictly containing the compact ball $\|A\|\le 1$. We refer
to~\citep{CAD-HH:72} for a list of properties enjoyed by matrix measures.

We will be specifically interested in diagonally weighted $\ell_{\infty}$
norms defined by
\begin{gather}
  \norm{x}{\infty,[\eta]^{-1}} = \max_{i}  \frac{1}{\eta_i}|x_i|,
\end{gather}
where, given a positive vector $\eta\in\realpositive^n$, we use $[\eta]$ to
denote the diagonal matrix with diagonal entries $\eta$.  The corresponding
matrix norm and measure are
\begin{align}
 \norm{A}{\infty,[\eta]^{-1}} &= \max_{i\in\until{n}} \sum_{j=1}^n \frac{\eta_j}{\eta_i}  |a_{ij}|,
  \quad
  \mu_{\infty,[\eta]^{-1}}(A) = \max_{i\in\until{n}} \Big( a_{ii} + \sum_{j=1,j\neq i}^n
  |a_{ij}| \frac{\eta_j}{\eta_i}\Big). \label{eq:inftyMatrixMeasure-short}
\end{align}

\paragraph{Lipschitz maps}
Given a norm $\norm{\cdot}{}$ with 
induced matrix measure $\mu(\cdot)$, a differentiable map $\map{\OF}{\real^n}{\real^n}$ is
Lipschitz continuous with constant $\Lip(\OF)\in\realnonnegative$ if
\begin{equation}
   \norm{\jac{\OF}(x)}{} \leq \Lip(\OF) \qquad \text{for all } x\in\real^n.
\end{equation}
For example, for an affine $\OF(x)=Ax+b$, the
(smallest) Lipschitz constant is $\Lip(\OF)=\norm{A}{}$.

\paragraph{One-sided Lipschitz maps}
Given a norm $\norm{\cdot}{}$, a differentiable map $\map{\OF}{\real^n}{\real^n}$ is one-sided Lipschitz
continuous with constant $\osL(\OF)\in\real$ if
\begin{equation}\label{eq:osL=muD}
  \mu(\jac{\OF}(x)) \leq \osL(\OF)  \qquad \text{for all } x\in\real^n.
\end{equation}
For example, for an affine $\OF(x)=Ax+b$, the (smallest)
one-sided Lipschitz constant is $\osL(\OF)=\mu(A)$.  Note that (i) the one-sided
Lipschitz constant is upper bounded by the Lipschitz constant, (ii) a
Lipschitz continuous map is always one-sided Lipschitz continuous, and (iii) the one-sided
Lipschitz constant may be negative. For a more in-depth review we
refer to
Appendix~\ref{app:reviewmaterial}. The notion of one-sided Lipschitz continuity unifies several important
concepts in dynamical systems and optimization theory. In operator theory, the map $\OF$ is called a monotone operator if it is one-sided
Lipschitz continuous with respect to the $\ell_2$-norm with the constant
$-\osL(-\OF)>0$~\citep{EKR-SB:16,HHB-PLC:17}. In 
control theory, the vector field $\OF$ is called strongly infinitesimally
contracting if it is one-sided Lipschitz continuous with the constant
$\osL(\OF)<0$~\citep{CAD-HH:72,WL-JJES:98,AP-AP-NVDW-HN:04}. In what follows, we let
$\osL_{\infty,[\eta]^{-1}}(\OF)\in\real$ denote the one-sided
Lipschitz constants with respect to the weighted $\ell_\infty$-norm.

\myclearpage

\section{Fixed-point equations and one-sided Lipschitz constants}\label{sec:fixed-point}

In this section, we show that the notion of one-sided Lipschitz
constant can be used to study solvability of fixed-point
equation:
\begin{align}\label{eq:fixed-point}
  x = \OF(x) ,
\end{align}
where $\map{\OF}{\real^n}{\real^n}$ is a differentiable map. Let
$\|\cdot\|$ be a norm on $\real^n$, then in view of the Banach
fixed-point theorem, a simple sufficient condition for existence of a
unique solution for the fixed-point equation~\eqref{eq:fixed-point} is
$\Lip(\OF) < 1$. We note that the sufficient condition
$\Lip(\OF)<1$ depends on the specific form of the fixed-point
equation~\eqref{eq:fixed-point} and can be relaxed by a suitable
rewriting of this fixed-point equation. Given an averaging parameter $\alpha\in (0,1]$ we define the
  \emph{average map} $\map{\OF_\alpha}{\real^n}{\real^n}$ by $\OF_{\alpha}
  := (1-\alpha)\OI + \alpha \OF$, where $\OI$ is the identity map. Using
  this notion, an equivalent reformulation of the fixed-point
  equation~\eqref{eq:fixed-point} is:
  \begin{align}\label{eq:fixed-point-parameter}
    x= (1-\alpha)x + \alpha \OF(x) = \OF_{\alpha}(x).
  \end{align}
  For $\alpha=1$, we have $\OF_{\alpha}(x) = \OF(x)$ and
  equation~\eqref{eq:fixed-point-parameter} coincides
  with equation~\eqref{eq:fixed-point}. For every
  $\alpha\in (0,1)$, the map $\OF_{\alpha}$ is different from $\OF$ but equations~\eqref{eq:fixed-point}
  and~\eqref{eq:fixed-point-parameter} are equivalent. Hence,
  if $\Lip(\OF_{\alpha})<1$, then by the Banach fixed-point theorem, the
  fixed point equation~\eqref{eq:fixed-point-parameter} (and therefore
  the fixed point equation~\eqref{eq:fixed-point}) has a unique solution $x^*$
  and the sequence $\{y_k\}_{k=1}^{\infty}$ defined by
  \begin{align}\label{eq:iteration}
    y_{k+1} = (1-\alpha)y_k + \alpha \OF(y_k),\qquad\mbox{ for all } k\in \mathbb{Z}_{\ge 0}
    \end{align}
    converges geometrically to $x^*$ with rate $\Lip(\OF_{\alpha})$. As
    a result of the parametrization~\eqref{eq:fixed-point-parameter},
    the condition $\Lip(\OF) < 1$ for existence and uniqueness of the fixed-point can be relaxed to 
    sufficient conditions
    \begin{align}\label{eq:family-condition}
      \Lip(\OF_{\alpha})<1,
    \end{align}
    parametrized by $\alpha\in (0,1]$. Additionally, if
    condition~\eqref{eq:family-condition} is satisfied, then
    algorithm~\eqref{eq:iteration} computes the fixed point $x^*$. It can be
      shown that the 
      condition~\eqref{eq:family-condition}
      becomes less conservative as $\alpha$ decreases. The
      next theorem shows that in the limit as $\alpha\to 0^{+}$, condition~\eqref{eq:family-condition}
      approaches the condition $\osL(\OF)<1$.

\myclearpage

\begin{theorem}[Fixed points via one-sided Lipschitz conditions]
  \label{thm:fixedpoint-osL}
  Let $\map{\OF}{\real^n}{\real^n}$ be differentiable and Lipschitz with
  constant $\ell>0$ with respect to a norm $\norm{\cdot}{}$.  Define the
  average map $\OF_\alpha=(1-\alpha)\OI+\alpha \OF$ and, for $c>0$, the
  function $\map{\gamma_{\ell,c}}{{]0,
    \frac{c}{(c+\ell+1)(\ell+1)}[}}{\real}$ by:
 \begin{align*}
   \gamma_{\ell,c}(\alpha) := \Big(1+\alpha c -
   \frac{\alpha^2(\ell+1)^2}{1-\alpha (\ell+1)}\Big)^{-1}.
 \end{align*}
 Then the following statements are equivalent:
 \begin{enumerate}
 \item\label{p2:mu} $\osL(\OF) < 1-c$,
 \item\label{p1:banach-fixed-suff}
   $\Lip(\OF_\alpha)=\gamma_{\ell,c}(\alpha)$, for 
   $0<\alpha<\frac{c}{(c+\ell+1)(\ell+1)}$.
 \end{enumerate}
 Moreover, if the equivalent conditions~\ref{p2:mu}
 or~\ref{p1:banach-fixed-suff} hold, then, for condition number
 $\kappa=\frac{1+\Lip(\OF)}{1-\osL(\OF)}$,
 \begin{enumerate}\setcounter{enumi}{2}
 \item\label{p3:unique-eq} $\OF$ has a unique fixed point $x^*$;
 \item\label{p4:iteration-converge} for
   $\ds 0<\alpha<\tfrac{1}{\kappa(\kappa+1)}$, $\OF_\alpha$ is a contraction
   mapping with contraction factor $\gamma_{\ell,c}(\alpha)<1$;
 \item\label{p5:optimalrate} the \emph{minimum} contraction factor $\gamma^*_{\ell,c}=1 - \frac{1}{4\kappa^2} +
     \frac{1}{8\kappa^3} + \bigO\left(\frac{1}{\kappa^4}\right)$ and
     the minimizing averaging parameter $\alpha^*$ of $\OF_\alpha$ is
   \begin{align*}
     \alpha^* = \frac{\kappa}{1-\osL(\OF)}\left(1-\frac{1}{\sqrt{1+1/\kappa}}\right)
     = \frac{1}{1-\osL(\OF)}\Big(\frac{1}{2\kappa^2} - \frac{3}{8\kappa^3} + \bigO\left(\frac{1}{\kappa^4}\right)\Big).
   \end{align*}
 \end{enumerate}
\end{theorem}
The average iteration~\eqref{eq:fixed-point-parameter} is often referred to as the \KM
iteration or the damped iteration~\citep{HHB-PLC:17}. Compared to~\cite[Theorem
5.15]{HHB-PLC:17},
Theorem~\ref{thm:fixedpoint-osL}\ref{p4:iteration-converge} studies
convergence of the \KM iteration for arbitrary norms, proposes a
weaker convergence condition of the form $\osL(\OF)<1$ (hence, $\OF$
need not be non-expansive). However, it ensures convergence for only
sufficiently small $\alpha>0$ and assumes that $\OF$ is differentiable
(as will be shown, however, the latter assumption can be relaxed).

\subsection{Accelerated convergence for weighted % $\ell_1$ and 
  $\ell_\infty$ norms}

For diagonally weighted % $\ell_1$ and 
$\ell_\infty$ norms, one can
strengthen Theorem~\ref{thm:fixedpoint-osL}\ref{p4:iteration-converge}
to prove the convergence of the average
iteration~\eqref{eq:fixed-point-parameter} on a larger domain of the
parameter $\alpha$.

\begin{theorem}[Accelerated fixed point algorithm for % $\ell_1$ and
  $\ell_\infty$ norms]
  \label{thm:acceleration-1-inf}
  Let $\map{\OF}{\real^n}{\real^n}$ be differentiable and Lipschitz with
  respect to the weighted non-Euclidean norm % $\norm{\cdot}{1,[\eta]}$ or
  $\norm{\cdot}{\infty,[\eta]^{-1}}$.  Define the average map
  $\OF_\alpha=(1-\alpha)\OI+\alpha \OF$ and pick $\diagL(\OF) \in
  [-\Lip(\OF),\osL(\OF)]$ to satisfy
  \begin{equation}\label{eq:diagL}
    \diagL(\OF) \leq \min_{i\in\until{n}}\inf_{x\in \real^n}\jac{\OF}_{ii}(x).
  \end{equation}
  If $\osL(\OF)<1$, then $\OF$ has a unique fixed-point $x^*$ and
  \begin{enumerate}
%%  \item \label{p3:1-inf-existence} ;
  \item\label{p4:1-inf:iteration-converge} for $\ds
    0<\alpha\leq\frac{1}{1-\diagL(\OF)}$, $\OF_\alpha$ is a contraction
    mapping with the contraction factor $1-\alpha(1-\osL(\OF))<1$;

  \item\label{p5:1-inf-optimal} the minimum contraction factor and
    minimizing averaging parameter of $\OF_\alpha$ are, respectively,
    \begin{align*}
      \Lip(\OF_{\alpha^*}) &= 1-\frac{1-\osL(\OF)}{1-\diagL(\OF)}
      = 1-\frac{1}{\kappa_{\infty}},\qquad \text{for }
      \kappa_{\infty}=\frac{1-\diagL(\OF)}{1-\osL(\OF)} \leq\frac{1+\Lip(\OF)}{1-\osL(\OF)},
      \\
     \alpha^* &= \frac{1}{1-\diagL(\OF)}.
   \end{align*}

    \end{enumerate}
\end{theorem}

Note that $\diagL(\OF)$ % exists and is
% \underline{well-posed},
is well-defined because of
the Lipschitz continuity assumption. Specifically, one can show that
$\diagL(\OF)$ is the minus the minimum over $i\in\until{n}$ of the
one-sided Lipschitz constants of the maps $x_i\mapsto - \OF(x_i,x_{-i})$ at
$x_{-i}=(x_1,\ldots,x_{i-1},x_{i+1},\ldots,x_n)$ fixed.

It is instructive to compare the minimum contraction factor in the general
Theorem~\ref{thm:fixedpoint-osL} with the minimum contraction factor for
% $\ell_1$ and 
$\ell_\infty$ norms in Theorem~\ref{thm:acceleration-1-inf}
and how they depend upon the corresponding condition numbers $\kappa$ and
$\kappa_{\infty}$. We note that (i) the relevant condition number
diminishes $\kappa\geq \kappa_{\infty}$, and (ii) the minimum contraction
factor $\Lip(\OF_{\alpha^*}) = 1 - \frac{1}{4\kappa^2}+\bigO(1/\kappa^4)$
improves to $\Lip(\OF_{\alpha^*}) = 1-\frac{1}{\kappa_{\infty}}$. This
acceleration justifies the title of this section.

  \subsection{Perturbed fixed-point problems}
    In this subsection, we focus on solvability of the perturbed fixed-point equation:
    \begin{align}\label{eq:perturbed-fixedpoint}
      x = \OF(x,u),
      \end{align}
    where $\map{\OF}{\real^n\times\real^r}{\real^n}$ is differentiable
    in $x$. We define
    $\OF_u(x)=\OF(x,u)$ and $\OF_x(u)=\OF(x,u)$. Given a norm
    $\norm{\cdot}{\mcX}$ in $\real^n$ and $\norm{\cdot}{\mcU}$ in
    $\real^r$, $\OF$ is Lipschitz in its first argument with constant
      $\Lip_x(\OF)\in\realnonnegative$ if
      \begin{equation*}
        \norm{\OF(x_1,u)-\OF(x_2,u)}{\mcX}\leq \Lip_x(\OF) \norm{x_1-x_2}{\mcX} \quad \text{for all } x_1,x_2\in\real^n
        \text{ and } u\in\real^r,
      \end{equation*}
    and it is Lipschitz in its second argument with constant $\Lip_u(\OF)\in\realnonnegative$ if
      \begin{equation*}
        \norm{\OF(x,u_1)-\OF(x,u_2)}{\mcX} \leq \Lip_u(\OF) \norm{u_1-u_2}{\mcU} \quad \text{for all } x\in\real^n
        \text{ and } u_1,u_2\in\real^r,
      \end{equation*}
    and it is one-sided Lipschitz in its first argument with constant $\osL_x(\OF)\in\real$ if
          \begin{equation*}
        \mu(D_{x}\OF(x,u))\leq \osL_x(\OF)
        \quad \text{for all } x_1,x_2\in\real^n
        \text{ and } u\in\real^r.
      \end{equation*}

    The following result, which is in the spirit of Lim's
    Lemma~\citep{TCL:85}, provides an upper bound on the distance
    between fixed-points of the perturbed
    equation~\eqref{eq:perturbed-fixedpoint}.

    \begin{theorem}[Perturbed fixed-points] \label{thm:perturbed-fixed}
      Given a norm $\|\cdot\|_{\mcX}$ in $\real^n$ and a norm
      $\|\cdot\|_\mcU$ in $\real^r$, consider a map
      $\map{\OF}{\real^n\times \real^r}{\real^n}$ differentiable in the
      first argument and Lipschitz in both arguments. If $\OF$ is one-sided
      Lipschitz with constant $\osL_x(\OF)<1$, then
      \begin{enumerate}
      \item\label{p1:unique-gu} for every $u\in \real^m$, the map $\OF_u$ has
        a unique fixed point $x^*_u$;
      \item\label{p2:lim} for every $u,v\in\real^m$, $\ds
          \norm{x^*_u-x^*_v}{\mcX} \le \frac{\Lip_u(\OF)}{1-\osL_x(\OF)} \norm{u-v}{\mcU}$.
    \end{enumerate}
  \end{theorem}

  Finally,
  Theorems~\ref{thm:fixedpoint-osL},~\ref{thm:acceleration-1-inf},
  and~\ref{thm:perturbed-fixed} are not directly applicable to
  activation function that are not differentiable. In
  Appendix~\ref{sec:non-differentiable}, we show that for specific
  form of the fixed-point equation~\eqref{eq:fixed-point}, where
  $\OF = \Phi\circ \OH$ and $\map{\Phi}{\real^n}{\real^n}$ is a weakly
  increasing, non-expansive, diagonal activation function and
  $\map{\OH}{\real^n\times \real^r}{\real^n}$ is a differentiable function, all of
  the conclusions of
  Theorems~\ref{thm:fixedpoint-osL},~\ref{thm:acceleration-1-inf},
  and~\ref{thm:perturbed-fixed} hold by requiring equation~\eqref{eq:diagL} to
  be true almost everywhere.

  \myclearpage
\section{Contraction analysis of implicit neural networks}\label{sec:inn}
{\color{black} In this section, we use contraction
  theory to lay the foundation for our Non-Euclidean Monotone Operator
  Network (NEMON) model of implicit neural networks.} Given $A\in\real^{n\times{n}}$, $B\in\real^{n\times{r}}$,
$C\in\real^{q\times{n}}$, and $D\in\real^{q\times{r}}$, we consider
the implicit neural network
\begin{equation}
  \label{eq:inn}
  x=\Phi(Ax+Bu) := \ON(x,u), \qquad y = Cx + Du,
\end{equation}
where $x\in\real^n$, $u\in\real^r$, $y\in\real^{q}$, and $\map{\Phi}{\real^n}{\real^n}$
is defined by $\Phi(x) =
(\phi_1(x_1),\ldots,\phi_n(x_n))$.
For every $i\in \{1,\ldots,n\}$, we assume the activation function
$\map{\phi_i}{\real}{\real}$ is weakly increasing, i.e., 
$\phi_i(x_i)\geq\phi_i(z_i)$ for $x_i\geq z_i$, and non-expansive, i.e.,
$|\phi_i(x_i)-\phi_i(z_i)|\leq|x_i-z_i|$ for all $x_i$ and $z_i$; if
$\phi_i$ is differentiable, these conditions are equivalent to
$0\leq\phi_i'(x_i)\leq1$ for all $x_i\in\real$.

We are able to provide the following estimates on all relevant
Lipschitz constants.

\begin{theorem}[Lipschitz and one-sided Lipschitz constants for
    the implicit neural network]
  \label{thm:inn-Lip-osL}
  Consider the implicit neural network in equation~\eqref{eq:inn} with
  weakly increasing and non-expansive activation functions $\Phi$.  With
  respect to $\|\cdot\|_{\infty,[\eta]^{-1}}$, $\eta \in \realpositive^n$,
  on $\real^n$ and $\norm{\cdot}{\mcU}$ on the input space $\real^r$, the
  map $\map{\ON}{\real^n\times\real^r}{\real^n}$ is one-sided Lipschitz
  continuous in the first variable and Lipschitz continuous in both
  variables with constants:
  \begin{align}
    \osL_{x}(\ON) &= \mu_{\infty,[\eta]^{-1}}(A)_+ \;,\qquad% \label{eq:osL-PhiA} \\ 
    \Lip_{x}(\ON) = \norm{A}{\infty,[\eta]^{-1}} \;, \label{eq:Lipx-PhiA}
    \\
    \Lip_{u}(\ON) &= \norm{B}{(\infty,[\eta]^{-1}),\mcU}  \;, \qquad% \label{eq:Lipu-PhiA} \\
    \diagL(\ON) = \min\nolimits_{i\in\until{n}}(A_{ii})_- \;,
  \end{align}
  where $(z)_+= z$ if $z\geq0$ and $(z)_+= 0$ if $z<0$;
  and $(z)_-= 0$ if $z\geq0$ and $(z)_-= z$ if $z<0$.
\end{theorem}

We now use these  estimates to establish multiple
properties of the implicit neural network.

\begin{corollary}[Well posedness, input-state Lipschitz constant, and computation]
  \label{corollary:inn-properties}
 Consider the model~\eqref{eq:inn}, with parameters $(A,B,C,D)$ and with
 weakly increasing and non-expansive activation functions
 $\Phi$. Define the average map
 $\ON_\alpha:= (1-\alpha)\OI + \alpha \ON$ and consider
 the norms $\|\cdot\|_{\infty,[\eta]^{-1}}$, $\eta \in \realpositive^n$, on
 $\real^n$, $\norm{\cdot}{\mcU}$ on the input space $\real^r$ and
 $\norm{\cdot}{\mcY}$ on the output space $\real^q$. Then
 \begin{enumerate}%[nosep]
  \item\label{fact:wellposed} if $\mu_{\infty,[\eta]^{-1}}(A)<1$,
    then~\eqref{eq:inn} is well posed, i.e., there exists a unique fixed
    point,

  \item the map $\ON_\alpha$ is a contraction mapping for $0<\alpha \leq \alpha^* := \big(1-\min_{i\in\until{n}}(A_{ii})_-\big)^{-1}$
    with minimum contraction factor $\Lip(\ON_{\alpha^*}) =
    1-\frac{1-\mu_{\infty,[\eta]^{-1}}(A)_{+}}{1-\min_{i\in\until{n}}(A_{ii})_-}$. 

  \item the Lipschitz constants from input $u$ to fixed point $x_u^*$ and to the
    output $y=Cx_u^*+Du$ are
    \begin{align}
      \Lip_{u \to x^*} &:=  \frac{ \Lip_{u}(\ON)} {1- \osL_{x}(\ON)}  =
      \frac{\norm{B}{(\infty,[\eta]^{-1}),\mcU}} {1-\mu_{\infty,[\eta]^{-1}}(A)_+},
      \\
      \Lip_{u \to y} &:=      \frac{\norm{B}{(\infty,[\eta]^{-1}),\mcU}
        \norm{C}{\mcY,(\infty,[\eta]^{-1})} }
          {1-\mu_{\infty,[\eta]^{-1}}(A)_+} + \norm{D}{\mcY,\mcU}. \label{eq:InputOutputLipschitz}
    \end{align}
  \end{enumerate}
\end{corollary}

%In~\citep{LEG-FG-BT-AA-AYT:19} a well-posedness
%condition of the form $\lambda_{\mathrm{pf}}(|A|)<1$ is proposed, where
%$|A|$ denotes the entrywise absolute value of the matrix $A$ and
%$\lambda_{\mathrm{pf}}$ denotes the Perron-Frobenius eigenvalue. However,
%for the training procedure, this condition is relaxed to the convex
%condition $\|A\|_{\infty}<1$. %~\citep{LEG-FG-BT-AA-AYT:19}.
%It is easy to see
%that our well-posedness condition in
%Corollary~\ref{corollary:inn-properties}\ref{fact:wellposed%} is less
%conservative than the condition %$\lambda_{\mathrm{pf}}(|A|)<1$ and its
%convex relaxation of the form $\|A\|_{\infty}<1$ proposed
%in~\citep{LEG-FG-BT-AA-AYT:19}.

\myclearpage

\section{Training implicit neural networks}\label{sec:training}
\newcommand{\diagm}[1]{\subscr{#1}{d}}
\paragraph*{Problem setup}
 Given an input data matrix $U = [u_1,\dots,u_m] \in \real^{r \times m}$
 and a corresponding output data matrix $Y = [y_1,\dots, y_m] \in \real^{q
   \times m}$, we aim to learn matrices $A,B,C,D$ so that the neural
 network~\eqref{eq:inn} approximates the input-output relationship.  We
 rewrite the model for matrix inputs as $ \widehat{Y} = CX + DU$, where $X
 = \Phi(AX + BU)$.  From
 Corollary~\ref{corollary:inn-properties}\ref{fact:wellposed}, if each
 $\phi_i$ is weakly increasing and non-expansive, the fixed point problem
 is well-posed when $\mu_{\infty,[\eta]^{-1}}(A) < 1$ for some $\eta \in
 \realpositive^n$. We consider a training problem of the form
\begin{equation} \label{eq:TrainingProblem}
\begin{aligned}
\min_{A,B,C,D,X}\qquad &\mathcal{L}(Y, CX + DU) +
\mathcal{P}(A,B,C,D)\\
& X = \Phi (AX+BU), \quad \mu_{\infty,[\eta]^{-1}}(A)\leq \gamma,
\end{aligned}
\end{equation}
where $\mathcal{L}$ is a loss function, $\mathcal{P}$ is a penalty function, and
$\gamma < 1$ is a hyperparameter ensuring the fixed point problem is
well-posed. {\color{black}For $\eta=\vect{1}_n$, we can remove the constraint $\mu_{\infty}(A)\leq
\gamma$ in the training optimization
problem~\eqref{eq:TrainingProblem} using the following parametrization
of weight matrix $A$:
\begin{align}\label{eq:parametrization-weight}
  A = T - \diag(|T|\vect{1}_n) +\gamma I_n.
\end{align}
In Appendix~\ref{app:matrixmeasure}, we show that
parametrization~\eqref{eq:parametrization-weight} characterizes the
set of matrices in $\real^{n\times n}$ satisfying
$\mu_{\infty}(A)\leq \gamma$. Using the
parametrization~\eqref{eq:parametrization-weight} in the training problem not only
improves the computational efficiency of the optimization but also
allows for the design of implicit
neural networks with additional structure such as convolutions. Suppose $u\in \real^{rs^2}$ is a $r$-channel input of size $s\times s$ and $x\in \real^{ns^2}$
is an $n$-channel hidden layer. To define our implicit CNN, we select
the weight matrix $A\in \real^{ns^2\times ns^2}$ as the matrix form of
a 2D convolutional operator. If we consider a circular convolution
operator, then $A$ is a circulant matrix. Using the
parametrization~\eqref{eq:parametrization-weight}, $A$ is
circulant if and only if $T$ is
circulant. Therefore, the training problem for implicit CNNs can be
cast as an unconstrained optimization problem using the above
parametrization with a circulant $T$.}

% Based upon equation~\eqref{eq:inftyMatrixMeasure-short} the set of matrices
% with $\mu_{\infty,[\eta]^{-1}}(A)\leq\gamma$ is an unbounded convex
% polytope:
% \begin{align}
%   \mathcal{M}_{\gamma,\eta} &:=
%   \setdef{A\in\real^{n\times{n}}}{\mu_{\infty,[\eta]^{-1}}(A)\leq\gamma} \nonumber \\
%   &=
%   \setdef{A\in\real^{n\times{n}}}{ \exists T \in \real^{n \times n} \text{ s.t. } (\diagm{A}+T)\eta\leq
%     \gamma\eta, -T\leq (A-\diagm{A})\leq T}, \label{eq:matmeasure-set}
% \end{align}
% where $\diagm{A}\in\real^{n\times{n}}$ is the diagonal matrix with
% same diagonal entries as $A$.

\paragraph*{Improving robustness via Lipschitz regularization}
We now focus on learning robust implicit neural networks with bounded
Lipschitz constants via a regularization strategy. Setting both
$\|\cdot\|_{\mcU}$ and $\|\cdot\|_{\mcY}$ as $\|\cdot\|_\infty$ in the
input-output Lipschitz bound~\eqref{eq:InputOutputLipschitz}, we get
\begin{align*}
\Lip_{u \to y} &= \frac{\norm{B}{(\infty,[\eta]^{-1}),(\infty)}
	\norm{C}{(\infty),(\infty,[\eta]^{-1})} }
{1-\mu_{\infty,[\eta]^{-1}}(A)_+} + \norm{D}{\infty,\infty} \\ &\leq \frac{1}{2}\frac{\norm{B}{(\infty,[\eta]^{-1}),(\infty)}^2 + \norm{C}{(\infty),(\infty,[\eta]^{-1})}^2}
{1-\mu_{\infty,[\eta]^{-1}}(A)_+} + \norm{D}{\infty,\infty},
\end{align*}
where the inequality provides a convex upper bound for the
input-output Lipschitz constant. Therefore, using the hyperparameter $\lambda>0$, the regularized optimization problem is written as
\begin{align}
\min_{A,B,C,D,X}\qquad &\mathcal{L}(Y, CX + DU) +
\lambda\Big(\frac{1}{2}\frac{\norm{B}{(\infty,[\eta]^{-1}),(\infty)}^2 + \norm{C}{(\infty),(\infty,[\eta]^{-1})}^2}
{1-\mu_{\infty,[\eta]^{-1}}(A)_+} + \norm{D}{\infty,\infty}\Big) \nonumber\\
& X = \Phi (AX+BU), \quad \mu_{\infty,[\eta]^{-1}}(A)\leq \gamma.\label{eq:regularizedLip}
\end{align}

\paragraph*{Certified adversarial robustness via Lipschitz bounds}
{\color{black} Given a nominal input $u\in \real^{r}$, we consider any perturbed
  input $v$ within an $\ell_\infty$-ball of radius $\varepsilon$
  around $u$. In this case, we have
\begin{equation}
\|y_u - y_v\|_\infty \leq \Lip_{u \to y}\|u - v\|_\infty \leq \Lip_{u \to y}\varepsilon.
\end{equation}
Then we define $\text{margin}(u) = (y_u)_i - \max_{j \neq i} (y_u)_j$, where $(y_u)_i$ is the logit corresponding to the (correct) label $i$ for the input $u$. Then provided $L\varepsilon \leq \frac{1}{2}\text{margin}(u)$, NEMON is certifiably robust to any perturbed input $v$ within an $\ell_\infty$-ball of radius $\varepsilon$ centered at $u$.}

\paragraph*{Backpropagation of gradients via average iteration}

From~\citep{LEG-FG-BT-AA-AYT:21} we now show how the average iteration can
be used to perform backpropagation via the implicit function theorem.  For
simplicity, we assume that each activation function $\phi_i$ is
differentiable and consider mini-batches of size 1, i.e., we have $X=x\in
\real^n$, $U=u\in \real^r$ and $\widehat{Y}=\widehat{y}\in \real^q$. 
Let $x^*$ be the unique solution of the fixed-point
equation~\eqref{eq:inn}. Then the chain rule implies
\begin{align*}
  \frac{\partial \mathcal{L}}{\partial A} &= (\nabla_{x^*}\mathcal{L})
  x^{\top},\qquad
  \frac{\partial \mathcal{L}}{\partial B} = (\nabla_{x^*}\mathcal{L})
  u^{\top},\\
  \frac{\partial \mathcal{L}}{\partial C} &= (\nabla_{\widehat{y}}\mathcal{L})
  x^{\top},\qquad
  \frac{\partial \mathcal{L}}{\partial D} = (\nabla_{\widehat{y}}\mathcal{L})
  u^{\top}.
\end{align*}
Since $\mathcal{L}$ depends explicitly on $\widehat{y}$, computing
$\nabla_{\widehat{y}}\mathcal{L}$ is straightforward. Computing
$\nabla_{x^*}\mathcal{L}$ is more complicated since $X^*$ is defined
only implicitly. However, it be shown that
\begin{align*}
  \nabla_{x^*}\mathcal{L} =
  (C(I-\jac{\Phi}A)^{-1}\jac{\Phi})^{\top}\nabla_{\widehat{y}}\mathcal{L}.
\end{align*}
Since $\mu_{\infty,[\eta]^{-1}}(A)<1$, by Lemma~\ref{lemma:affine-scaling}
we get that $\mu_{\infty,[\eta]^{-1}}(\jac{\Phi} A)<1$.  This implies that
the matrix $G:=(I_n-\jac{\Phi}A)^{-1}\jac{\Phi}\in \real^{n\times n}$
exists and is the solution to the following fixed-point
equation~\cite[Section 6.2]{LEG-FG-BT-AA-AYT:21}
\begin{align}\label{eq:fixed-point-backpropagartion}
  G = \jac{\Phi}(A G + I_n).
\end{align}
Moreover, $\mu_{\infty,[\eta]^{-1}}(\jac{\Phi} A)<1$ and
Theorem~\ref{thm:acceleration-1-inf} together imply that the fixed-point
equation~\eqref{eq:fixed-point-backpropagartion} has a unique solution
$G^*$ and, for every $0<\alpha\le \alpha^* :=
\big(1-\min_i(A_{ii})_{-}\big)^{-1}$, the average iterations
\begin{align*}
  G_{k+1} = (1-\alpha)G_k + \alpha \jac{\Phi}(A G_{k} + I_n),\qquad\mbox{ for
    all }k\in \mathbb{Z}_{\ge 0}
\end{align*}
are contracting with the minimum contraction factor
$1-\alpha^*\big(1-\mu_{\infty,[\eta]^{-1}}(A)_{+}\big)$ at step size
$\alpha^*$.

\myclearpage

\section{Theoretical and numerical comparisons}\label{sec:numericalexperiment}

{\color{black}In this section, we provide a comprehensive comparison of
  our framework with the state-of-the-art implicit neural
  networks\footnote{All models were trained using Google Colab with a Tesla P100-PCIE-16GB GPU.}.

  % We
  % use NEMON as well as some other existing frameworks in the
  % literature to train
  % implicit neural networks for image classification in the
  % MNIST
  % %\footnote{MNIST is licensed by CC BY-SA 3.0. and is available
  %   %at \tt{http://yann.lecun.com/exdb/mnist/.}} 
  %   and
  % the CIFAR-10
  % %\footnote{CIFAR-10 is licensed by The MIT License and is
  %  % available at \tt{https://www.cs.toronto.edu/~kriz/cifar.html.}}
  % datasets. 
  % Furthermore, we compare the accuracy and robustness of
  % the NEMON model with the existing implicit network models in the literature.

\subsection{Implicit network models}
We start by reviewing the existing models for implicit networks in the
literature. 

\paragraph{Implicit deep learning model.}
~\citep{LEG-FG-BT-AA-AYT:21} proposes a class of implicit neural
networks with input-output behavior described
by~\eqref{eq:inn}. It is shown that a
sufficient condition for existence and uniqueness of a solution and
convergence of the Picard iterations for the fixed point equation
$x=\Phi(Ax+Bu)$ is $\lambda_{\mathrm{pf}}(|A|)<1$, where
$|A|$ denotes the entrywise absolute value of the matrix $A$ and
$\lambda_{\mathrm{pf}}$ denotes the Perron-Frobenius eigenvalue. For training, the optimization
problem~\eqref{eq:TrainingProblem} is used where the constraint
$\mu_{\infty,[\eta]^{-1}}(A)\le \gamma$ is replaced by
$\|A\|_{\infty}\le \gamma$~\cite[Equation 6.3]{LEG-FG-BT-AA-AYT:21}.\footnote{The implicit deep learning implementation is available at
  \tt{https://github.com/beeperman/idl}.} It is easy to see
that our well-posedness condition in
Corollary~\ref{corollary:inn-properties}\ref{fact:wellposed} is less
conservative than $\lambda_{\mathrm{pf}}(|A|)<1$ and its
convex relaxation $\|A\|_{\infty}<1$.

\paragraph{Monotone operator deep equilibrium network (MON).}
\citep{EW-JZK:20} proposes to use monotone operator theory to
guarantee well-posedness of the fixed-point equation as well
as its convergence to the solutions. The
input-output behavior of the network is described
by~\eqref{eq:inn}. For training, the optimization
problem~\eqref{eq:TrainingProblem} is used where the constraint
$\mu_{\infty,[\eta]^{-1}}(A)\le \gamma$ is replaced by
$I_n-\tfrac{1}{2}(A+A^{\top})\succeq (1-\gamma)I_n$. In order to
ensure that this constraint is always satisfied in the training procedure, the weight matrix $A$ is parametrized as
$A=\gamma I_n - W^{\top}W - Z + Z^{\top}$, for arbitrary
$W,Z\in \real^{n\times n}$~\cite[Appendix D]{EW-JZK:20}.\footnote{The MON implementation is
  available at
  \tt{https://github.com/locuslab/monotone\_op\_net}.} In the
context of contraction theory,
\begin{equation*}
 I_n-\tfrac{1}{2}(A+A^{\top})\succeq (1-\gamma)I_n \quad \iff \quad \mu_2(A)\le \gamma,
\end{equation*}
which is shown in Appendix~\ref{app:reviewmaterial}. Thus, the parametrization
$A=\gamma I_n - W^{\top}W - Z + Z^{\top}$ can be considered as the
$\ell_2$-norm version of
the parametrization described by equation~\eqref{eq:parametrization-weight}. In
other words, the monotone operator network formulation is a Euclidean
transcription of the framework we propose in this paper.}

% \paragraph{Lipschitz bounded Equilibrium networks (LBEN)}
% \citep{MR-RW-IRM:20} develops a novel class of deep equilibrium
% networks with a less restrictive well-posedness condition
% $2[\eta]-([\eta]A+A^{\top}[\eta])\succ 0$, for some $\eta\in
% \real^n_{>0}$, and with a guaranteed
% Lipschitz bound in the training problem. In the context of contraction
% theory, the well-posedness condition in~\citep{MR-RW-IRM:20} can be
% characterized by
% \begin{equation*}
%  2[\eta]-([\eta]A+A^{\top}[\eta])\succ 0 \quad \iff \quad \mu_{2,[\eta]}(A)< 1,
% \end{equation*}}

\subsection{MNIST experiments}
\begin{figure}[!ht]
	\begin{tabular}{cc}
		\includegraphics[width = 0.482\linewidth,clip]{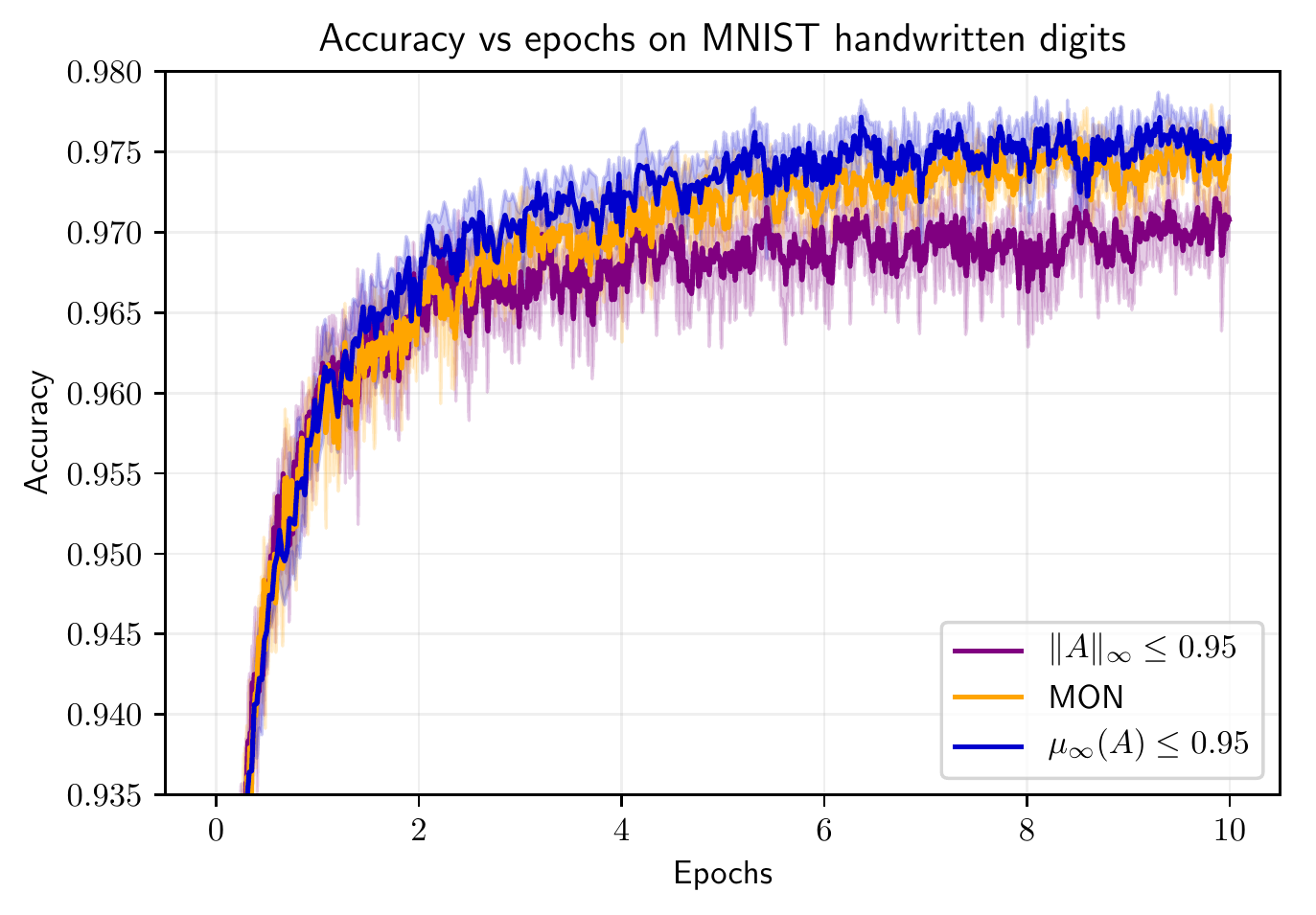}&
		\includegraphics[width = 0.482\linewidth,clip]{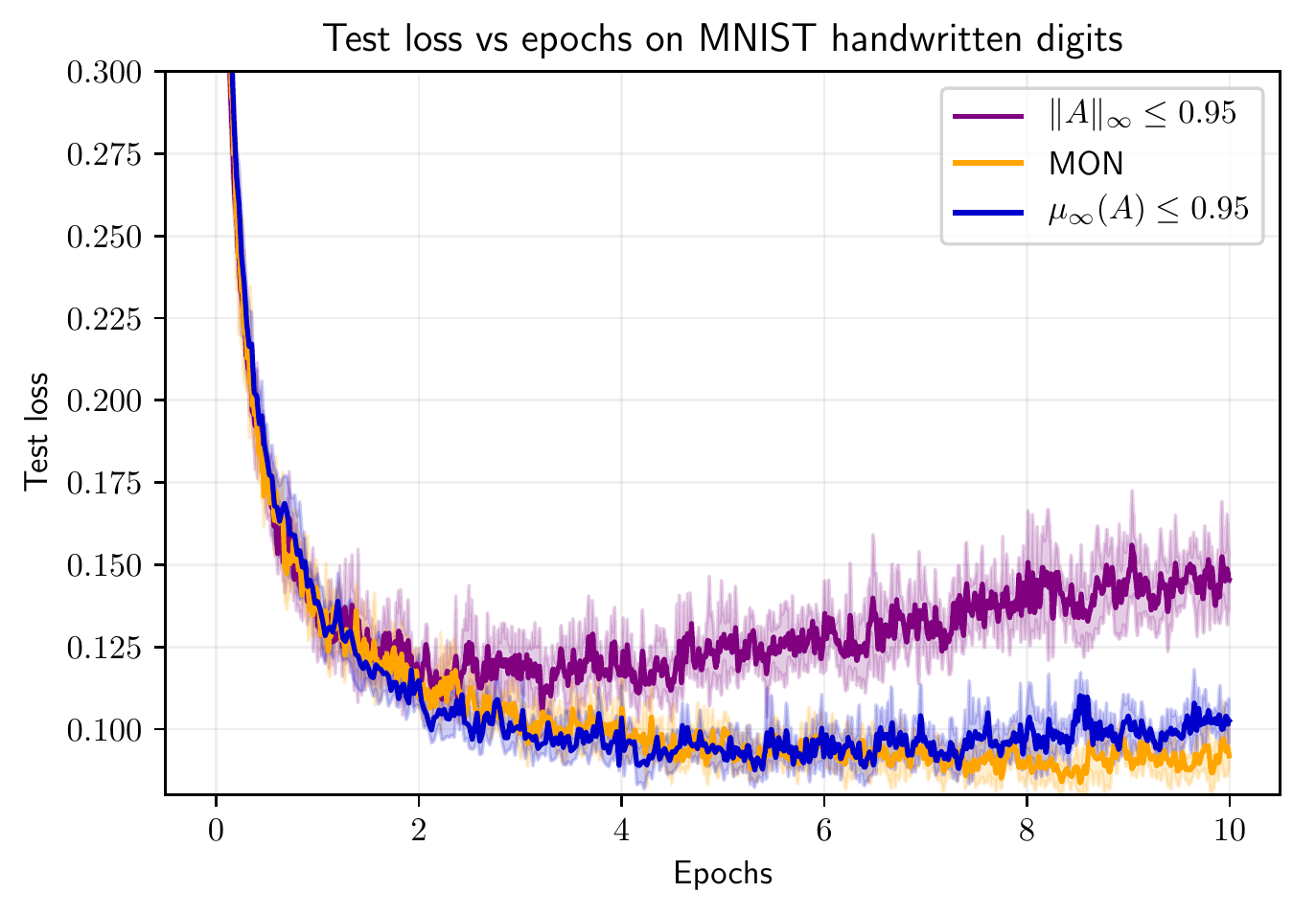}
	\end{tabular}
	\caption{Performance comparison between the NEMON model with
          $\mu_\infty(A) \leq 0.95$, the implicit deep learning model with
          $\|A\|_\infty \leq 0.95$, and MON with $I_n-\frac{1}{2}(A+A^{\top})\succeq 0.05
I_n$ on the MNIST dataset. The curves are generated by mean accuracy
and mean loss over 5 different runs while light envelopes around the
curves correspond to the standard deviation over the runs. Average
best accuracy for the NEMON model is $0.9772$, while it is $0.9721$ for
implicit deep learning model and $0.9762$ for the MON model.}
	\label{fig:MNIST}
\end{figure}

In the digit classification dataset MNIST, input data are $28 \times 28$ pixel images of handwritten digits between 0-9. There are
60000 training images and 10000 test images. For training, images are
reshaped into 784 dimensional column vectors and entries are scaled
into the range $[0,1]$. As a loss function, we use the cross-entropy. All models are of order
$n = 100$, used the ReLU activation function $\phi_i(x) = (x)_+$, and
are trained with a batch size of 300 over 10 epochs with a learning
rate of $1.5 \times 10^{-2}$. % For {\color{blue}the NEMON model, the
%   constraint $\mu_\infty(A) \leq 0.95$ is imposed. While, for
% the implicit deep learning model~\citep{LEG-FG-BT-AA-AYT:19} the constraint
% $\|A\|_\infty \leq 0.95$ is imposed. Both models are additionally
% compared to a fully connected monotone operator equilibrium network
% (MON)~\citep{EW-JZK:20} with the constraint  $I_n-\frac{1}{2}(A+A^{\top})\succeq 0.05
% I_n$}. 
Curves for accuracy and loss versus epochs for the three models are shown
in Figure~\ref{fig:MNIST}.  {\color{black} Regarding training times, using
  the average iteration, NEMON took, on average, 12 forward iterations, 13
  backward iterations, and 9.8 seconds to train per epoch. Using the
  Peaceman-Rachford iteration, MON took, on average, 17 forward iterations,
  16 backward iterations, and 9.5 seconds to train per epoch. Using the
  Picard iteration, the implicit deep learning model took, on average, $10$
  forward iterations, $5$ backward iterations, and $5.8$ seconds to train
  per epoch.}  We observe that the NEMON model performs better than the
implicit deep learning model and has a comparable performance to MON.

% \subsection{Robustness regularization experiments}
\begin{figure}[ht]\centering
	\begin{tabular}{cc}
		\includegraphics[width = 0.482\linewidth,clip]{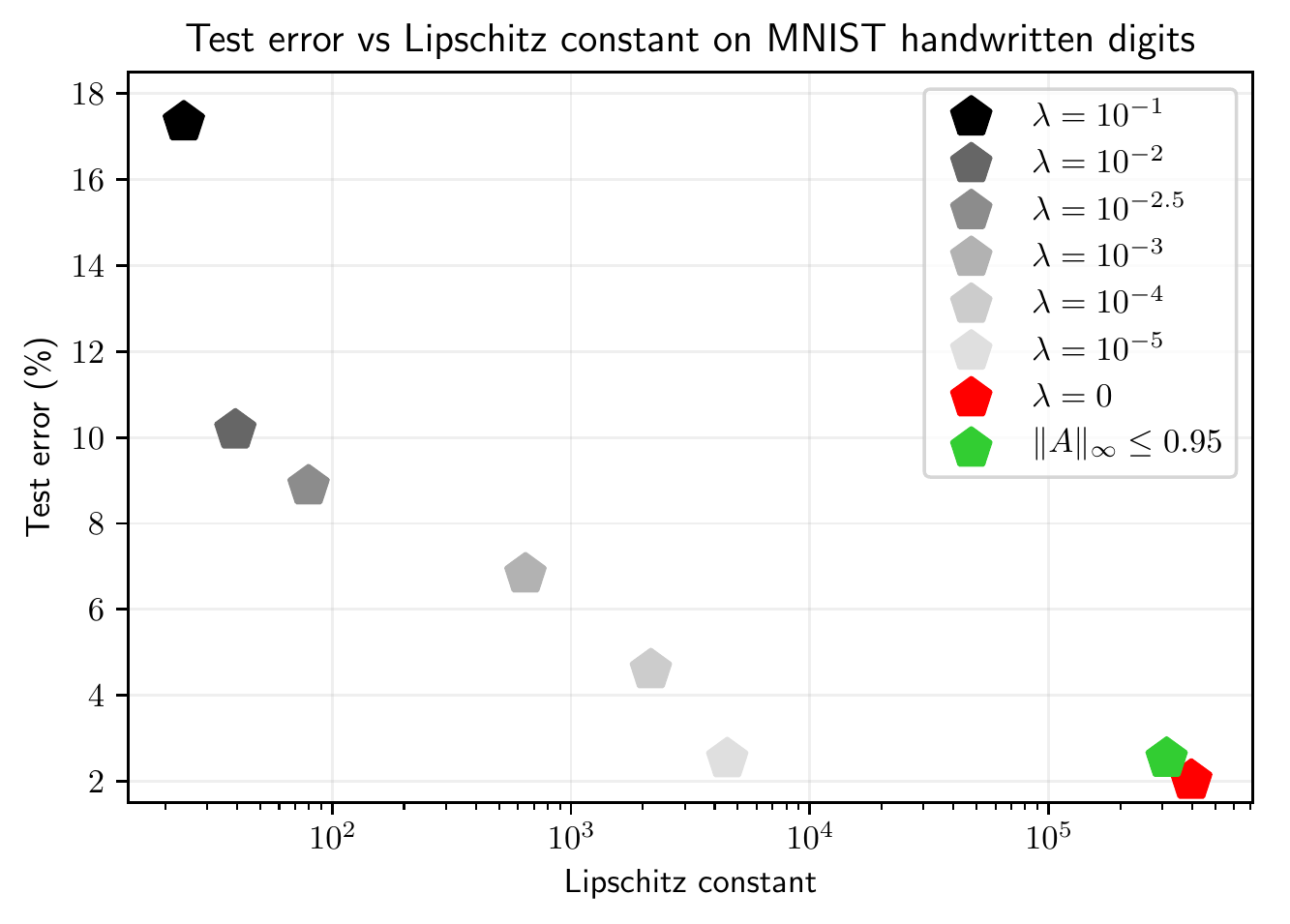}&
		\includegraphics[width = 0.482\linewidth,clip]{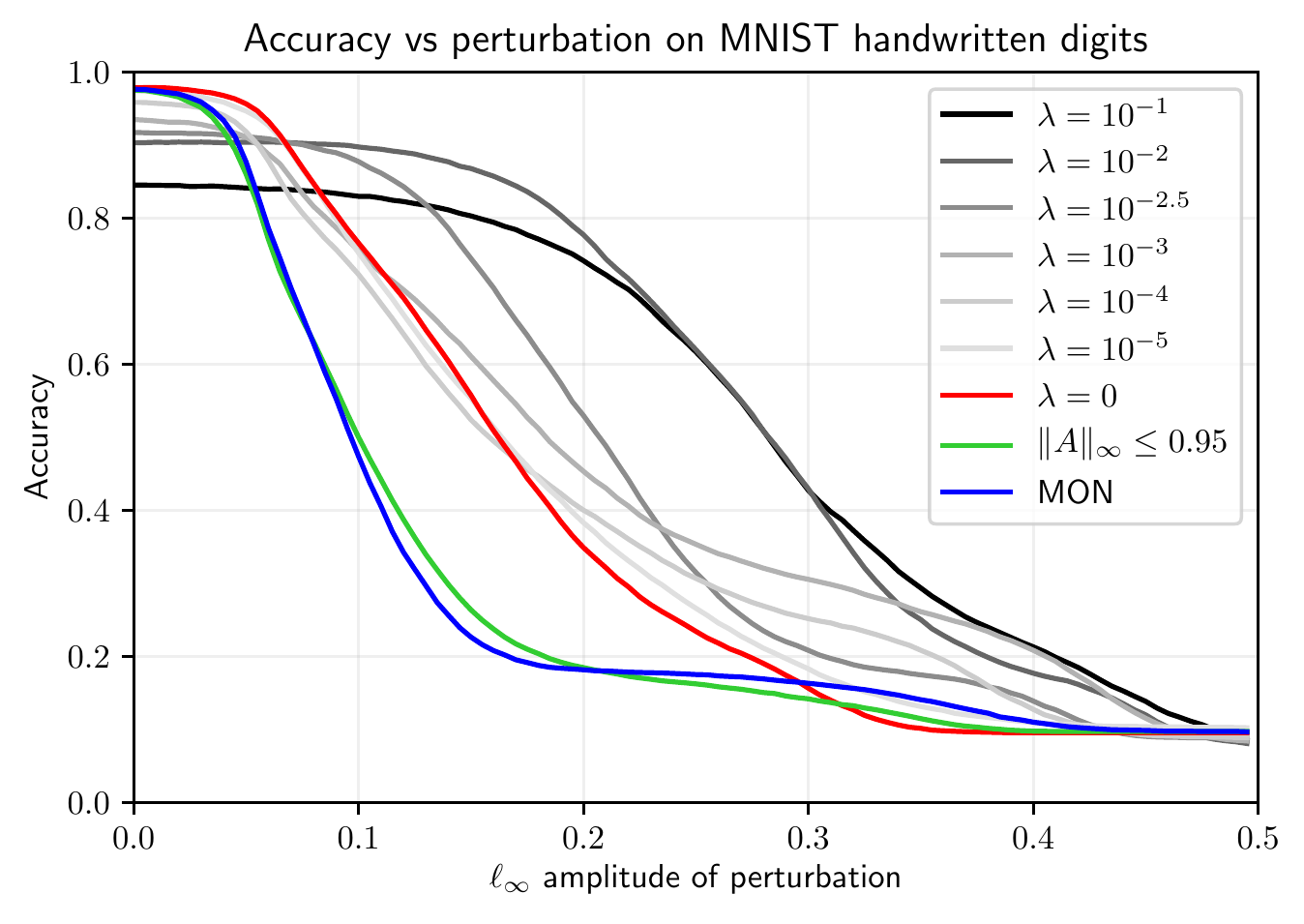}
	\end{tabular}
	\caption{On the left is a plot of test error versus Lipschitz
          constant for the implicit deep learning model with
          $\|A\|_\infty \leq 0.95$  and for NEMON with $\mu_\infty(A)
          \leq 0.95$ and parametrized by the regularization
          hyperparameter $\lambda$. We define the test error as $1$
          minus the accuracy. On the right is a plot of accuracy
          versus $\ell_\infty$ perturbation of a deterministic
          adversarial image inversion attack where we additionally
          include the MON model with $I_n-\frac{1}{2}(A+A^{\top})\succeq 0.05I_n$.}\label{fig:robustness}
\end{figure}
We {\color{black}also} study the robustness of the NEMON model compared
to the implicit deep learning model and the MON model on the MNIST dataset. We train
various models regularized by the input-output Lipschitz constant as
in~\eqref{eq:regularizedLip}. Additionally, to verify robustness of
the different models, we consider several adversarial attacks and plot
the accuracy versus perturbation of such an attack.
%in Figure~\ref{fig:robustness}. 
In Figure~\ref{fig:robustness}, we consider a continuous image inversion
attack~\citep{HH-BX-MJ-RP:17}, where each pixel is perturbed in the
direction of pixel value inversion with amplitude given by the
$\ell_\infty$ perturbation. For more details on this and other types of
adversarial perturbations, we refer to Appendix~\ref{app:morenumerics}. We
observe that for $\lambda=10^{-5}$, the regularized NEMON model achieves a two
order of magnitude decrease in its input-output Lipschitz constant compared
to the un-regularized NEMON models. In addition, we see that the implicit
deep learning model and the MON model are more sensitive to the continuous image
inversion attack than NEMON. Moreover, as the regularization
parameter $\lambda$ increases, the NEMON model becomes increasingly robust to
this attack.

\subsection{CIFAR-10 experiments}
\textcolor{black}{In the image classification dataset CIFAR-10, input
  data are $32 \times 32$ color images in $10$ classes. There are
  $50000$ training images and $10000$ test images. We compare our
  proposed NEMON model with a convolutional structure to a single
  convolutional layer MON model. Each model used $81$ channels. We
  train both models with a batch size of $256$ and a learning rate of $10^{-3}$ for $40$ epochs. For training, using the average iteration, NEMON took, on average, $10$ forward iterations, $10$ backward iterations, and $75.0$ seconds per epoch to train. Using the Peaceman-Rachford iteration, MON took, on average, $5$ forward iterations, $5$ backward iterations, and $101.8$ seconds per epoch to train. }

\textcolor{black}{We focus primarily on the robustness of NEMON and MON
  with respect to $\ell_\infty$-norm bounded perturbations on
  CIFAR-10. To this end, we additionally trained two NEMON models with
  regularization parameters $\lambda \in \{10^{-4}, 10^{-5}\}$. In
  Figure~\ref{fig:robustness-CIFAR}, on the left is a plot of the
  certified robustness of each of the models via their
  $\ell_\infty$-Lipschitz constants. For MON, we got the
  $\ell_\infty$-Lipschitz bound using the method
  in~\citep{CP-EW-JZK:21} for the $\ell_2$-Lipschitz bound and using
  the upper bound $\|u\|_2 \leq \sqrt{rs^2}\|u\|_{\infty}$. On the
  right is a plot of the accuracy of different models with respect to
  the projected gradient descent attack. We observe that the
  un-regularized and regularized
  NEMON models are more robust to $\ell_\infty$-norm bounded perturbations than is MON.}

\begin{figure}[ht]\centering
	\begin{tabular}{cc}
		\includegraphics[width = 0.482\linewidth,clip]{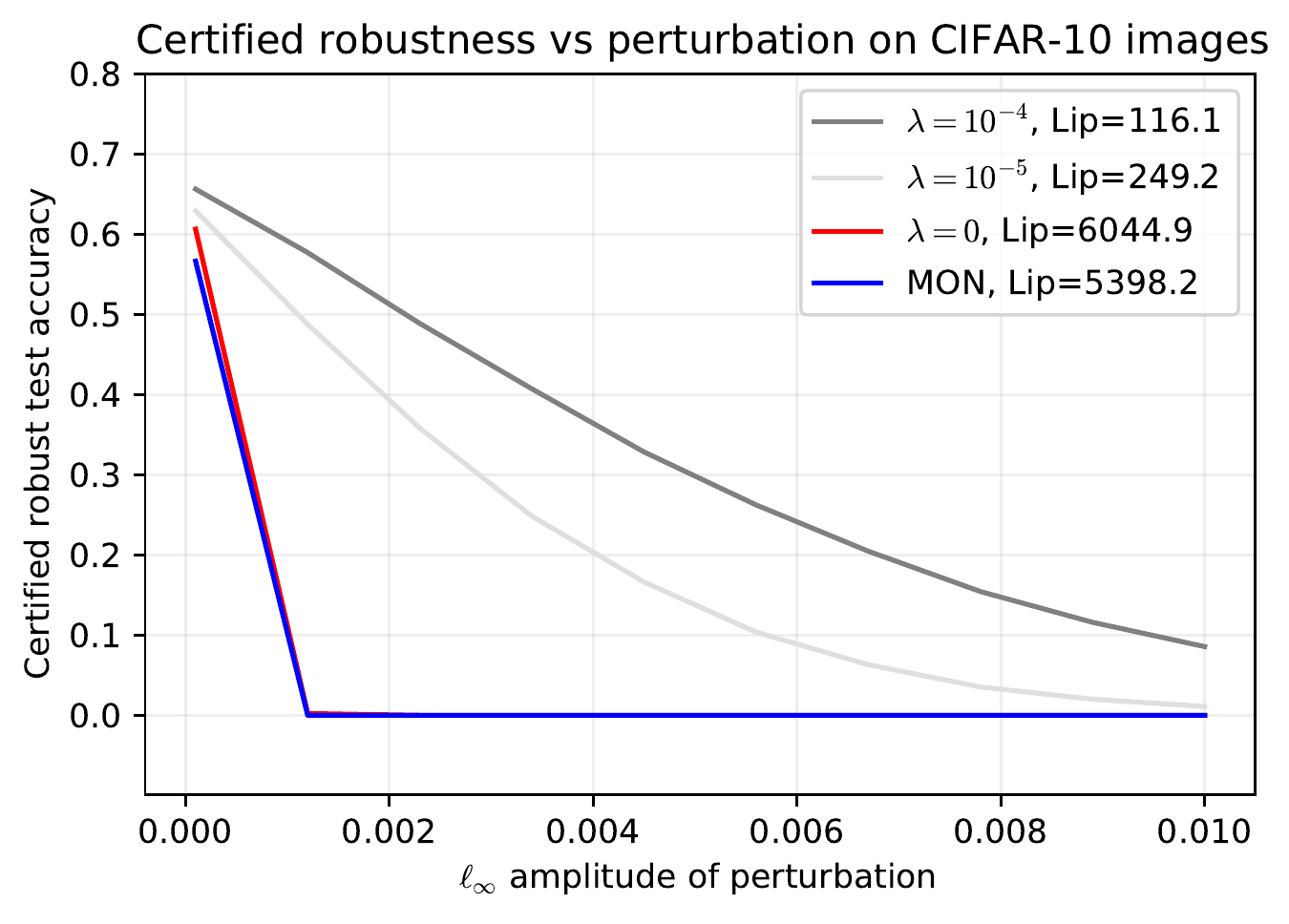}&
		\includegraphics[width = 0.482\linewidth,clip]{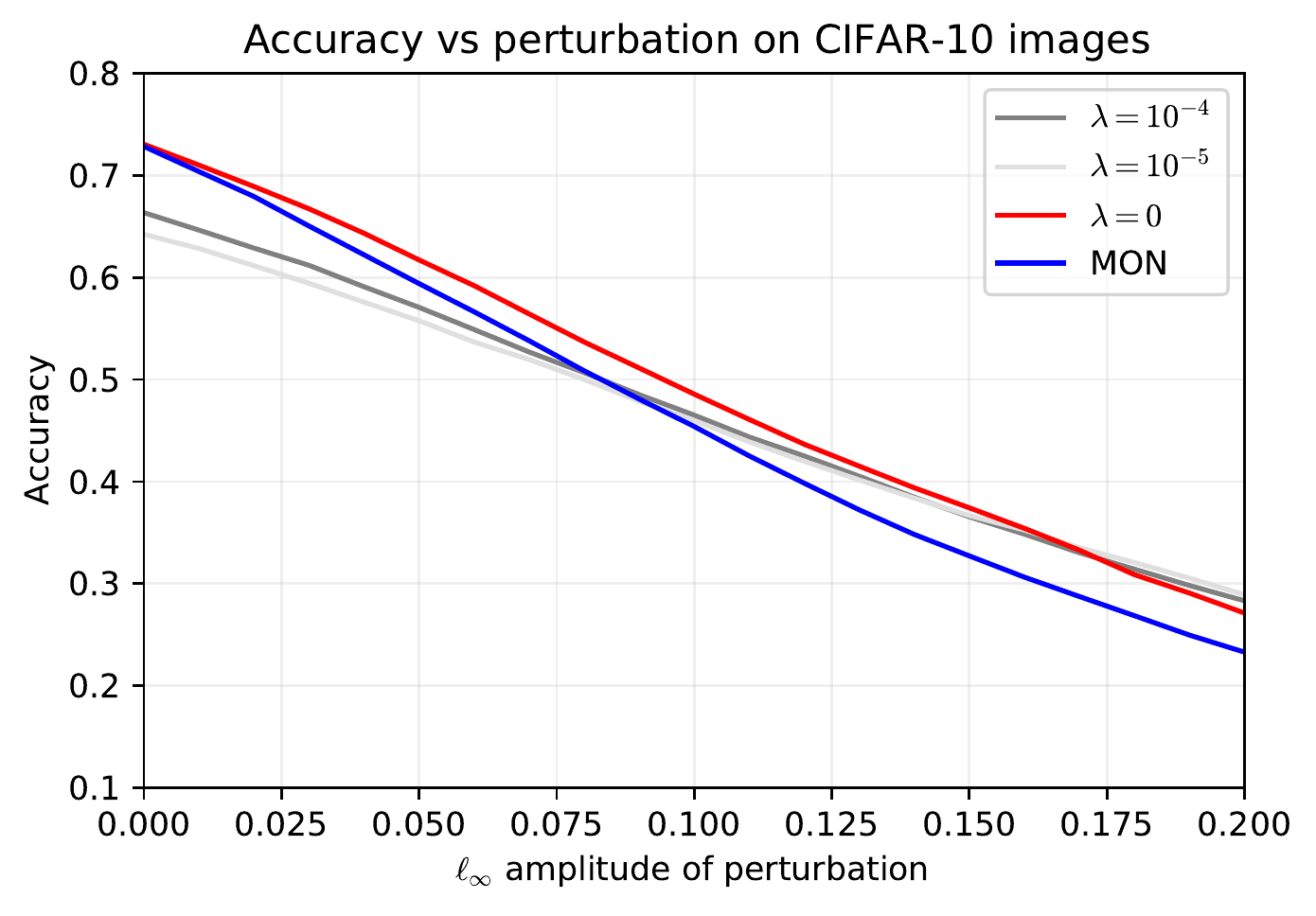}
	\end{tabular}
	\caption{On the left is a plot of certified robustness via the
          Lipschitz constants of MON with the constraint
          $I_n-\frac{1}{2}(A+A^{\top})\succeq I_n$ and NEMON with the
          constraint $\mu_{\infty}(A)\le 0$. On the right is a plot of accuracy versus $\ell_\infty$ perturbation of the projected gradient descent attack.}\label{fig:robustness-CIFAR}
\end{figure}

%% \myclearpage

\section{Conclusion}

Using non-Euclidean contraction theory, we propose a framework to
study stability of fixed-point equations. We apply this framework
to analyze well-posedness and convergence of implicit neural
networks and to design an efficient training algorithm to incorporate
robustness guarantees. For future research, we envision that our framework is
applicable to study stability and robustness of implicit learning
models with additional structure such as graph neural
networks. 

\section{Acknowledgments}
The authors thank Ian Manchester for stimulating discussions about
contraction theory and the anonymous reviewers for their insightful
feedback. This work was supported in part by DTRA under grant
HDTRA1-19-1-0017, AFOSR under grant FA9550-22-1-0059, and NSF Graduate
Research Fellowship under grant 2139319.

\bibliographystyle{abbrvnat}
\bibliography{alias,Main,FB}

\begin{thebibliography}{34}
\providecommand{\natexlab}[1]{#1}
\providecommand{\url}[1]{\texttt{#1}}
\expandafter\ifx\csname urlstyle\endcsname\relax
  \providecommand{\doi}[1]{doi: #1}\else
  \providecommand{\doi}{doi: \begingroup \urlstyle{rm}\Url}\fi

\bibitem[Agrawal et~al.(2019)Agrawal, Amos, Barratt, Boyd, Diamond, and
  Kolter]{AA-BA-SB-SB-SD-JZK:19}
A.~Agrawal, B.~Amos, S.~Barratt, S.~Boyd, S.~Diamond, and J.~Z. Kolter.
\newblock Differentiable convex optimization layers.
\newblock In \emph{Advances in Neural Information Processing Systems}, 2019.
\newblock URL \url{https://arxiv.org/abs/1910.12430}.

\bibitem[Amos et~al.(2018)Amos, Jimenez, Sacks, Boots, and
  Kolter]{BA-IJ-JS-BB-JZK:18}
B.~Amos, I.~Jimenez, J.~Sacks, B.~Boots, and J.~Z. Kolter.
\newblock Differentiable {MPC} for end-to-end planning and control.
\newblock In \emph{Advances in Neural Information Processing Systems}, 2018.
\newblock URL \url{https://arxiv.org/abs/1810.13400}.

\bibitem[Bai et~al.(2019)Bai, Kolter, and Koltun]{SB-JZK-VK:19}
S.~Bai, J.~Z. Kolter, and V.~Koltun.
\newblock Deep equilibrium models.
\newblock In \emph{Advances in Neural Information Processing Systems}, 2019.
\newblock URL \url{https://arxiv.org/abs/1909.01377}.

\bibitem[Bauschke and Combettes(2017)]{HHB-PLC:17}
H.~H. Bauschke and P.~L. Combettes.
\newblock \emph{Convex Analysis and Monotone Operator Theory in Hilbert
  Spaces}.
\newblock Springer, 2 edition, 2017.
\newblock ISBN 978-3-319-48310-8.

\bibitem[Carlini and Wagner(2017)]{NC-DW:17}
N.~Carlini and D.~Wagner.
\newblock Adversarial examples are not easily detected: {B}ypassing ten
  detection methods.
\newblock In \emph{ACM Workshop on Artificial Intelligence and Security}, pages
  3--14, 2017.
\newblock \doi{10.1145/3128572.3140444}.

\bibitem[Chen et~al.(2018)Chen, Rubanova, Bettencourt, and
  Duvenaud]{RTQC-YR-JB-DD:18}
R.~T.~Q. Chen, Y.~Rubanova, J.~Bettencourt, and D.~Duvenaud.
\newblock Neural ordinary differential equations.
\newblock In \emph{Advances in Neural Information Processing Systems}, 2018.
\newblock URL \url{https://arxiv.org/abs/1806.07366}.

\bibitem[Combettes and Pesquet(2020)]{PLC-JCP:20}
P.~L. Combettes and J.-C. Pesquet.
\newblock Lipschitz certificates for layered network structures driven by
  averaged activation operators.
\newblock \emph{SIAM Journal on Mathematics of Data Science}, 2\penalty0
  (2):\penalty0 529--557, 2020.
\newblock \doi{10.1137/19M1272780}.

\bibitem[{Combettes} and {Pesquet}(2021)]{PLC-JCP:21}
P.~L. {Combettes} and J.-C. {Pesquet}.
\newblock Fixed point strategies in data science.
\newblock \emph{IEEE Transactions on Signal Processing}, 2021.
\newblock \doi{10.1109/TSP.2021.3069677}.

\bibitem[Davydov et~al.(2021)Davydov, Jafarpour, and
  Bullo]{AD-SJ-FB:20o-simple}
A.~Davydov, S.~Jafarpour, and F.~Bullo.
\newblock {Non-Euclidean} contraction theory for robust nonlinear stability.
\newblock 2021.
\newblock URL \url{https://arxiv.org/abs/2103.12263}.

\bibitem[Desoer and Haneda(1972)]{CAD-HH:72}
C.~A. Desoer and H.~Haneda.
\newblock The measure of a matrix as a tool to analyze computer algorithms for
  circuit analysis.
\newblock \emph{IEEE Transactions on Circuit Theory}, 19\penalty0 (5):\penalty0
  480--486, 1972.
\newblock \doi{10.1109/TCT.1972.1083507}.

\bibitem[{El~Ghaoui} et~al.(2021){El~Ghaoui}, Gu, Travacca, Askari, and
  Tsai]{LEG-FG-BT-AA-AYT:21}
L.~{El~Ghaoui}, F.~Gu, B.~Travacca, A.~Askari, and A.~Tsai.
\newblock Implicit deep learning.
\newblock \emph{SIAM Journal on Mathematics of Data Science}, 3\penalty0
  (3):\penalty0 930--958, 2021.
\newblock \doi{10.1137/20M1358517}.

\bibitem[Fang and Kincaid(1996)]{YF-TGK:96}
Y.~Fang and T.~G. Kincaid.
\newblock Stability analysis of dynamical neural networks.
\newblock \emph{IEEE Transactions on Neural Networks}, 7\penalty0 (4):\penalty0
  996--1006, 1996.
\newblock \doi{10.1109/72.508941}.

\bibitem[Fazlyab et~al.(2019)Fazlyab, Robey, Hassani, Morari, and
  Pappas]{MF-AR-HH-MM-GJP:19}
M.~Fazlyab, A.~Robey, H.~Hassani, M.~Morari, and G.~J. Pappas.
\newblock Efficient and accurate estimation of {L}ipschitz constants for deep
  neural networks.
\newblock In \emph{Advances in Neural Information Processing Systems}, 2019.
\newblock URL \url{https://arxiv.org/abs/1906.04893}.

\bibitem[Fazlyab et~al.(2020)Fazlyab, Morari, and Pappas]{MF-MM-GJP:20}
M.~Fazlyab, M.~Morari, and G.~J. Pappas.
\newblock Safety verification and robustness analysis of neural networks via
  quadratic constraints and semidefinite programming.
\newblock \emph{IEEE Transactions on Automatic Control}, 2020.
\newblock \doi{10.1109/TAC.2020.3046193}.

\bibitem[Goodfellow et~al.(2015)Goodfellow, Shlens, and Szegedy]{IJG-JS-CZ:15}
I.~J. Goodfellow, J.~Shlens, and C.~Szegedy.
\newblock Explaining and harnessing adversarial examples.
\newblock In \emph{International Conference on Learning Representations
  (ICLR)}, 2015.
\newblock URL \url{https://arxiv.org/abs/1412.6572}.

\bibitem[He and Cao(2009)]{WH-JC:09}
W.~He and J.~Cao.
\newblock Exponential synchronization of chaotic neural networks: a matrix
  measure approach.
\newblock \emph{Nonlinear Dynamics}, 55:\penalty0 55--65, 2009.
\newblock \doi{10.1007/s11071-008-9344-4}.

\bibitem[Hosseini et~al.(2017)Hosseini, Xiao, Jaiswal, and
  Poovendran]{HH-BX-MJ-RP:17}
H.~Hosseini, B.~Xiao, M.~Jaiswal, and R.~Poovendran.
\newblock On the limitation of convolutional neural networks in recognizing
  negative images.
\newblock In \emph{IEEE International Conference on Machine Learning and
  Applications}, pages 352--358, 2017.
\newblock \doi{10.1109/ICMLA.2017.0-136}.

\bibitem[Kag et~al.(2020)Kag, Zhang, and Saligrama]{AK-ZZ-VS:20}
A.~Kag, Z.~Zhang, and V.~Saligrama.
\newblock {RNNs} incrementally evolving on an equilibrium manifold: {A} panacea
  for vanishing and exploding gradients?
\newblock In \emph{International Conference on Learning Representations}, 2020.
\newblock URL \url{https://openreview.net/forum?id=HylpqA4FwS}.

\bibitem[Kawaguchi(2021)]{KK:21}
K.~Kawaguchi.
\newblock On the theory of implicit deep learning: {Global} convergence with
  implicit layers.
\newblock In \emph{International Conference on Learning Representations}, 2021.
\newblock URL \url{https://openreview.net/forum?id=p-NZIuwqhI4}.

\bibitem[Krishnan et~al.(2020)Krishnan, Makdah, and Pasqualetti]{VK-AAAM-FP:20}
V.~Krishnan, A.~A.~A. Makdah, and F.~Pasqualetti.
\newblock Lipschitz bounds and provably robust training by {Laplacian}
  smoothing.
\newblock In \emph{Advances in Neural Information Processing Systems}, 2020.
\newblock URL \url{https://arxiv.org/abs/2006.03712}.

\bibitem[Lim(1985)]{TCL:85}
T.~C. Lim.
\newblock On fixed point stability for set-valued contractive mappings with
  applications to generalized differential equations.
\newblock \emph{Journal of Mathematical Analysis and Applications},
  110\penalty0 (2):\penalty0 436--441, 1985.
\newblock \doi{10.1016/0022-247X(85)90306-3}.

\bibitem[Lohmiller and Slotine(1998)]{WL-JJES:98}
W.~Lohmiller and J.-J.~E. Slotine.
\newblock On contraction analysis for non-linear systems.
\newblock \emph{Automatica}, 34\penalty0 (6):\penalty0 683--696, 1998.
\newblock \doi{10.1016/S0005-1098(98)00019-3}.

\bibitem[Madry et~al.(2018)Madry, Makelov, Schmidt, Tsipras, and
  Vladu]{AM-AM-LS-DT-AV:17}
A.~Madry, A.~Makelov, L.~Schmidt, D.~Tsipras, and A.~Vladu.
\newblock Towards deep learning models resistant to adversarial attacks.
\newblock In \emph{International Conference on Machine Learning}, 2018.
\newblock URL \url{https://arxiv.org/abs/1706.06083}.

\bibitem[Pabbaraju et~al.(2021)Pabbaraju, Winston, and Kolter]{CP-EW-JZK:21}
C.~Pabbaraju, E.~Winston, and J.~Z. Kolter.
\newblock Estimating {L}ipschitz constants of monotone deep equilibrium models.
\newblock In \emph{International Conference on Learning Representations}, 2021.
\newblock URL \url{https://openreview.net/forum?id=VcB4QkSfyO}.

\bibitem[Papernot et~al.(2016)Papernot, McDaniel, Wu, Jha, and
  Swami]{NP-PM-XW-SJ-AS:16}
N.~Papernot, P.~McDaniel, X.~Wu, S.~Jha, and A.~Swami.
\newblock Distillation as a defense to adversarial perturbations against deep
  neural networks.
\newblock In \emph{IEEE Symposium on Security and Privacy}, pages 582--597,
  2016.
\newblock \doi{10.1109/SP.2016.41}.

\bibitem[Pauli et~al.(2021)Pauli, Koch, Berberich, Kohler, and
  Allgower]{PP-AK-JB-PK-FA:21}
P.~Pauli, A.~Koch, J.~Berberich, P.~Kohler, and F.~Allgower.
\newblock Training robust neural networks using {Lipschitz} bounds.
\newblock \emph{IEEE Control Systems Letters}, 2021.
\newblock \doi{10.1109/LCSYS.2021.3050444}.

\bibitem[Pavlov et~al.(2004)Pavlov, Pogromsky, {Van de Wouw}, and
  Nijmeijer]{AP-AP-NVDW-HN:04}
A.~Pavlov, A.~Pogromsky, N.~{Van de Wouw}, and H.~Nijmeijer.
\newblock Convergent dynamics, a tribute to {B}oris {P}avlovich {D}emidovich.
\newblock \emph{Systems \& Control Letters}, 52\penalty0 (3-4):\penalty0
  257--261, 2004.
\newblock \doi{10.1016/j.sysconle.2004.02.003}.

\bibitem[Qiao et~al.(2001)Qiao, Peng, and Xu]{HQ-JP-ZBX:01}
H.~Qiao, J.~Peng, and Z.-B. Xu.
\newblock Nonlinear measures: {A} new approach to exponential stability
  analysis for {Hopfield}-type neural networks.
\newblock \emph{IEEE Transactions on Neural Networks}, 12\penalty0
  (2):\penalty0 360--370, 2001.
\newblock \doi{10.1109/72.914530}.

\bibitem[Revay et~al.(2020)Revay, Wang, and Manchester]{MR-RW-IRM:20}
M.~Revay, R.~Wang, and I.~R. Manchester.
\newblock Lipschitz bounded equilibrium networks.
\newblock 2020.
\newblock URL \url{https://arxiv.org/abs/2010.01732}.

\bibitem[Revay et~al.(2021)Revay, Wang, and Manchester]{MR-RW-IRM:21}
M.~Revay, R.~Wang, and I.~R. Manchester.
\newblock A convex parameterization of robust recurrent neural networks.
\newblock \emph{IEEE Control Systems Letters}, 5\penalty0 (4):\penalty0
  1363--1368, 2021.
\newblock \doi{10.1109/LCSYS.2020.3038221}.

\bibitem[Ryu and Boyd(2016)]{EKR-SB:16}
E.~K. Ryu and S.~Boyd.
\newblock Primer on monotone operator methods.
\newblock \emph{Applied Computational Mathematics}, 15\penalty0 (1):\penalty0
  3--43, 2016.

\bibitem[Szegedy et~al.(2014)Szegedy, Zaremba, Sutskever, Bruna, Erhan,
  Goodfellow, and Fergus]{CZ-WZ-IS-JB-DE-IG-RF:13}
C.~Szegedy, W.~Zaremba, I.~Sutskever, J.~Bruna, D.~Erhan, I.~Goodfellow, and
  R.~Fergus.
\newblock Intriguing properties of neural networks.
\newblock In \emph{International Conference on Learning Representations}, 2014.
\newblock URL \url{https://arxiv.org/abs/1312.6199}.

\bibitem[Winston and Kolter(2020)]{EW-JZK:20}
E.~Winston and J.~Z. Kolter.
\newblock Monotone operator equilibrium networks.
\newblock In \emph{Advances in Neural Information Processing Systems}, 2020.
\newblock URL \url{https://arxiv.org/abs/2006.08591}.

\bibitem[Wong and Kolter(2018)]{EW-ZK:18}
E.~Wong and J.~Z. Kolter.
\newblock Provable defenses against adversarial examples via the convex outer
  adversarial polytope.
\newblock In \emph{International Conference on Machine Learning}, pages
  5286--5295, 2018.
\newblock URL \url{http://proceedings.mlr.press/v80/wong18a.html}.

\end{thebibliography}

\newpage

\clearpage\appendix

\section{A comprehensive review of non-Euclidean contraction theory}\label{app:reviewmaterial}
\paragraph*{Matrix measures}
Let $\|\cdot\|$ be a norm on $\real^n$ and its induced norm on
$\real^{n\times{n}}$.  The matrix measure of $A \in \real^{n \times n}$
with respect to $\norm{\cdot}{}$ is
\begin{equation}
  \mu(A) := \lim_{h \to 0^+} \frac{\|I_n + hA\| - 1}{h}.
\end{equation}

It is well known that this limit is well posed because the right-hand side
is non-increasing in $h$, due to the convexity of the norm.  For arbitrary
$n\times n$ matrices $A$ and $B$, the following properties hold:
  \begin{subequations}
    \begin{alignat}{2}
      &\text{sub-additivity:}\quad&
      \mu(A+B) &\leq   \mu(A)+\mu(B)  , \label{measure:subadd}
      \\
      &\text{weak homogeneity:}& 
      \mu(\alpha A) &= \alpha \mu( A), \enspace\forall{\alpha}\geq0,
      \label{measure:scaling}
      \\
      &\text{convexity:}& 
      \mu(\theta A + (1-\theta) B) & \leq \theta \mu(A) + (1-\theta)  \mu(B) ,
      \enspace\forall\theta\in[0,1],
      \label{measure:convex}
      \\
      &\text{norm/spectrum:}&
      -\norm{A}{}\leq -\mu(-A)  \leq 
        \realpart(\lambda) & \leq \mu(A) \leq \norm{A}{},
      \enspace\forall\lambda\in\spectrum(A),
      \label{measure:munegative-Hurwitz}
      \\
      &\text{translation:}& 
      \mu(A+c I_n) &= \mu(A)+c, \enspace\forall{c}\in\real,
      \label{measure:translation}
      \\
      &\text{product:}&
      \max\{-\mu(A),-\mu(-A)\}\norm{x}{} & \leq \norm{Ax}{}, \enspace\forall{x}\in\real^n, \label{measure:product}
      \\
      %% &\text{norm of difference property:}&
      %% |\mu(A)-\mu(B)| & \leq \norm{A-B}{},
      %% \label{measure:difference-norm}
      %% \\
      &\text{norm of inverse:}&
      \mu(A) < 0 \enspace &\implies  \norm{A^{-1}}{}\leq -1 /\mu(A) \label{measure:normA}.
    \end{alignat}
  \end{subequations}
  Note that convexity is an immediate consequence of sub-additivity
  and weak homogeneity. Additionally, by property~\eqref{measure:munegative-Hurwitz}, the matrix measure is upper
  bounded by the matrix norm and may be negative. We refer to
  \citep{CAD-HH:72}, and references therein, for the proof of these
  and additional properties enjoyed by matrix measures.

We will be specifically interested in diagonally weighted $\ell_1$ and $\ell_\infty$
norms defined by
\begin{gather}
  \norm{x}{1,[\eta]} = \sum_i\eta_i|x_i|
  \qquad\text{and}\qquad
  \norm{x}{\infty,[\eta]^{-1}} = \max_{i}  \frac{1}{\eta_i}|x_i|,
\end{gather}
where, given a positive vector $\eta\in\realpositive^n$, we use $[\eta]$ to denote the diagonal matrix with diagonal entries $\eta$.  The corresponding matrix norms and measures are
\begin{align}
  \norm{A}{1,[\eta]} &= \max_{j\in\until{n}} \sum_{i=1}^n \frac{\eta_i}{\eta_j} |a_{ij}|,
  \quad
  &&\mu_{1,[\eta]}(A)  =
  \max_{j\in\until{n}} \Big( a_{jj} + \sum_{i=1,i\neq j}^n  |a_{ij}| \frac{\eta_i}{\eta_j}   \Big), \\
 \norm{A}{\infty,[\eta]^{-1}} &= \max_{i\in\until{n}} \sum_{j=1}^n \frac{\eta_j}{\eta_i}  |a_{ij}|,
  \quad
  &&\mu_{\infty,[\eta]^{-1}}(A) = \max_{i\in\until{n}} \Big( a_{ii} + \sum_{j=1,j\neq i}^n
  |a_{ij}| \frac{\eta_j}{\eta_i}\Big). \label{eq:inftyMatrixMeasure-appendix}
\end{align}
Finally, we include the Euclidean norm $\ell_2$. Given a positive definite
$P$, we define the weighted $\ell_2$ norm by
\begin{equation*}
  \|x\|_{2,P^{1/2}} = \sqrt{x^{\top} P x}.
  \end{equation*}
Then the following equalities are well known, e.g.,
see~\citep{CAD-HH:72,AD-SJ-FB:20o-simple},
\begin{align}  
  \mu_{2,P^{1/2}}(A)  &= \subscr{\lambda}{max}\Big( \frac{P A P^{-1} + A^\top  }{2} \Big)
  = \min\setdef{b\in\real}{A^\top P + PA\preceq 2b P} \label{eq:LMIell2}
  \\
  &= \max\setdef{x^\top PA x}{x^\top P x=1}. \label{eq:Lumer:ell2}
\end{align}

\paragraph{Weak pairings}
We briefly review the notion of a weak pairing (WP) on
$\real^{n}$ from~\citep{AD-SJ-FB:20o-simple}. A \emph{WP} on $\real^n$ is a map
$\WP{\cdot}{\cdot}: \real^n \times \real^n \to \real$ satisfying:
\begin{enumerate}
\item\label{WP1-appendix}(sub-additivity and continuity of first argument)
  $\WP{x_1+x_2}{y} \leq \WP{x_1}{y} + \WP{x_2}{y}$, for all
  $x_1,x_2,y \in \real^n$ and $\WP{\cdot}{\cdot}$ is continuous in its
  first argument,
\item\label{WP3-appendix}(weak homogeneity)
  $\WP{\alpha x}{y} = \WP{x}{\alpha y} = \alpha\WP{x}{y}$ and
  $\WP{-x}{-y} = \WP{x}{y}$, for all
  $x,y \in \real^n, \alpha \geq 0$,
\item\label{WP4-appendix}(positive definiteness) $\WP{x}{x} > 0$, for all
  $x \neq \vectorzeros[n],$
\item\label{WP5-appendix}(Cauchy-Schwarz inequality)
  $|\WP{x}{y}| \leq \WP{x}{x}^{1/2}\WP{y}{y}^{1/2}$, for all
  $x, y \in \real^n.$
\end{enumerate}
For every norm $\|\cdot\|$ on $\real^n$, there exists a (possibly not
unique) compatible WP $\WP{\cdot}{\cdot}$ such that
$\|x\|^2=\WP{x}{x}$, for every $x\in \real^n$. If the norm is induced by
an inner product, the WP coincides with the inner product.

Specifically, from~\citep[Table~III]{AD-SJ-FB:20o-simple}, we introduce the
WPs
$\map{\WP{\cdot}{\cdot}_{1,[\eta]}}{\real^n\times\real^n}{\real}$ and
$\map{\WP{\cdot}{\cdot}_{\infty,[\eta]^{-1}}}{\real^n\times\real^n}{\real}$, 
defined by
\begin{equation}
  \label{def: WP-1+infty-appendix}
  \WP{x}{y}_{1,[\eta]} = \|y\|_{1,[\eta]}\sign(y)^\top [\eta]x
  \qquad\text{and}\qquad  \WP{x}{y}_{\infty,[\eta]^{-1}} = \max_{i \in I_{\infty}([\eta]^{-1}y)} \eta_i^{-2} y_ix_i.
\end{equation}
where $\Iinfty(x) = \setdef{i\in\until{n}}{|x_i|=\norm{x}{\infty}}$. One
can show the so-called Lumer equalities (generalizing
equation~\eqref{eq:Lumer:ell2}):
\begin{align}
  \mu_{1,[\eta]}(A) &=  \max_{\norm{x}{1,[\eta]}=1} \quad \sign(x)^\top [\eta] A  x,
  \label{eq:Lumer-1-appendix}
  \\
  \mu_{\infty,[\eta]^{-1}}(A)  &= \max_{\|x\|_{\infty,[\eta]^{-1}} = 1} \quad \max_{i \in I_{\infty}([\eta]^{-1}x)} ([\eta]^{-1}x)_i([\eta]^{-1}Ax)_i.
  \label{eq:Lumer-infty-appendix}
\end{align}

\paragraph{Lipschitz maps}
Given a norm $\norm{\cdot}{}$ with induced matrix norm $\norm{\cdot}{}$ and
induced matrix measure $\mu(\cdot)$, a map $\map{\OF}{\real^n}{\real^n}$ is
Lipschitz continuous with constant $\Lip(\OF)\in\realnonnegative$ if
\begin{equation}
  \norm{\OF(x_1)-\OF(x_2)}{}\leq \Lip(\OF) \norm{x_1-x_2}{} \qquad \text{for all } x_1,x_2\in\real^n.
\end{equation}
If the map $\OF$ is differentiable, then $\OF$ is Lipschitz continuous with
constant $\Lip(\OF)$ if and only if
\begin{equation}
   \norm{\jac{\OF}(x)}{} \leq \Lip(\OF) \qquad \text{for all } x\in\real^n.
\end{equation}
% For example, for an affine $\OF(x)=Ax+b$, one can see
% $\Lip(\OF)=\norm{A}{}$.

\paragraph{One-sided Lipschitz maps}
Given a norm $\norm{\cdot}{}$ with compatible WP
$\WP{\cdot}{\cdot}$ and associated matrix measure $\mu(\cdot)$, a
continuous map $\map{\OF}{\real^n}{\real^n}$ is one-sided Lipschitz
continuous with constant $\osL(\OF)\in\real$ if
\begin{equation} \label{eq:osL=WP-appendix}
  \WP{\OF(x_1)-\OF(x_2)}{x_1-x_2}  \leq \osL(\OF)
    \norm{x_1-x_2}{}^2   \qquad \text{for all } x_1,x_2\in\real^n.
\end{equation}
If the map ${\OF}$ is differentiable, then $\OF$ is one-sided Lipschitz
continuous with constant $\osL(\OF)\in\real$ if and only if
\begin{equation}\label{eq:osL=muD-appendix}
  \mu(\jac{\OF}(x)) \leq \osL(\OF)  \qquad \text{for all } x\in\real^n.
\end{equation}
In other words, when the map $\OF$ is differentiable, the two
definitions~\eqref{eq:osL=WP-appendix} and~\eqref{eq:osL=muD-appendix} are equivalent. % For
% example, for an affine $\OF(x)=Ax+b$, one can see that the (smallest)
% Lipschitz constant is $\osL(\OF)=\mu(A)$.
Note that (i) the one-sided
Lipschitz constant is upper bounded by the Lipschitz constant, (ii) a
Lipschitz map is always one-sided Lipschitz, {\color{black} but the
converse is not necessarily true,} and (iii) the one-sided
Lipschitz constant may be negative. {\color{black}For instance,
  consider the scalar function $f(x)=-x-x^3$. It is easy to check that this function is not
  globally Lipschitz and $\Lip(f)=\infty$. However, $f$ is one-sided
  Lipschitz with $\osL(f)=-1$.}

In the following example, we compare the regions $\Lip(A)<1$ and $\osL(A)<1$ for a matrix $A\in
\real^{2\times 2}$ with respect to the $\ell_\infty$-norm.

\begin{example}
  Let $A= \begin{bmatrix}a & b\\ b& a\end{bmatrix}$, it is easy to see
  that condition $\Lip(A)<1$ for $\ell_{\infty}$-norm can be written
  as $\|A\|_{\infty}=|a|+|b|<1$. One can also define the average operator $A_{\alpha}$ using parameter $\alpha\in (0,1]$ as follows:
  \begin{align*}
    A_{\alpha} = (1-\alpha)I_2 + \alpha A.
  \end{align*}
 Figure~\ref{fig:regions} compares the regions
  $\Lip(A)<1$, $\Lip(A_{\alpha})<1$, and $\osL(A)<1$ based on the parameters $a$ and $b$. It can be shown that as $\alpha\to 0^+$,
  the condition $\Lip(A_{\alpha})<1$ converges to
  $\osL(A)<1$. 

      \begin{figure}[ht]\centering
        \includegraphics[width=.8\linewidth]{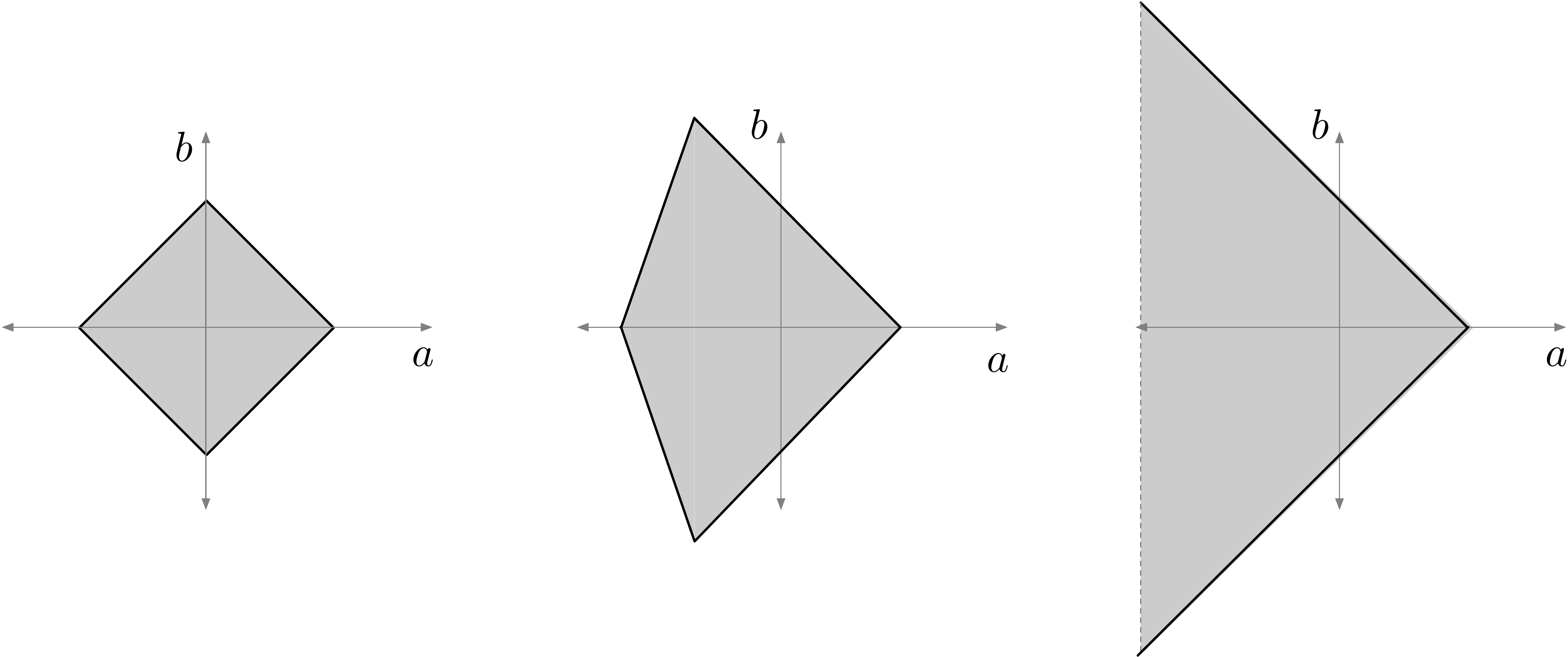}
        \caption{The left figure shows the region $\Lip(A)\le 1$,
          the middle figure shows the region $\Lip(A_{\alpha})\le 1$ for $\alpha =\tfrac{1}{2}$, and the right
          figure shows $\osL(A)\le 1$. Both $\Lip$ and $\osL$ are with
        respect to the $\ell_{\infty}$-norm}\label{fig:regions}
      \end{figure}
    \end{example}

\section{Novel results about non-Euclidean matrix measures}\label{app:matrixmeasure}

In this appendix we provide some results regarding the matrix measure and
matrix norm for weighted $\ell_1$ and $\ell_{\infty}$-norms.

\begin{lemma}[Non-Euclidean contraction estimates]\label{lem:norm-mu-nonEuc}
  Let $A=[a_{ij}]\in \real^{n\times n}$ and $\eta\in
  \real^n_{>0}$,
  \begin{enumerate}
    \item\label{p1:identity} For every $\alpha\in\real$ such
  that $|\alpha|\leq(\max_{i}|a_{ii}|)^{-1}$,
  \begin{align*}
    \|I_n + \alpha A\|_{1,[\eta]} &= 1+\alpha\mu_{1,[\eta]}(A), \\
    \|I_n + \alpha A\|_{\infty,[\eta]^{-1}} &= 1+\alpha\mu_{\infty,[\eta]^{-1}}(A).
  \end{align*}
  \item\label{p2:optimization} the minimizer and minimum value of
    $\min_{\alpha\geq 0} \|I_n + \alpha A\|_{\infty,[\eta]^{-1}}$ can be
    computed via the linear program:
    \begin{align*}
      \min_{\alpha,t} \qquad &t\\
      1&+\alpha(a_{ii}+r_i) \le t, \qquad i\in \{1,\ldots,n\},\\
      -1&+\alpha(-a_{ii}+r_i) \le t, \qquad i\in \{1,\ldots,n\},\\
      &\alpha \ge 0.
    \end{align*}
    where $r_i = \sum_{j\ne i} \frac{\eta_j}{\eta_i}|a_{ij}|$.
  \end{enumerate}
\end{lemma}
\begin{proof}
  Regarding part~\ref{p1:identity}, we compute
  \begin{align}\label{eq:formula}
    \|I_n + \alpha A\|_{\infty,[\eta]^{-1}} =
    \max_{i\in\until{n}}
    \big\{|1+\alpha a_{ii}| +
    \alpha\sum_{j=1,j\ne i}^n \frac{\eta_j}{\eta_i}|a_{ij}| \big\}.
  \end{align}
  Since $|\alpha|\leq(\max_i|a_{ii}|)^{-1}$, we know $|\alpha|
    |a_{ii}|\leq1$ for all $i\in\until{n}$. Therefore $1+\alpha a_{ii}\geq0$
    and $|1+\alpha a_{ii}| = 1+\alpha a_{ii}$, for every
    $i\in\until{n}$. In summary, replacing in~\eqref{eq:formula},
    \begin{align*}
      \|I_n + \alpha A\|_{\infty,[\eta]^{-1}} =\max_{i\in\until{n}}
      \big\{1+\alpha a_{ii}+   \alpha\sum_{j=1,j\ne i}^n
      \frac{\eta_j}{\eta_i}|a_{ij}| \big\} = 1 + \alpha \mu_{\infty,[\eta]^{-1}}(A).
  \end{align*}
  The proof of the formula relating the weighted $\ell_1$-norm and the
  weighted $\ell_1$-matrix measure will follow mutatis mutandis to the
  above proof for $\ell_\infty$-norm and we omit it in the
  interest of brevity.

  Regarding part~\ref{p2:optimization}, using formula~\eqref{eq:formula}, we get
  \begin{align*}
    \|I_n + \alpha A\|_{\infty,[\eta]^{-1}} &=
    \max_{i\in\until{n}}
    \big\{|1+\alpha a_{ii}| +
    \alpha r_i\big\} \\ & = \max_{i\in\until{n}}\big\{1+\alpha a_{ii} +
    \alpha r_i, -1-\alpha a_{ii}+
    \alpha r_i\big\}.
  \end{align*}
  The result then follows.
\end{proof}

The following results are related to~\citep[Theorem~3.8]{YF-TGK:96}
and~\citep[Lemma~3]{WH-JC:09} and, indirectly, to~\citep{HQ-JP-ZBX:01}. In
comparison with~\citep{YF-TGK:96,WH-JC:09}, we prove sharper bounds for a
more general setting. 

\begin{lemma}[Matrix measure inequalities under multiplicative scalings]
  \label{lemma:affine-scaling}
  For each $A\in\real^{n\times{n}}$, $C\in\real^{n\times{n}}$ diagonal
  positive, and $\eta\in\realpositive^{n}$,
  \begin{enumerate}
  \item\label{fact:ms:1} $\ds \max_{d\in[0,1]^n} \mu_{\infty,[\eta]} (-C+[d]A)=
    \max\big\{ \mu_{\infty,[\eta]} (-C), \mu_{\infty,[\eta]} (-C+A)\big\}$, and
  \item\label{fact:ms:2} $\ds \max_{d\in[0,1]^n} \mu_{1,[\eta]} (-C+A[d])=
    \max\big\{ \mu_{1,[\eta]} (-C), \mu_{1,[\eta]} (-C+A)\big\}$.
  \end{enumerate}
\end{lemma}
\begin{proof}
  Define the short-hand $r_i=a_{ii}+\sum_{j=1,j\not=i}^n |a_{ij}|
  \eta_i/\eta_j$ and note
  \begin{gather*}
    \mu_{\infty,[\eta]} (-C)  = \max_{i\in\until{n}} \{-c_i\},\quad
    \mu_{\infty,[\eta]} (-C+A)  = \max_{i\in\until{n}} \{-c+r_i\},
    \quad\text{and}\\
    \mu_{\infty,[\eta]} (-C + [d]A) = \max_{i\in\until{n}} \{ -c_i + d_ir_i \}.
  \end{gather*}
  Since $0\leq d_i\leq1$, we note
  \begin{align*}
    r_i&\leq 0 \quad\implies\quad d_i r_i  \leq0 \quad\implies\quad
    -c_i + d_i r_i \leq -c_i, \\
    r_i&> 0 \quad\implies\quad
    d_i r_i  \geq0 \quad\implies\quad
    -c_i+d_i r_i    \leq    -c_i+r_i.
  \end{align*}
  Therefore
  \begin{align*}
    &\max_{d\in[0,1]^n} \max_{i \,:\, r_i \leq 0} \{-c_i + d_i r_i\} =
    \max_{i \,:\, r_i \leq 0} \max_{d_i\in[0,1]} \{-c_i + d_i r_i\}
    = \max_{i \,:\, r_i \leq 0} \{-c_i\} \leq \mu_{\infty,[\eta]} (-C),
    \\
    &\max_{d\in[0,1]^n} \max_{i \,:\, r_i > 0} \{-c_i + d_i r_i\} =
    \max_{i \,:\, r_i > 0} \max_{d_i\in[0,1]} \{-c_i + d_i r_i\}
    = \max_{i \,:\, r_i \leq 0} \{-c_i+r_i\} \leq \mu_{\infty,[\eta]} (-C+A).
  \end{align*}
  In summary
  \begin{align*}
    \max_{d\in[0,1]^n} \mu_{\infty,[\eta]} (-C+[d]A) &=
    \max_{d\in[0,1]^n} \max_{i\in\until{n}} \{ -c_i + d_ir_i \}  \\
    &= \max_{d\in[0,1]^n} \max\Big\{
     \max_{i \,:\, r_i \leq 0} \{-c_i + d_i r_i\},
     \max_{i \,:\, r_i > 0} \{-c_i + d_i r_i\}
    \Big\} \\
    &\leq
    \max\big\{ \mu_{\infty,[\eta]} (-C), \mu_{\infty,[\eta]} (-C+A)\big\}.
  \end{align*}
  On the other hand, we note that
  \begin{align*}
    \max_{d\in[0,1]^n} \mu_{\infty,[\eta]} ([d]A-C)
    &\geq
    \max \big\{ \mu_{\infty,[\eta]} ([\vect{0}_n]A-C) , \mu_{\infty,[\eta]} ([\vect{1}_n]A-C) \big\}
    \\
    &=    \max\big\{ \mu_{\infty,[\eta]} (-C), \mu_{\infty,[\eta]} (-C+A)\big\},
  \end{align*}
  thereby proving the equality in statement~\ref{fact:ms:1}.  Next, recall
  $\mu_{1,[\eta]}(B)=\mu_{\infty,[\eta]}(B^\top)$ for all $B$ and compute
  \begin{align*}
    \max_{d\in[0,1]^n} \mu_{1,[\eta]} (-C+A[d])
    &= \max_{d\in[0,1]^n} \mu_{\infty,[\eta]} (-C+[d] A^\top) \\
    &=
    \max\big\{ \mu_{\infty,[\eta]} (-C), \mu_{\infty,[\eta]} (-C+A^\top)\big\} \\
    &=   \max\big\{ \mu_{1,[\eta]} (-C), \mu_{\infty,[\eta]} (-C+A)\big\}.
  \end{align*}
  This concludes the proof of statement~\ref{fact:ms:2}.
\end{proof}

In the same style as~\citep[Proposition~1]{EW-JZK:20} and~\citep[Theorems~1
    and~2]{MR-RW-IRM:20}, the next lemma provides a parametrization of all
  matrices satisfying a $\mu_\infty$ constraint.

\begin{lemma}[Parametrization of matrices with bounded $\ell_\infty$ measure]
  \label{lemma:true-param}
  For any $\gamma\in\real$,
  \begin{enumerate}
  \item\label{true-param:1} given any $A\in\real^{n\times{n}}$ with
    $\mu_{\infty}(A)\leq\gamma$, there exists a $T\in\real^{n\times{n}}$
    such that $A=T-\diag(|T|\vect{1}_n)+\gamma{I_n}$,

  \item\label{true-param:2} given any $T\in\real^{n\times{n}}$, the matrix
    $A=T-\diag(|T|\vect{1}_n)+\gamma{I_n}\in\real^{n\times{n}}$ satisfies
    $\mu_{\infty}(A)\leq\gamma$,
  \end{enumerate}
  where we let $|T|$ denote the entry-wise absolute value of $T$.
\end{lemma}
\begin{proof}
  Regarding statement~\ref{true-param:1}, define
  \begin{align*}
    t_{ij} & = a_{ij} \quad && \text{for all } i\neq j \in\until{n},\\
    t_{ii} &= \frac{1}{2}\big( a_{ii}+\sum_{j=1,j\neq{i}}^n |a_{ij}| - \gamma\big) , \quad && \text{for } i\in\until{n}.
  \end{align*}
  Because $\mu_{\infty}(A)\leq\gamma$, we know $a_{ii}+
  \sum_{j=1,j\neq{i}}^n |a_{ij}|\leq \gamma$ for each $i$. This implies
  that $t_{ii}\leq0$ and therefore $t_{ii}-|t_{ii}|=
  a_{ii}+\sum_{j=1,j\neq{i}}^n |a_{ij}| - \gamma$. It is an easy
  transcription now to show that this equality and the off-diagonal
  equality $t_{ij}=a_{ij}$ together imply
  $A=T-\diag(|T|\vect{1}_n)+\gamma{I_n}$.

  Regarding statement~\ref{true-param:2}, note that $a_{ij}=t_{ij}$ for all
  $j\neq{i}$, and $a_{ii}=t_{ii}-\sum_{j=1}^n|t_{ij}|+\gamma$.  Then, for
  all $i$,
  \begin{align*}
    a_{ii} + \sum_{j=1,j\neq{i}}^n|a_{ij}|
    &= \big( t_{ii}-\sum_{j=1}^n|t_{ij}|+\gamma \big)
    + \sum_{j=1,j\neq{i}}^n|t_{ij}| \\
    &= t_{ii} - |t_{ii}| + \gamma = \begin{cases} \gamma, &\text{ if } t_{ii}\geq0, \\
      -2|t_{ii}|+\gamma, &\text{ if }t_{ii}<0.
    \end{cases}
  \end{align*}
  Therefore, $a_{ii} + \sum_{j=1,j\neq{i}}^n|a_{ij}| \leq \gamma$ for all
  $i$ and, in turn, $\mu_{\infty}(A)\leq\gamma$.
\end{proof}

We conclude with a simple graph-theoretical interpretation of the main
well-posedness condition $\mu_\infty(A)<1$. Loosely speaking, we call
$-a_{ii}$ the self-attenuation of neuron $i$ and
$\sum_{j=1,j\neq{i}}^n|a_{ij}|$ the strength of its outgoing synapses.
Then
\begin{multline}
  \mu_\infty(A)<1 \quad\iff\quad a_{ii}+\sum_{j=1,j\neq{i}}^n|a_{ij}|<1
  \quad \text{for all $i$} \\ \quad\iff\quad \text{ for each neuron,
    strength of outgoing synapses $<$ $1+$ self-attenuation}.
\end{multline}
  
\section{Proofs and additional results on non-differentiable
  activation functions}\label{app:proof}

\subsection{Proofs of Theorems~\ref{thm:fixedpoint-osL} and~\ref{thm:acceleration-1-inf}}

\begin{proof}[Proof of Theorem~\ref{thm:fixedpoint-osL}]
  Regarding~\ref{p1:banach-fixed-suff}$\implies$\ref{p2:mu}, note that, for
  every $x\in \real^n$ and every $0<\alpha\le \alpha^*$,
  \begin{align*}
    \mu(\jac{\OF}_{\alpha}(x))\le  \|\jac{\OF}_{\alpha}(x))\| \le \gamma_{\ell,c}(\alpha).
  \end{align*}
  As a result, $\alpha \mu(\jac{\OF}(x)) = \mu(\jac{\OF}_{\alpha}(x)) -
  1+\alpha \le -1 + \alpha +\gamma_{\ell,c}(\alpha)$. Thus,
\begin{align*}
\mu(\jac{\OF}(x)) \le 1-\frac{1-\gamma_{\ell,c}(\alpha)}{\alpha},\qquad\mbox{for
  all } x\in \real^n.
\end{align*}
By choosing $\alpha =\widehat{\alpha}=
\frac{2c}{(2c+\ell+1)(\ell+1)}< \frac{c}{(c+\ell+1)(\ell+1)}$, we get
\begin{align*}
\mu(\jac{\OF}(x)) \le
  1-\frac{1-\gamma_{\ell,c}(\widehat{\alpha})}{\widehat{\alpha}} = 1-
  \frac{1-(1-\widehat{\alpha} c)}{\widehat{\alpha}} = 1-c,\qquad\mbox{for
  all } x\in \real^n.
\end{align*}
Thus, $\sup_{x\in \real^n} \mu(\jac{\OF}(x))\le 1-c$. This implies
that $\osL(\OF)\le 1-c$. 

Regarding \ref{p2:mu}$\implies$\ref{p1:banach-fixed-suff}, using the mean value theorem for vector
valued functions, we compute
\begin{align*}
  \|\OF_{\alpha}(x)-\OF_{\alpha}(y)\|= \Big\|\int_{0}^{1}\jac{\OF}_{\alpha}(tx+(1-t)y)dt (x-y)\Big\|
  \le \|\overline{\jac{\OF}}_{\alpha}(x,y)\| \|x-y\|,
\end{align*}
where $\overline{\jac{\OF}}_{\alpha}(x,y)
=\int_{0}^{1}\jac{\OF}_{\alpha}(tx+(1-t)y)dt$, for every $x,y\in
\real^n$.

Next, to obtain an upper bound on $\|\overline{\jac{\OF}}_{\alpha}(x,y)\|$,
we first derive a lower bound on
$\|\overline{\jac{\OF}}^{-1}_{\alpha}(x,y)\|$. We start by noting that, the
product property~\eqref{measure:product} implies $\|Av\|\ge -\mu(-A)\|v\|$,
for every $v\in \real^n$ and every $A\in \real^{n\times n}$. Therefore, for
every $v\in \real^n$,
\begin{align}\label{eq:useful-lowerbound}
  \|\overline{\jac{\OF}}^{-1}_{\alpha}(x,y) v\| \ge -\mu(-\overline{\jac{\OF}}^{-1}_{\alpha}(x,y))\|v\|.
\end{align}
Since $\overline{\jac{\OF}}_{\alpha}(x,y) = I_n + \alpha(-I_n +
\overline{\jac{\OF}}(x,y))$ and $\alpha<\frac{c}{(c+\ell+1)(\ell+1)}\le \frac{1}{\ell+1}$, we can use the Neumann series to get 
\begin{align}\label{eq:neuman_series}
  \overline{\jac{\OF}}^{-1}_{\alpha}(x,y) = \sum_{i=0}^{\infty}(-1)^{i}\alpha^i(-I_n + \overline{\jac{\OF}}(x,y))^{i}.
\end{align}
We first compute an upper bound for $\mu(\overline{\jac{\OF}}(x))$. Since
$\osL(\OF)\le 1-c$, by the subadditive property~\eqref{measure:subadd} of
the matrix measures, we get
\begin{align}
  \mu(-I_n + \overline{\jac{\OF}}(x,y)) &=
  \mu\left(\int_{0}^{1}(-I_n+\jac{\OF}(tx+(1-t)y))dt\right) \nonumber\\ & \le
  \int_{0}^{1}\mu\big(-I_n + \jac{\OF}(tx+(1-t)y)\big)dt \le - c. \label{eq:mu-average-upper}
\end{align}
Now, we use equation~\eqref{eq:neuman_series} to obtain
\begin{align}
  \|\overline{\jac{\OF}}^{-1}_{\alpha}(x,y) v\| &\ge -\mu\Big(\sum_{i=0}^{\infty}(-1)^{i+1}\alpha^i(-I_n+\overline{\jac{\OF}}(x,y))^{i}\Big) \|v\|
  \nonumber  \\ & \ge -\Big(\mu(-I_n)
  + \alpha \mu(-I_n+\overline{\jac{\OF}}(x,y))  \nonumber \\
  &\quad +
  \sum_{i=2}^{\infty} \alpha^i
  \mu\big((-1)^{i+1}(-I_n+\overline{\jac{\OF}}(x,y))^i\big)\Big)\|v\| \nonumber \\
  & \ge
  (1+\alpha c - \sum_{i=2}^{\infty} (\alpha(\ell+1))^i)\|v\| =
  \Big(1+\alpha c - \frac{\alpha^2(\ell+1)^2}{1-\alpha (\ell+1)}\Big)\|v\|, \label{eq:tttt}
\end{align}
where the first inequality holds by~\eqref{eq:useful-lowerbound}, the
second inequality holds by subadditive property of the matrix
measures~\eqref{measure:subadd}, and the third inequality holds because,
using \eqref{eq:mu-average-upper} and~\eqref{measure:munegative-Hurwitz},
we obtain the upper bound:
\begin{align*}
\mu\big((-1)^{i+1}(-I_n+\overline{\jac{\OF}}(x,y))^i\big) \le
  \|(-I_n+\overline{\jac{\OF}}(x,y))^i\| \le (1+\ell)^{i}, \qquad\mbox{
  for all } i\in \mathbb{Z}_{\ge 0}. 
\end{align*}
Note that $\alpha\in {]0,\frac{c}{(c+\ell+1)(\ell+1)}[}$.
Equation~\eqref{eq:tttt} implies that, for each $w\in \real^n$ and
$v=\overline{\jac{\OF}}_{\alpha}(x,y)w$,
\begin{align*}
  \frac{\|\overline{\jac{\OF}_{\alpha}}(x,y)w\|}{\|w\|} =
  \frac{\|v\|}{\|\overline{\jac{\OF}}^{-1}_{\alpha}(x,y) v\|}
  \le \gamma_{\ell,c}(\alpha) .
\end{align*}
As a result, $\|\overline{\jac{\OF}}_{\alpha}(x,y)\|\le
\gamma_{\ell,c}(\alpha)$ and
\begin{align*}
  \|\OF_{\alpha}(x)-\OF_{\alpha}(y)\|
  \le \gamma_{\ell,c}(\alpha) \|x-y\|, \qquad\mbox{ for all }x,y\in \real^n.
\end{align*}
Regarding parts~\ref{p3:unique-eq} and~\ref{p4:iteration-converge}, a
straightforward calculation shows that, if $0<\alpha <
\frac{c}{(c+\ell+1)(\ell+1)}$, then $1/\Big(1+\alpha c -
\frac{\alpha^2(\ell+1)^2}{1-\alpha (\ell+1)}\Big)<1$. The result then
follows from the Banach fixed-point theorem. Regarding
part~\ref{p5:optimalrate}, we define the function
$\xi:{]0,\frac{c}{(c+\ell+1)(\ell+1)}[}\to \real_{>0}$ by
  $\xi(\alpha)=1+\alpha c -
  \frac{\alpha^2(\ell+1)^2}{1-\alpha(\ell+1)}$. Then it is clear that
  $\xi(\alpha) = 1/\gamma_{\ell,c}(\alpha)$. Note that
  \begin{align*}
    \frac{d\xi}{d\alpha}&= (c+\ell+1)-\frac{\ell+1}{(1-\alpha(\ell+1))^2},\\
    \frac{d^2\xi}{d\alpha^2}&= -\frac{2(\ell+1)^2}{(1-\alpha(\ell+1))^3}.
  \end{align*}
  Since $\frac{d^2\xi}{d\alpha^2}\le 0$, we conclude that $\xi$
  is a concave function on
  ${]0,\frac{c}{(c+\ell+1)(\ell+1)}[}$ and its maximum is
    achieved at $\alpha^*$ for which
    $\frac{d\xi}{d\alpha}(\alpha^*)=0$. By a straightforward
    calculation, we get
    \begin{align*}
      \alpha^* = \frac{\kappa}{c}\left(1-\frac{1}{\sqrt{1+1/\kappa}}\right)
    \end{align*}
    and it is easy to see that the optimal value is as claimed in the theorem statement.
  \end{proof}

\begin{proof}[Proof of Theorem~\ref{thm:acceleration-1-inf}]
  We restrict ourselves to the norm $\|\cdot\|_{\infty,[\eta]^{-1}}$; the
  proof for $\|\cdot\|_{1,[\eta]}$ is similar and omitted in the interest
  of brevity.

  Regarding part~\ref{p4:1-inf:iteration-converge}, first we note that
  $\diagL(\OF)\leq\osL(\OF)<1$, since for every $i\in \{1\ldots,n\}$ and
  every $x\in\real^n$
  \begin{align}
    \label{eq:prop-use}
    \jac{\OF}_{ii}(x) \le \jac{\OF}_{ii}(x) + \sum_{j\ne i}
    |\jac{\OF}_{ij}(x)|\tfrac{\eta_i}{\eta_j}
    = \mu_{\infty,[\eta]^{-1}}(\jac{\OF}(x))\leq \osL(\OF)<1.
  \end{align}
  This implies that $\frac{1}{1-\diagL(\OF)}>0$ and $(1-\osL(\OF))/(1-\diagL(\OF))\leq1$.
  Moreover, for every $x\in \real^n$,
  \begin{align*}
    \|(1-\alpha)I_n + \alpha \jac{\OF}(x)\|_{\infty,[\eta]^{-1}} = \|I_n+\alpha(-I_n+\jac{\OF}(x))\|_{\infty,[\eta]^{-1}}.
  \end{align*}
    Next, we study the diagonal entries of $-I_n+\jac{\OF}(x)$.  By the
    definition of $\diagL(\OF)$ and by equation~\eqref{eq:prop-use},
    \begin{align*}
      & -1+\diagL(\OF) \leq -1+\jac{\OF}_{ii}(x) < 0  && \text{(for every $i\in\until{n}$ and $x$)} \\
      & \implies \quad |1-\diagL(\OF)| \geq | -1+\jac{\OF}_{ii}(x)| \\
      & \implies \quad 1-\diagL(\OF) \geq \max\nolimits_i | -1+\jac{\OF}_{ii}(x)| \\
      & \implies \quad \frac{1}{1-\diagL(\OF)} \leq \frac{1}{\max_i | -1+\jac{\OF}_{ii}(x)|}.
    \end{align*}
    Therefore, $\alpha\leq\frac{1}{\max_{i}|-1+\jac{\OF}_{ii}(x)|}$ and we
    can use Lemma~\ref{lem:norm-mu-nonEuc}\ref{p1:identity} to deduce that
    \begin{align*}
      \|(1-\alpha)I_n + \alpha \jac{\OF}(x)\|_{\infty,[\eta]^{-1}}
      &=   1+\alpha \mu_{\infty,[\eta]^{-1}}(-I_n+\jac{\OF}(x)) \\
      &=   1+\alpha (-1 + \mu_{\infty,[\eta]^{-1}}(\jac{\OF}(x)))
      \qquad\qquad\mbox{ for all } x\in \real^n\\
      &\leq   1+\alpha (-1 + \osL(\OF))  =  1 -\alpha(1-\osL(\OF))<1.
    \end{align*}
    where the second equality follows from the translation
    property~\eqref{measure:translation} of matrix measures, and the
    inequality holds because $\mu_{\infty,[\eta]^{-1}}(\jac{\OF}(x))\leq
    \osL(\OF)$ for all $x$, and the last inequality holds because
    $\osL(\OF)<1$. This means that $\Lip(\OF_{\alpha})<1$, for every
    $0<\alpha\le \frac{1}{1-\diagL(\OF)}$ and the result follows from the
    Banach fixed-point theorem.

    Regarding part~\ref{p5:1-inf-optimal}, we note the contraction factor
    is a strictly decreasing function of $\alpha$. At $\alpha=0$ the factor
    is $1$ and at the maximum of value of $\alpha$ that is, at
    $\alpha^*=(1-\diagL(\OF))^{-1}$ the contraction factor is still
    positive since $(1-\osL(\OF))/(1-\diagL(\OF))\leq1$. Hence the minimum
    contraction factor is achieved at~$\alpha^*$.
  \end{proof}

  \subsection{Proof of Theorem~\ref{thm:perturbed-fixed} and
    comparison with the literature}

  Before we prove Theorem~\ref{thm:perturbed-fixed}, it is useful 
  to compare it with similar results in the literature. The result in~\cite[Lemma 1]{TCL:85} is more general than
  Theorem~\ref{thm:perturbed-fixed} by allowing $\OF$ to be a multi-valued
  map defined on a metric space. However,
  Theorem~\ref{thm:perturbed-fixed}\ref{p2:lim} uses the one-side Lipschitz
  constant and provides a tighter upper bound on the distance between
  fixed-points of $\OF$ compared to its counterpart in~\cite[Lemma
  1]{TCL:85}.

  \begin{proof}[Proof of Theorem~\ref{thm:perturbed-fixed}]
    Let $\WP{\cdot}{\cdot}$ be a WP for the norm
    $\norm{\cdot}{\mcX}$ on $\real^n$.
    
    Regarding part~\ref{p1:unique-gu}, for every $u\in \real^m$, we
    note that by definition of $\osL_x(\OF)$, for every
    $u\in \real^r$,
      \begin{align*}
       \WP{\OF(x,u)-\OF(y,u)}{x-y}\le \osL_x(\OF)\|x-y\|_{\mcX}^2,
      \end{align*}
      This implies that $\osL(\OF_u) \le \osL_x(\OF)< 1$, for every $u\in
      \real^r$. Thus, by
      Theorem~\ref{thm:fixedpoint-osL}\ref{p3:unique-eq}, $\OF_u$ has a unique
      fixed-point $x^*_u$.

    Regarding part~\ref{p2:lim}, let $\WP{\cdot}{\cdot}$ be a WP for
    the norm $\norm{\cdot}{\mcX}$ on $\real^n$ and compute
    \begin{align*}
      \norm{x^*_u-x^*_v}{\mcX}^2 &= \WP{x^*_u-x^*_v}{x^*_u-x^*_v}  && \text{(by compatibility)}\\
      & =   \WP{\OF_u(x^*_u)-\OF_v(x^*_v)}{x^*_u-x^*_v} \\
      & \le
      \WP{\OF_u(x^*_u)-\OF_u(x^*_v)}{x^*_u-x^*_v} +
      \WP{\OF_u(x^*_v)-\OF_v(x^*_v)}{x^*_u-x^*_v}  && \text{(by sub-additivity)} \\
      &\le
      \osL_x(\OF)\norm{x^*_u-x^*_v}{\mcX}^2 +
      \norm{\OF_u(x^*_v)-\OF_v(x^*_v)}{\mcX} \norm{x^*_u-x^*_v}{\mcX}  &&\text{(by
                                                        Cauchy-Schwarz)}\\
      & \le \osL_x(\OF) \norm{x^*_u-x^*_v}{\mcX}^2 +
      \Lip_u(\OF) \norm{u-v}{\mcU} \norm{x^*_u-x^*_v}{\mcX}. &&
    \end{align*}
    This implies that $(1-\osL_x(\OF))\norm{x^*_u-x^*_v}{\mcX}\le
    \Lip_u(\OF)\norm{u-v}{\mcU}$ and the result of part~\ref{p2:lim}
    follows.
  \end{proof}

\subsection{Non-differentiable fixed-point problems}\label{sec:non-differentiable}

In many machine learning applications, the activation functions are
continuous but non-differentiable and thus our results in
Sections~\ref{sec:fixed-point} do not directly apply to these
problems. In this subsection, we focus on a specific form of the
fixed-point equation~\eqref{eq:fixed-point}, where
$\OF = \Phi\circ \OH$ and $\map{\Phi}{\real^n}{\real^n}$ is a diagonal
activation function with absolutely continuous components and
$\map{\OH}{\real^n\times \real^r}{\real^n}$ is a differentiable function. It can be
shown that, for this class of systems, conclusions of
Theorems~\ref{thm:fixedpoint-osL},~\ref{thm:acceleration-1-inf},
and~\ref{thm:perturbed-fixed} still hold with respect to weighted
$\ell_\infty$-norms. Here, we present a result which extends 
Theorems~\ref{thm:acceleration-1-inf} and~\ref{thm:perturbed-fixed}
for $\OH(x,u) = \OG(x) + Bu$ given some $B\in \real^{n\times r}$ and with
respect to the norm $\|\cdot\|_{\infty.[\eta]^{-1}}$. 

    \begin{theorem}[Fixed points for non-differentiable activation functions]\label{thm:non-diff}
      Consider the norm $\|\cdot\|_{\infty,[\eta]^{-1}}$ on $\real^n$ for
      some $\eta \in \realpositive^n$ and the norm
      $\|\cdot\|_{\mathcal{U}}$ on $\real^r$. Additionally, consider the following perturbed fixed point problem:
      \begin{align*}
        x = \Phi(\OG(x)+Bu):=\Phi^\OG(x,u),
      \end{align*}
      where $\map{\Phi}{\real^n}{\real^n}$ is a diagonal function given by
      $(\phi_1(x_1),\dots,\phi_n(x_n))$ with non-expansive and
      weakly increasing
      $\phi_i$, $\map{\OG}{\real^n}{\real^n}$ is a continuously
      differentiable function, and $B\in \real^{n\times r}$. Define the average map
      $\Phi^\OG_{\alpha}(x,u):= (1 - \alpha)x+ \Phi^\OG(x,u)$ and pick
      $\diagL(\OG)_-\in [-\Lip(\OG),\osL(\OG)]$ such that
      \begin{align*}
        \diagL(\OG)_{-} \leq \min_{i}\inf_{x\in \real^n}
      \jac{\OG}_{ii}(x)_-.
      \end{align*}
     Assume that $\osL(\OG)<1$. Then,
      \begin{enumerate}
      \item\label{p1:non-diff-existence} for every
        $u\in \real^n$, the map $\Phi^\OG(\cdot,u)$ has a unique fixed-point $x_u^*$;
      \item\label{p2:non-diff-iterations} for every $0<\alpha\leq
        \frac{1}{1-\diagL(\OG)_-}$ and every $u\in \real^r$, $\Phi^\OG_{\alpha}(\cdot,u)$ is a contraction map with
        contraction factor $1-\alpha(1-\osL(\OG)_+)$;
        \item\label{p3:non-diff-perturbation} for every $u,v\in \real^r$,
          we have $\norm{x^*_u-x^*_v}{\infty,[\eta]^{-1}} \le
          \frac{\Lip_u{\Phi^\OG}}{1-\osL{\OG}_+}\norm{u-v}{\mcU}$.
        \end{enumerate}
      \end{theorem}
      \begin{proof}[Proof of Theorem~\ref{thm:non-diff}]
        Regarding
        part~\ref{p1:non-diff-existence}, the assumptions on each
        scalar activation function imply that (i)
        $\map{\Phi}{\real^n}{\real^n}$ is non-expansive with respect
        to $\norm{\cdot}{\infty,[\eta]^{-1}}$ and (ii) for every
        $p,q \in \real$, there exists $\theta_i \in [0,1]$ such that
        $\phi_i(p) - \phi_i(q) = \theta_i(p - q)$ or in the matrix
        form $\Phi(\mathbf{p})-\Phi(\mathbf{q}) = \Theta (\mathbf{p}-\mathbf{q})$ where $\Theta$ is a
        diagonal matrix with diagonal elements $\theta_i\in
        [0,1]$ and $\mathbf{p},\mathbf{q}\in \real^n$. As a result, we have
        \begin{align*}
          \|\Phi^{\OG}_{\alpha}(x_1,u)-\Phi^{\OG}_{\alpha}(x_2,u)\|_{\infty,[\eta]^{-1}}
          &= \|(1-\alpha)(x_1-x_2) + \alpha \Theta
            (\OG(x_1)-\OG(x_2))\|_{\infty,[\eta]^{-1}} \\ & \le \sup_{y\in \real^n}\|I_n + \alpha
                                                            (-I_n
                                                            +\Theta\jac{\OG}(y))\|_{\infty,[\eta]^{-1}}\|x_1-x_2\|_{\infty,[\eta]^{-1}}. 
        \end{align*}
        where the inequality holds by the mean value theorem. Then, for
        every $\alpha \in {]0,\frac{1}{1-\diagL(\Theta\jac{\OG})}]}$,
        \begin{align*}
          \|I_n + \alpha(-I_n+\Theta\jac{\OG}(y))\|_{\infty,[\eta]^{-1}} &= 1 +
          \alpha \mu_{\infty,[\eta]^{-1}}\big(-I_n+\Theta\jac{\OG}(y)\big) \\
          &
          \le
          1
          +
          \alpha
          \big(-1
          +
          \mu_{\infty,[\eta]^{-1}}(\Theta\jac{\OG}(y))\big)\\ &
          \le
          1
          +
          \alpha \big(-1+\mu_{\infty,[\eta]^{-1}}(\jac{\OG}(y))_{+}\big) \\ & \le 1 -
          \alpha (1-\osL(\OG)_{+}) < 1, 
        \end{align*}
        where the first equality holds by
        Lemma~\ref{lem:norm-mu-nonEuc}\ref{p1:identity}, the second
        inequality holds by subadditive property of matrix
        measures~\eqref{measure:subadd}, and the third
        inequality holds by
        Lemma~\ref{lemma:affine-scaling}\ref{fact:ms:1}. Moreover, since $\theta_i\in
        [0,1]$, we have $\theta_i\jac{\OG}_{ii} \ge (\jac{\OG}_{ii})_-$, for every
        $i\in \{1,\ldots,n\}$. This means that 
        \begin{align*}
          \diagL(\Theta\jac{\OG}) = \min_{i}\inf_{y\in \real^n}
          (\Theta\jac{\OG}(y))_{ii} \ge  \min_{i}\inf_{y\in \real^n}
          (\jac{\OG}_{ii}(y))_- = \diagL(\OG)_-. 
          \end{align*}
        This implies that,  for
        every $\alpha \in {]0,\frac{1}{1-\diagL(\OG)_-}]}$,
        \begin{align*}
          \|\Phi^{\OG}_{\alpha}(x_1,u)-\Phi^{\OG}_{\alpha}(x_2,u)\|_{\infty,[\eta]^{-1}}\le
          (1 - \alpha (1-\osL(\OG)_{+})) \|x_1-x_2\|_{\infty,[\eta]^{-1}}. 
        \end{align*}
        Since $1 -\alpha (1-\osL(\OG)_+) < 1$, the map
        $\Phi^{\OG}_{\alpha}(\cdot,u)$ is a contraction for every
        $\alpha \in {]0,\frac{1}{1-\diagL(\OG)_{-}}]}$. This concludes the
        proof of parts~\ref{p1:non-diff-existence}
        and~\ref{p2:non-diff-iterations},

        Regarding part~\ref{p3:non-diff-perturbation}, from
        formula~\eqref{eq:osL=WP-appendix} for the one-sided
        Lipschitz constant and
        formula~\eqref{def: WP-1+infty-appendix} for the relevant
        WP, we obtain that, for all $x_1,x_2\in\real^n$,
        \begin{align*}
          &\WP{\Phi(\OG(x_1) + Bu) - \Phi(\OG(x_2) + Bu)}{x_1 - x_2}_{\infty,[\eta]^{-1}} \\
          &\qquad\qquad = \max_{i \in \Iinfty([\eta]^{-1}(x_1-x_2))} \eta_i^{-2} (x_{1} - x_{2})_i(\phi_i((\OG(x_1) + Bu)_i) - \phi_i((\OG(x_2) + Bu)_i)) \\
          &\qquad\qquad = \max_{i \in \Iinfty([\eta]^{-1}(x_1-x_2))} \theta_i\eta_i^{-2} (x_{1} - x_{2})_i((\OG(x_1) + Bu)_i - (\OG(x_2) + Bu)_i) \\
          &\qquad\qquad = \max_{i \in \Iinfty([\eta]^{-1}(x_1-x_2))} \theta_i\eta_i^{-2} (x_{1} - x_{2})_i(\OG(x_1) - \OG(x_2))_i,
        \end{align*}
        Next, we recall Lumer's
        equality~\eqref{eq:Lumer-infty-appendix} and write it as
        \begin{equation*}
          \osL(\OG) = \sup_{x_1 \neq x_2} \max_{i \in \Iinfty([\eta]^{-1}(x_1-x_2))} \eta_i^{-2}(x_{1} - x_{2})_i(\OG(x_1)- \OG(x_2))_i.
        \end{equation*}
        Next, we consider two cases. Suppose that $\osL(\OG)\leq 0$.
        Since $\theta_i \in [0,1]$ for all $i$, we obtain
        \begin{equation*}
          \WP{\Phi(\OG(x_1) + Bu) - \Phi(\OG(x_2) + Bu)}{x_1 - x_2}_{\infty,[\eta]^{-1}} \leq 0,
        \end{equation*}
        since the maximum value is achieved at $\theta_i=0$ for all
        $i$.  Alternatively, suppose that $\osL(\OG) > 0$. Then
        \begin{align*}
          &\WP{\Phi(\OG(x_1) + Bu) - \Phi(\OG(x_2) + Bu)}{x_1 - x_2}_{\infty,[\eta]^{-1}} \\
          &\qquad\qquad = \max_{i \in \Iinfty([\eta]^{-1}(x_1-x_2))} \theta_i\eta_i^{-2} (x_{1}
            - x_{2})_i(\OG(x_1) - \OG(x_2))_i \\
          &\qquad\qquad \leq \max_{i \in  \Iinfty([\eta]^{-1}(x_1-x_2))} \eta_i^{-2} (x_{1} - x_{2})_i(\OG(x_1) -
            \OG(x_2))_i\leq \osL(\OG)\|x_1 - x_2\|_{\infty,[\eta]^{-1}}^2,
        \end{align*}
        since the maximum value is achieved at $\theta_i=1$ for all
        $i$. This means that $\osL(\Phi^{\OG})= \osL(\OG)_{+}$. Now we
        compute
        \begin{align*}
          \norm{x^*_u-x^*_v}{\infty,[\eta]^{-1}}^2 &= \WP{x^*_u-x^*_v}{x^*_u-x^*_v}_{\infty,[\eta]^{-1}}  \\
                                                   & =   \WP{\Phi^\OG_u(x^*_u)-\Phi^\OG_v(x^*_v)}{x^*_u-x^*_v}_{\infty,[\eta]^{-1}} \\
                                                   & \le
                                                     \WP{\Phi^\OG_u(x^*_u)-\Phi^\OG_u(x^*_v)}{x^*_u-x^*_v}_{\infty,[\eta]^{-1}} +
                                                     \WP{\Phi^\OG_u(x^*_v)-\Phi^\OG_v(x^*_v)}{x^*_u-x^*_v}_{\infty,[\eta]^{-1}}  \\
                                                   &\le
                                                     \osL(\OG)_{+}\norm{x^*_u-x^*_v}{\infty,[\eta]^{-1}}^2 +
                                                     \norm{\Phi^\OG_u(x^*_v)-\Phi^\OG_v(x^*_v)}{\infty,[\eta]^{-1}} \norm{x^*_u-x^*_v}{\infty,[\eta]^{-1}}  \\
                                                   & \le \osL(\OG)_{+} \norm{x^*_u-x^*_v}{\infty,[\eta]^{-1}}^2 +
                                                     \Lip_u(\Phi^\OG) \norm{u-v}{\mcU} \norm{x^*_u-x^*_v}{\infty,[\eta]^{-1}}.
        \end{align*}
        This implies that
        $(1-\osL(\OG)_{+})\norm{x^*_u-x^*_v}{\infty,[\eta]^{-1}}\le
        \Lip_u(\Phi^\OG)\norm{u-v}{\mcU}$ and the result
        follows.\end{proof}

\subsection{Proofs of results in Section~\ref{sec:inn}}

\begin{proof}[Proof of Theorem~\ref{thm:inn-Lip-osL}]
  The assumptions on each scalar activation function imply that (i)
  $\map{\Phi}{\real^n}{\real^n}$ is non-expansive with respect to
  $\norm{\cdot}{\infty,[\eta]^{-1}}$, and (ii) for every $p,q \in \real$,
  there exists $\theta_i \in [0,1]$ such that $\phi_i(p) - \phi_i(q) =
  \theta_i(p - q)$.  Regarding the equality $\osL_{x}(\ON) = \mu_{\infty,[\eta]^{-1}}(A)_+$, from
  formula~\eqref{eq:osL=WP-appendix} for the one-sided Lipschitz constant and
  formula~\eqref{def: WP-1+infty-appendix} for the relevant WP, we obtain that,
  for all $x_1,x_2\in\real^n$,
 \begin{align*}
   &\WP{\Phi(Ax_1 + Bu) - \Phi(Ax_2 + Bu)}{x_1 - x_2}_{\infty,[\eta]^{-1}} \\
   &\qquad\qquad = \max_{i \in \Iinfty([\eta]^{-1}(x_1-x_2))} \eta_i^{-2} (x_{1} - x_{2})_i(\phi_i((Ax_1 + Bu)_i) - \phi_i((Ax_2 + Bu)_i)) \\
   &\qquad\qquad = \max_{i \in \Iinfty([\eta]^{-1}(x_1-x_2))} \theta_i\eta_i^{-2} (x_{1} - x_{2})_i((Ax_1 + Bu)_i - (Ax_2 + Bu)_i) \\
   &\qquad\qquad = \max_{i \in \Iinfty([\eta]^{-1}(x_1-x_2))} \theta_i\eta_i^{-2} (x_{1} - x_{2})_i(Ax_1 - Ax_2)_i,
 \end{align*}
 Next, we recall Lumer's equality~\eqref{eq:Lumer-infty-appendix} and write it as
 \begin{equation*}
   \mu_{\infty,[\eta]^{-1}}(A) = \sup_{x_1 \neq x_2} \max_{i \in \Iinfty([\eta]^{-1}(x_1-x_2))} \eta_i^{-2}(x_{1} - x_{2})_i((Ax_1)_i - (Ax_2)_i).
 \end{equation*}
 Next, we consider two cases. Suppose that $\mu_{\infty,[\eta]^{-1}}(A) \leq 0$.
 Since $\theta_i \in [0,1]$ for all $i$, we obtain
 \begin{equation*}
   \WP{\Phi(Ax_1 + Bu) - \Phi(Ax_2 + Bu)}{x_1 - x_2}_{\infty,[\eta]^{-1}} \leq 0,
 \end{equation*}
 since the maximum value is achieved at $\theta_i=0$ for all $i$.
 Alternatively, suppose that $\mu_{\infty,[\eta]^{-1}}(A) > 0$. Then
 \begin{align*}
   &\WP{\Phi(Ax_1 + Bu) - \Phi(Ax_2 + Bu)}{x_1 - x_2}_{\infty,[\eta]^{-1}} \\
   &\qquad\qquad = \max_{i \in \Iinfty([\eta]^{-1}(x_1-x_2))} \theta_i\eta_i^{-2} (x_{1}
   - x_{2})_i(Ax_1- Ax_2)_i\\
   &\qquad\qquad \leq \max_{i \in  \Iinfty([\eta]^{-1}(x_1-x_2))} \eta_i^{-2} (x_{1} - x_{2})_i(Ax_1 -
   Ax_2)_i \leq \mu_{\infty,[\eta]^{-1}}(A)\|x_1 - x_2\|_{\infty,[\eta]^{-1}}^2,
 \end{align*}
 since the maximum value is achieved at $\theta_i=1$ for all $i$.  This
 concludes the proof of formula $\osL_{x}(\ON) = \mu_{\infty,[\eta]^{-1}}(A)_+$.  Next, since $\Phi$ is
 non-expansive, we compute 
 \begin{align*}
   \norm{\ON(x_1,u)-\ON(x_2,u)}{\infty,[\eta]^{-1}} & =
   \norm{\Phi(Ax_1+Bu)-\Phi(Ax_2+Bu)}{\infty,[\eta]^{-1}} \\
   & \leq
   \norm{(Ax_1+Bu)-(Ax_2+Bu)}{\infty,[\eta]^{-1}}\\
   &\leq
   \norm{A(x_1-x_2)}{\infty,[\eta]^{-1}} \leq \norm{A}{\infty,[\eta]^{-1}} \norm{x_1-x_2}{\infty,[\eta]^{-1}},
 \end{align*}
 proving the formula $\Lip_{x}(\ON) = \norm{A}{\infty,[\eta]^{-1}}$.
 The proof of the formula $\Lip_{u}(\ON) = \norm{B}{(\infty,[\eta]^{-1}),\mcU}$ is essentially identical.
 Finally, if each $\phi_i$ is differentiable then we compute
  \begin{align}
    \diagL(\ON)
    &=\min_{i\in\until{n}}\inf_{x\in \real^n u\in\real^r}\jac{\ON}_{ii}(x,u)
    =\min_{i\in\until{n}}\inf_{x\in \real^n u\in\real^r} \phi_i'((Ax+Bu)_i) A_{ii}\nonumber\\
    &\le \min_{i\in\until{n}}\begin{cases} 0, \quad &\text{if } A_{ii} >0 \\ A_{ii},\quad &\text{if } A_{ii} \leq 0
    \end{cases} \enspace = \min_{i\in\until{n}}(A_{ii})_- ,
  \end{align}
  because of the properties of the activation
  functions. Now suppose that there exists $i\in \{1,\ldots,n\}$ such
  that $\phi_i$ is not differentiale. Using
  Theorem~\ref{thm:non-diff}\ref{p2:non-diff-iterations} with $\OG=A$,
  $\diagL(\ON)$ is chosen to be equal to be $\diagL(A)_-$ which
  in turn is equal to $\min_{i\in\until{n}}(A_{ii})_-$. \end{proof}

  %\sashamargin{Is this an equality or an upper bound? I
  %   don't see how the first equality in line (34) holds.}
  % \sabertodo{add proof for non-diff}

\begin{proof}[Proof of Corollary~\ref{corollary:inn-properties}]
  The results are immediate consequences of
  Theorem~\ref{thm:acceleration-1-inf} (or more generally Theorem~\ref{thm:non-diff} for
  non-differentiable activation functions) and of the Lipschitz estimates in
  Theorem~\ref{thm:inn-Lip-osL}.
\end{proof}

% \section{Comparison with respect to adversarial attacks on different
%   models}\label{app:morenumerics}

\section{Adversarial attacks on implicit neural networks}\label{app:morenumerics}

In this appendix, we study the effect of different adversarial attacks
on the existing implicit network models as well as to the NEMON model. 

\subsection{Attack models}
First, we review several attack models that are used in the
literature to study the input-output resilience of neural
networks. Each attack consists of a model for generating
  suitable perturbations of the test input data. Perturbations with respect to these attacks were generated using the Foolbox software package\footnote{The Foolbox implementation is licensed under the MIT License and is available at
  	\tt{https://github.com/bethgelab/foolbox}.}.

\paragraph{Continuous image inversion.}
\begin{figure}[ht]\centering
	\begin{tabular}{c}
		\includegraphics[width = 0.9\linewidth,clip]{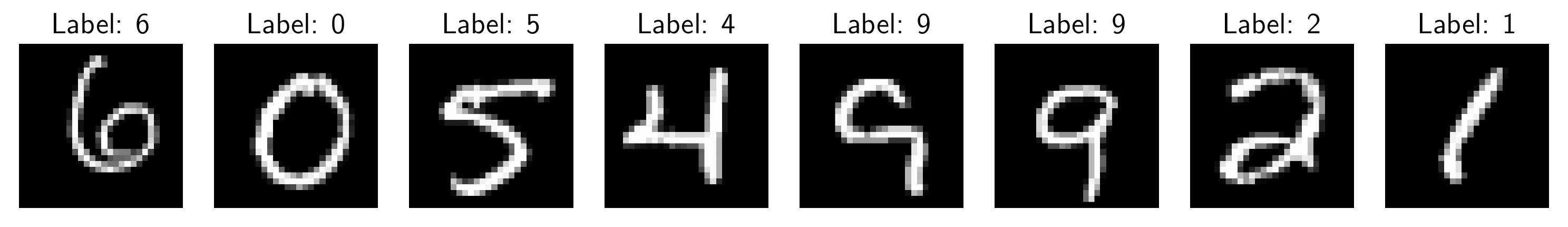} \\
		\includegraphics[width = 0.9\linewidth,clip]{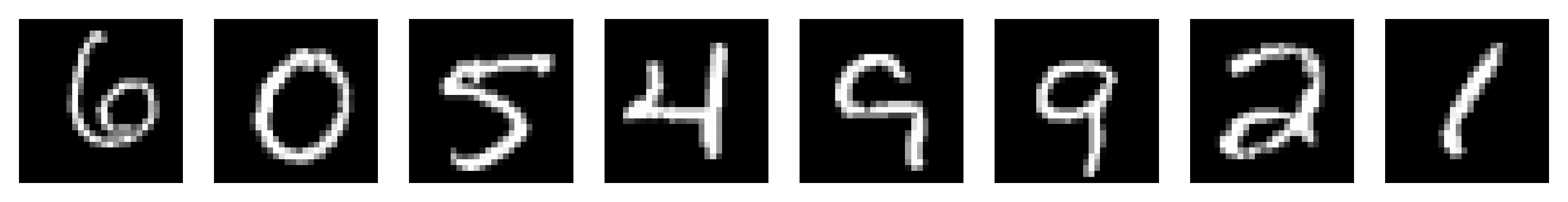} \\
		\includegraphics[width = 0.9\linewidth,clip]{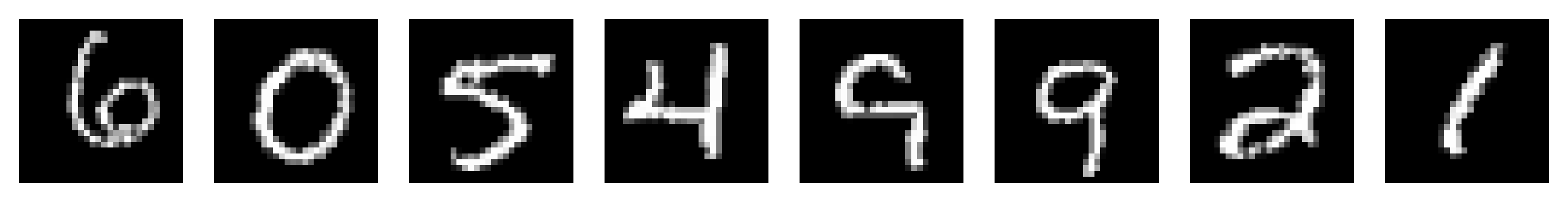} \\
		\includegraphics[width = 0.9\linewidth,clip]{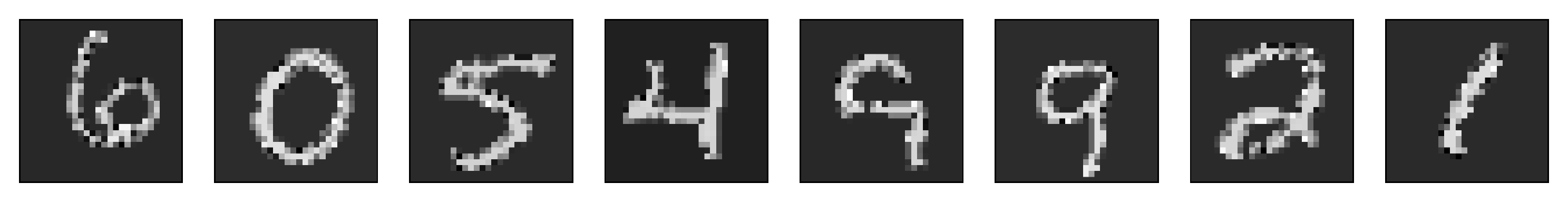} \\
		\includegraphics[width = 0.9\linewidth,clip]{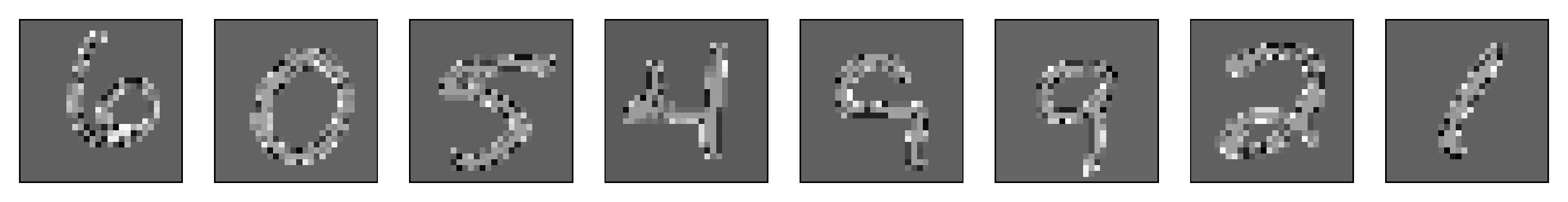} \\
	\end{tabular}
	\caption{Images of MNIST handwritten digits perturbed by the
          continuous image inversion attack. For $i\in
          \{1,\ldots,5\}$, row $i$ corresponds to an $\ell_{\infty}$ perturbation amplitude $\varepsilon=0.1 \times (i-1)$. In other words, the top row has unperturbed images, the second row has images that is perturbed by an $\ell_\infty$ amplitude $\varepsilon =0.1$, etc.}\label{fig:perturbedInversion}
\end{figure}
 % \sabermargin{Quick question about the image perturbation figure: each row ``add''
 %            perturbation to the previous image? or row $i$
 %            corresponde to a perturbation of magnitude $0.1\times i$?}
The continuous image inversion attack is defined by:
\begin{equation}
\subscr{U}{adversarial} = U + \varepsilon \sign\Big(\tfrac{1}{2}\vect{1}_r\vect{1}_m^{\top}-U\Big).
\end{equation}
It is clear that this attack is independent of the neural network
model. Plots of perturbed MNIST images under the continuous image
inversion attack are shown in Figure~\ref{fig:perturbedUniform}. In
Figure~\ref{fig:robustness}, the right plot compares the accuracy of
the NEMON model, the implicit deep learning
model~\citep{LEG-FG-BT-AA-AYT:21}, and the MON model~\citep{EW-JZK:20}
for $\varepsilon\in [0.0.5]$.

\paragraph{Uniform additive $\ell_\infty$-noise.}
\begin{figure}[ht]\centering
	\begin{tabular}{c}
		\includegraphics[width = 0.9\linewidth,clip]{images/Unperturbed} \\
		\includegraphics[width = 0.9\linewidth,clip]{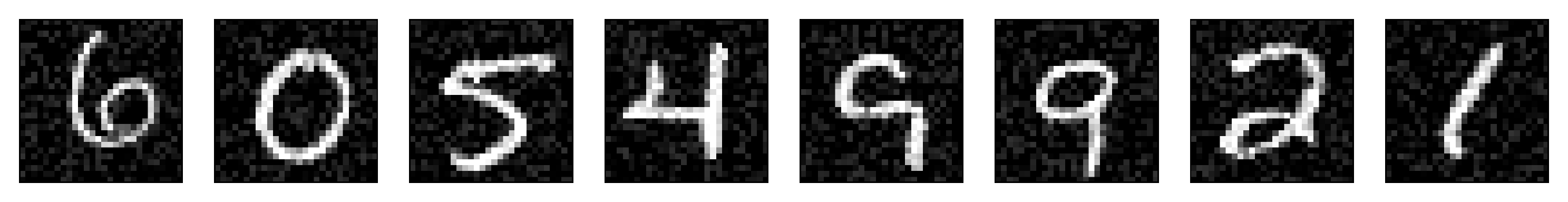} \\
		\includegraphics[width = 0.9\linewidth,clip]{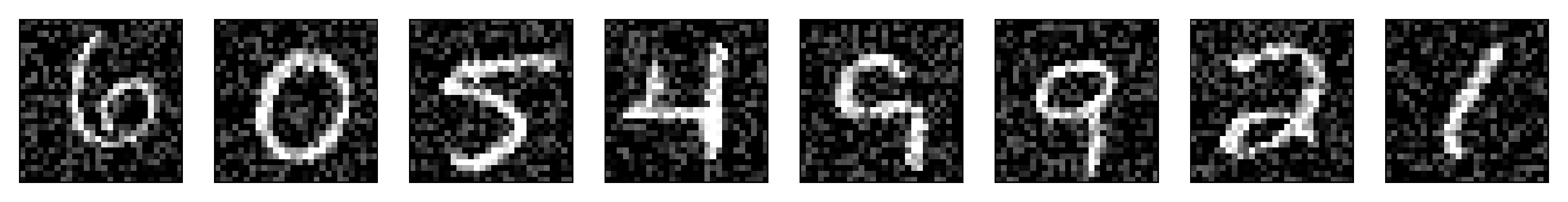} \\
		\includegraphics[width = 0.9\linewidth,clip]{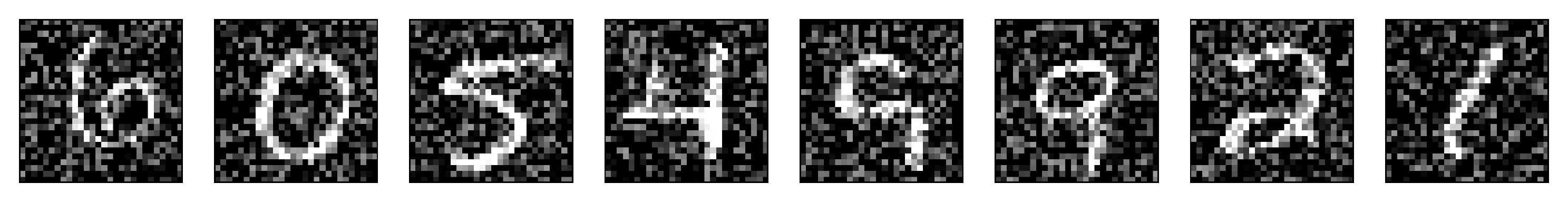} \\
		\includegraphics[width = 0.9\linewidth,clip]{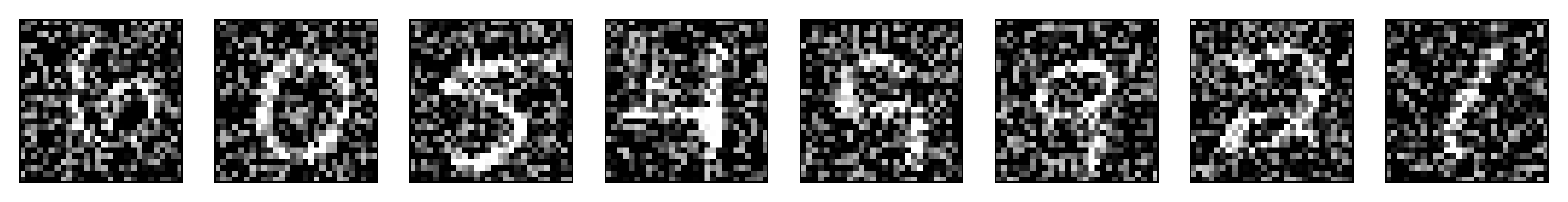} \\
	\end{tabular}
	\caption{Images of MNIST handwritten digits as perturbed by uniform additive $\ell_{\infty}$ noise. For $i\in
          \{1,\ldots,5\}$, row $i$ corresponds to an $\ell_{\infty}$ perturbation amplitude $\varepsilon=0.2 \times (i-1)$. In other words, the top row has unperturbed images, the second row has images that is perturbed by an $\ell_\infty$ amplitude $\varepsilon =0.2$, etc.}\label{fig:perturbedUniform}
\end{figure}
\begin{figure}[ht]\centering
	\begin{tabular}{cc}
		\includegraphics[width = 0.482\linewidth,clip]{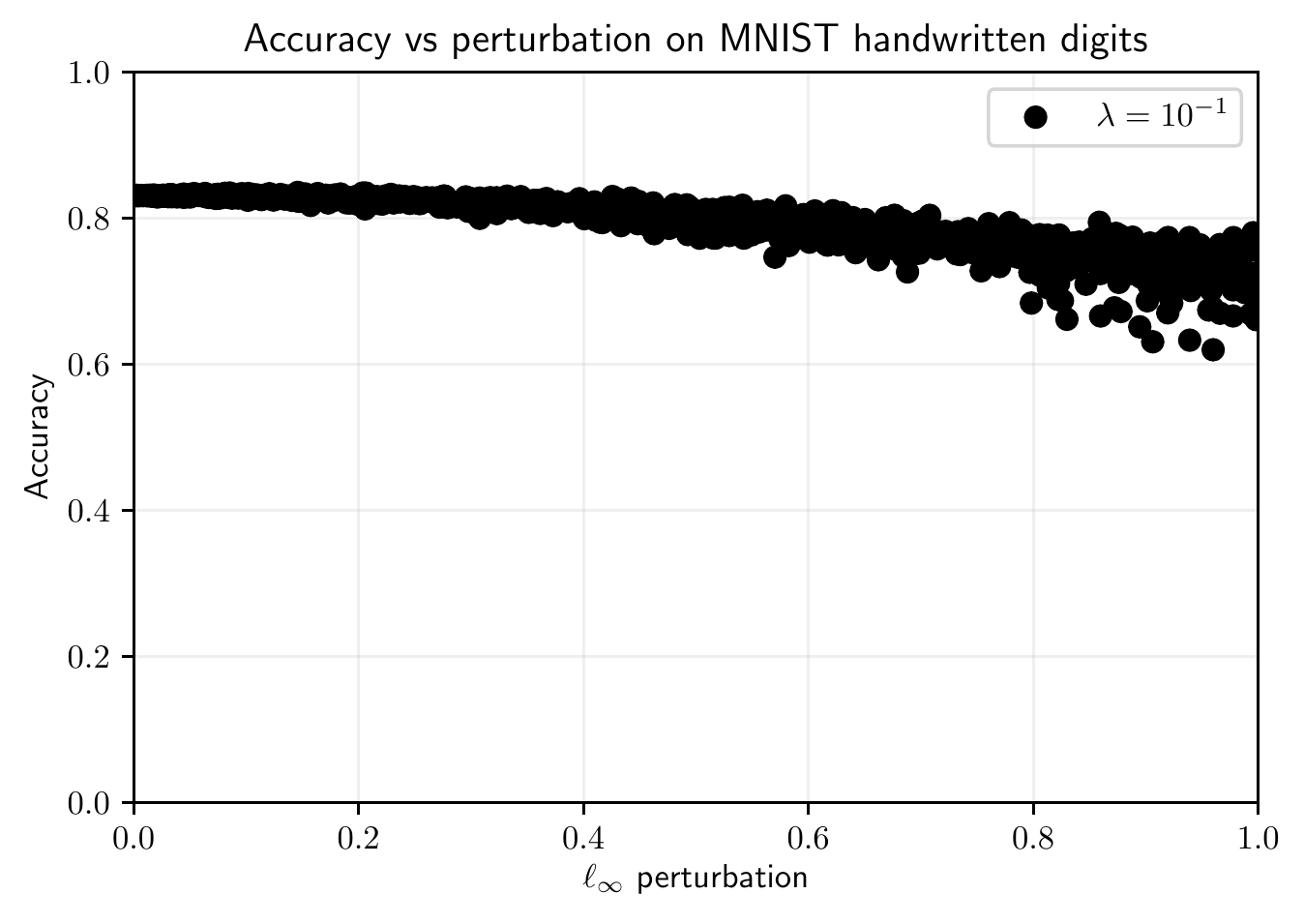}&
		\includegraphics[width = 0.482\linewidth,clip]{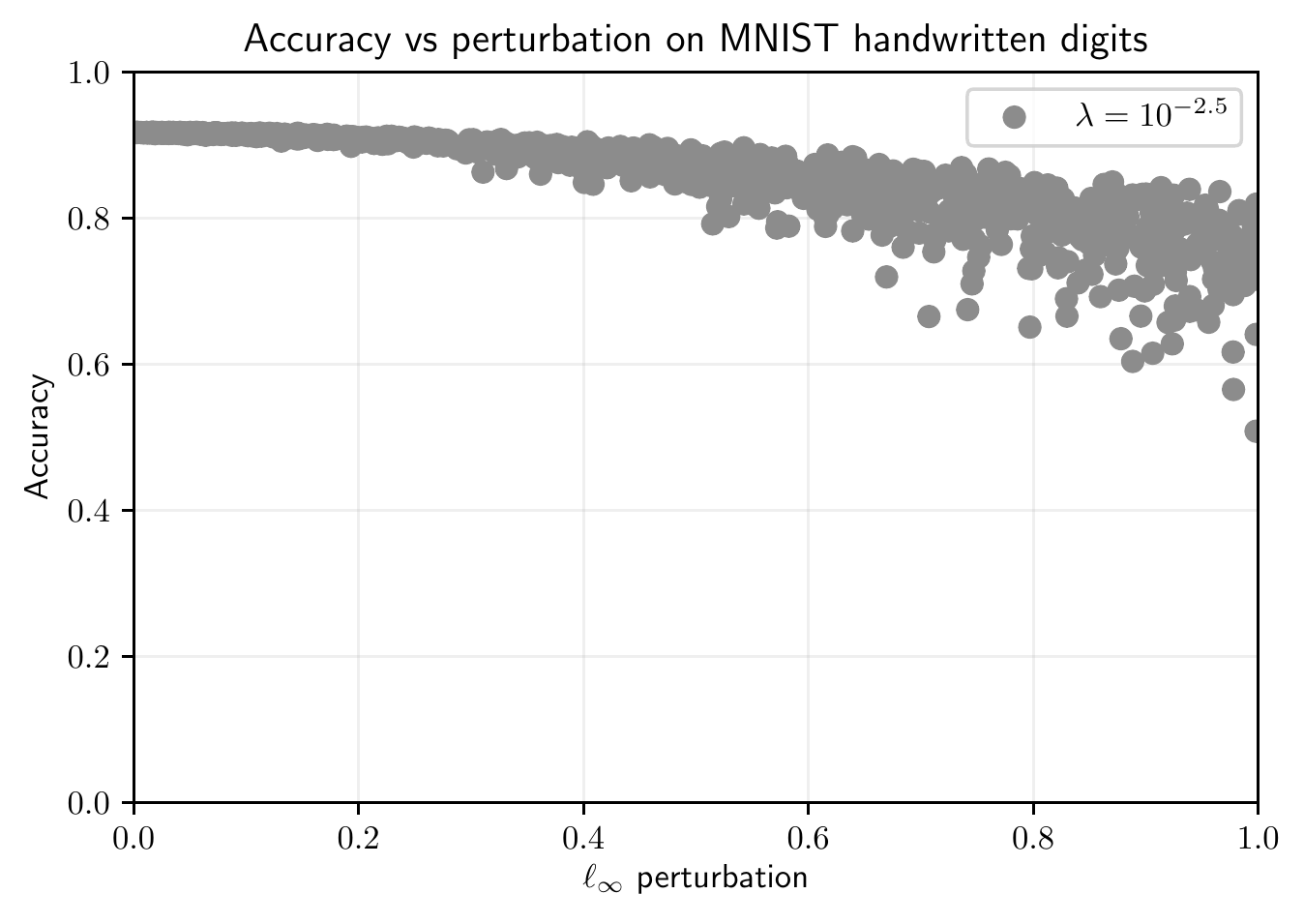} \\
		\includegraphics[width = 0.482\linewidth,clip]{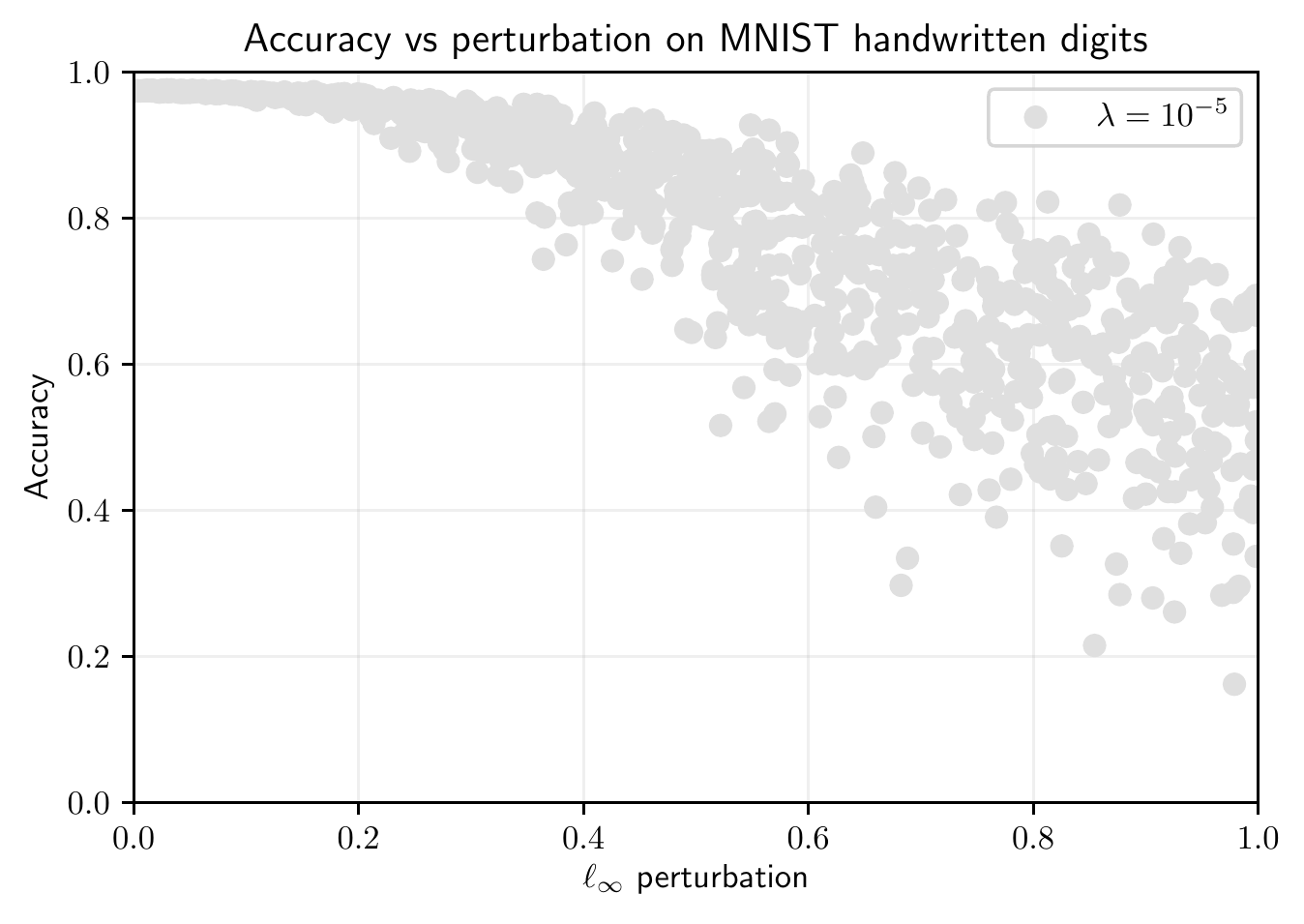}&
		\includegraphics[width = 0.482\linewidth,clip]{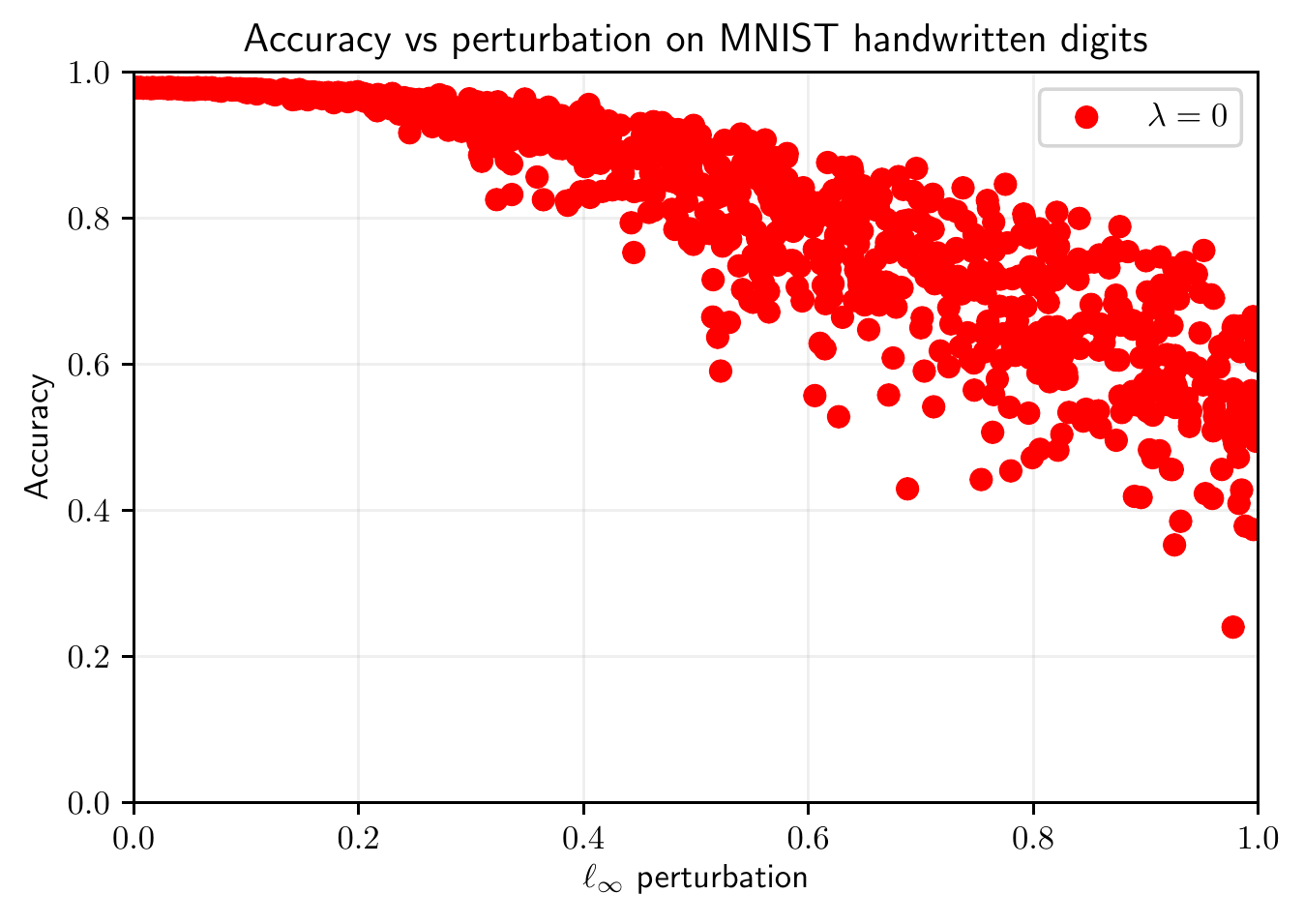} \\
		\includegraphics[width = 0.482\linewidth,clip]{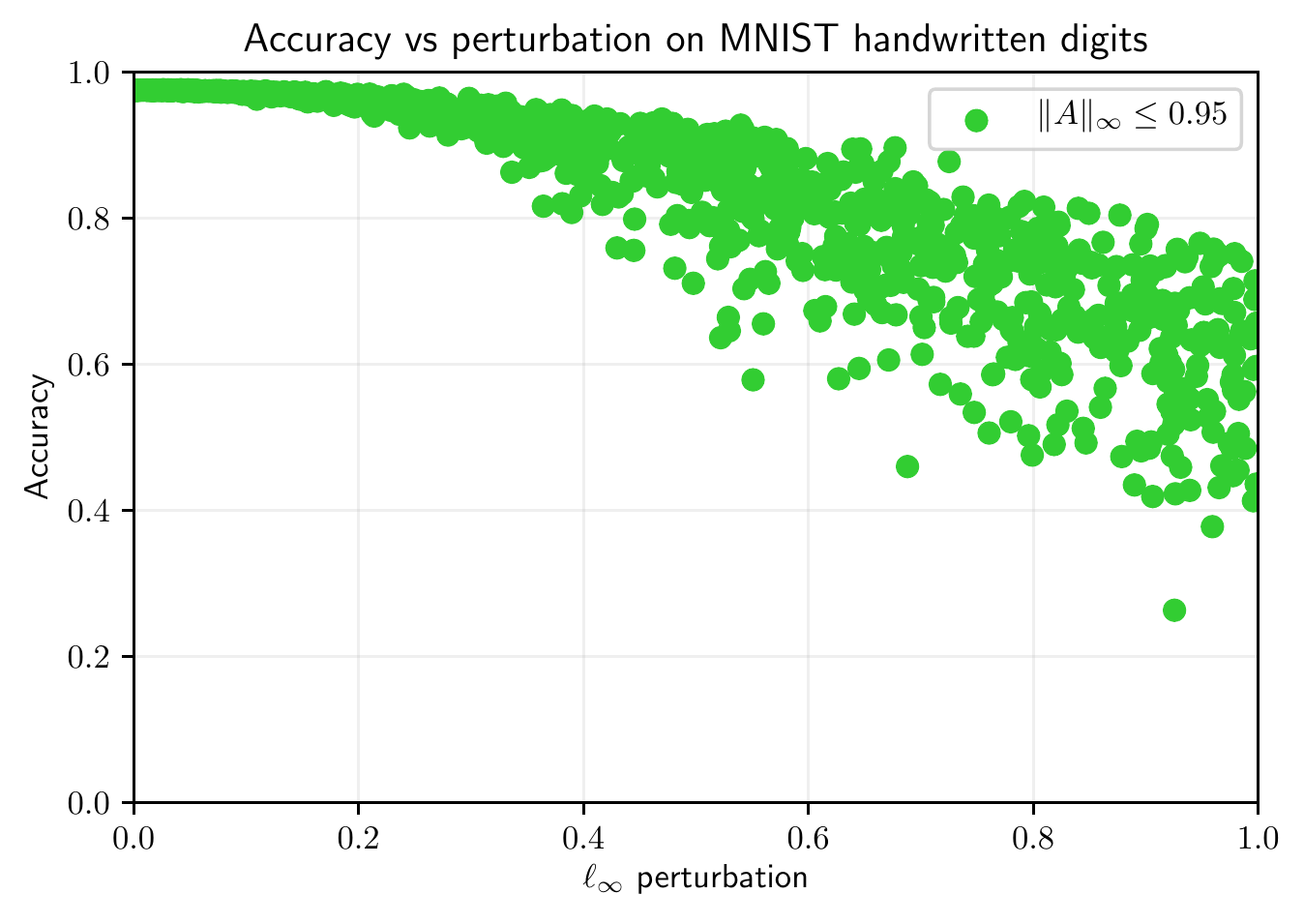}&
		\includegraphics[width = 0.482\linewidth,clip]{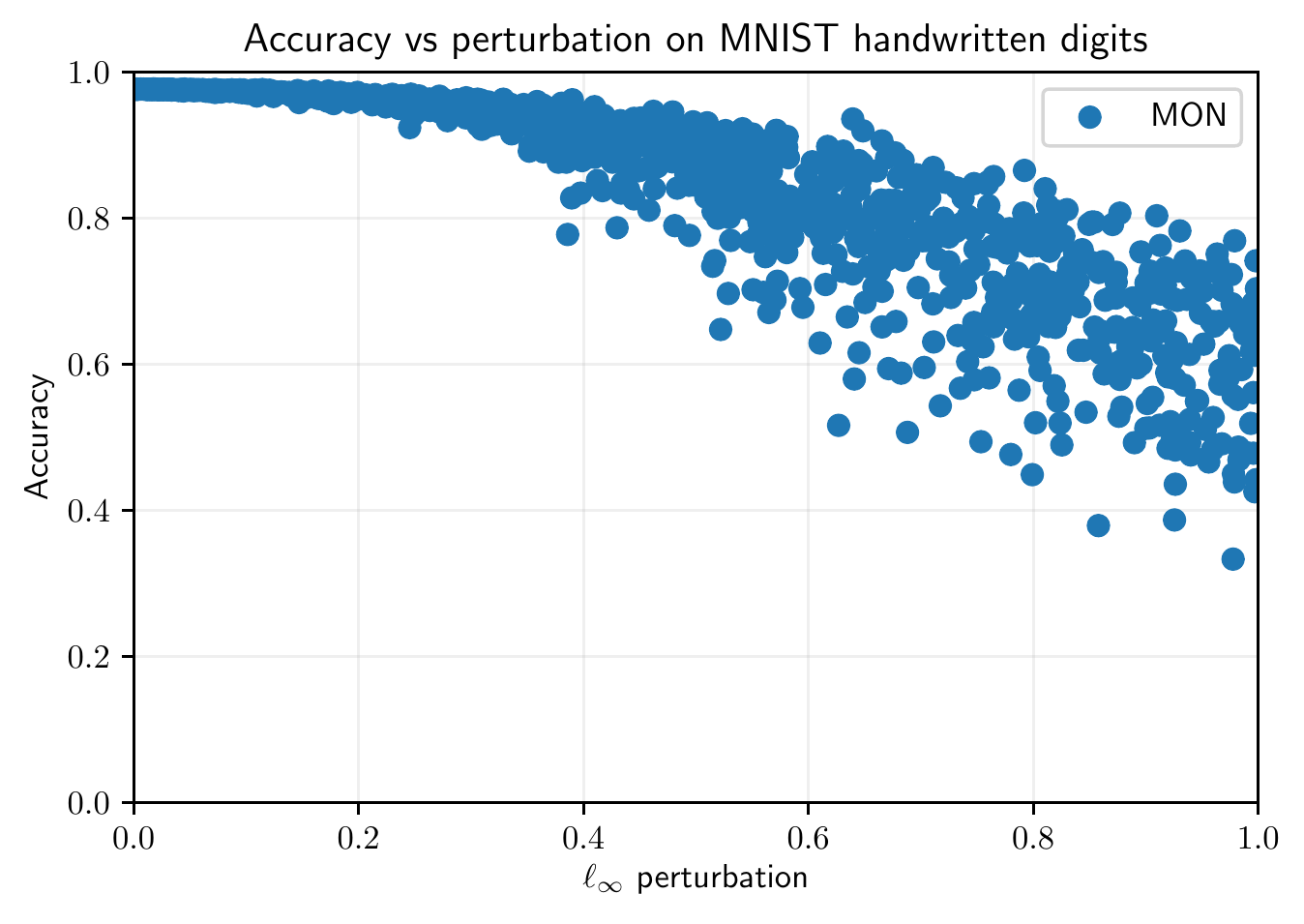}
	\end{tabular}
	\caption{Scatter plots of accuracy versus $\ell_{\infty}$
          perturbation as generated by uniform additive $\ell_\infty$
          noise over 1000 trials. Plots are shown for the NEMON model $\mu_\infty(A) \leq 0.95$ with $\lambda \in \{10^{-1}, 10^{-2.5}, 10^{-5}, 0\}$, the implicit deep learning model $\|A\|_{\infty} \leq 0.95$, and the monotone operator equilibrium network (MON) with $I_n-\frac{1}{2}(A+A^{\top})\succeq 0.05I_n$.}\label{fig:uniformNoise}
\end{figure}
For this attack, the test images are perturbed by an additive noise
with $\ell_\infty$ magnitude sampled uniformly from the interval
$[0,1]$. Plots of perturbed MNIST images under uniform additive
$\ell_\infty$-noise are shown in
Figure~\ref{fig:perturbedUniform}. Figure~\ref{fig:uniformNoise} shows
scatter plots of the accuracy of the NEMON model, the implicit deep learning
model, and the MON model over $1000$ sample attacks.

\paragraph{Fast gradient sign method.}
\begin{figure}[ht]\centering
	\begin{tabular}{c}
		\includegraphics[width = 0.9\linewidth,clip]{images/Unperturbed} \\
		\includegraphics[width = 0.9\linewidth,clip]{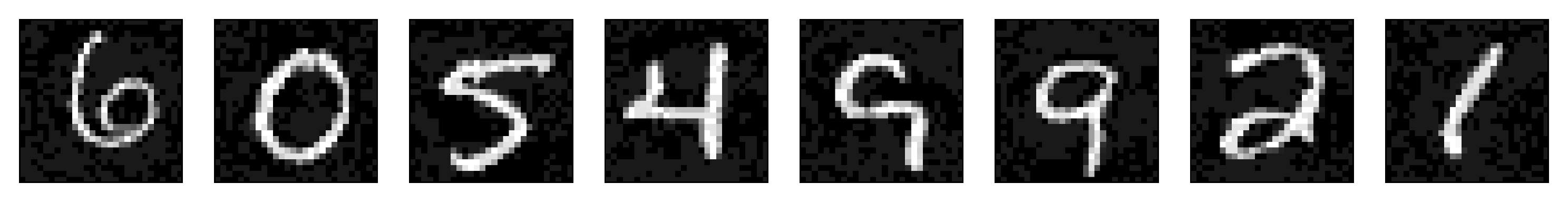} \\
		\includegraphics[width = 0.9\linewidth,clip]{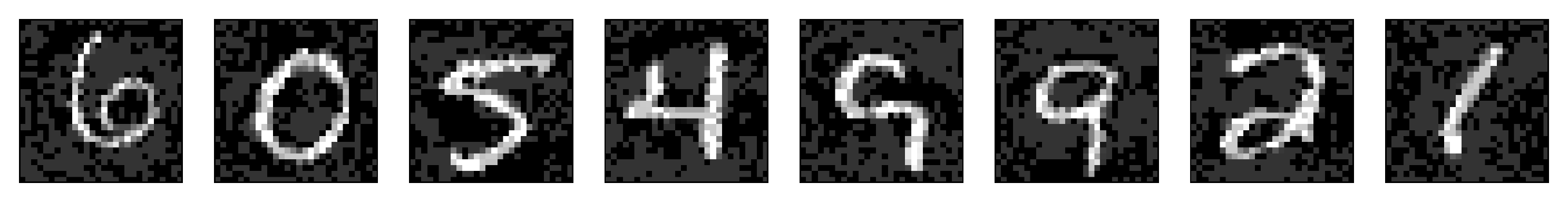} \\
		\includegraphics[width = 0.9\linewidth,clip]{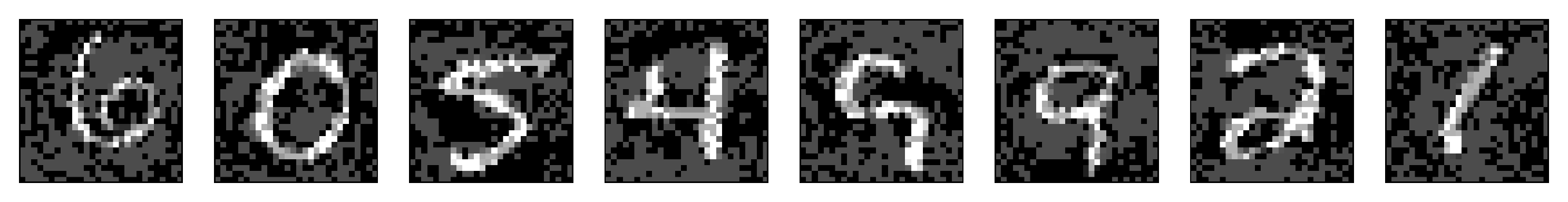} \\
		\includegraphics[width = 0.9\linewidth,clip]{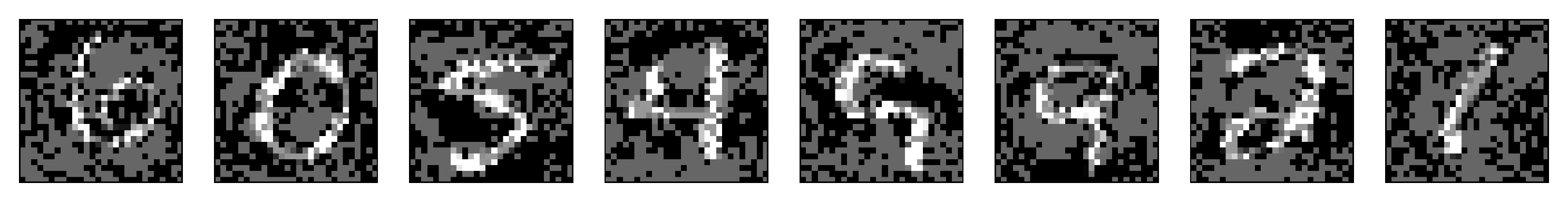} \\
	\end{tabular}
	\caption{Images of MNIST handwritten digits as perturbed by the FGSM. For $i\in
          \{1,\ldots,5\}$, row $i$ corresponds to an $\ell_{\infty}$ perturbation amplitude $\varepsilon=0.1 \times (i-1)$. In other words, the top row has unperturbed images, the second row has images that is perturbed by an $\ell_\infty$ amplitude $\varepsilon =0.1$, etc.}\label{fig:perturbedFGSM}
\end{figure}
\begin{figure}[ht]
	\centering
	\includegraphics[width=0.75\linewidth]{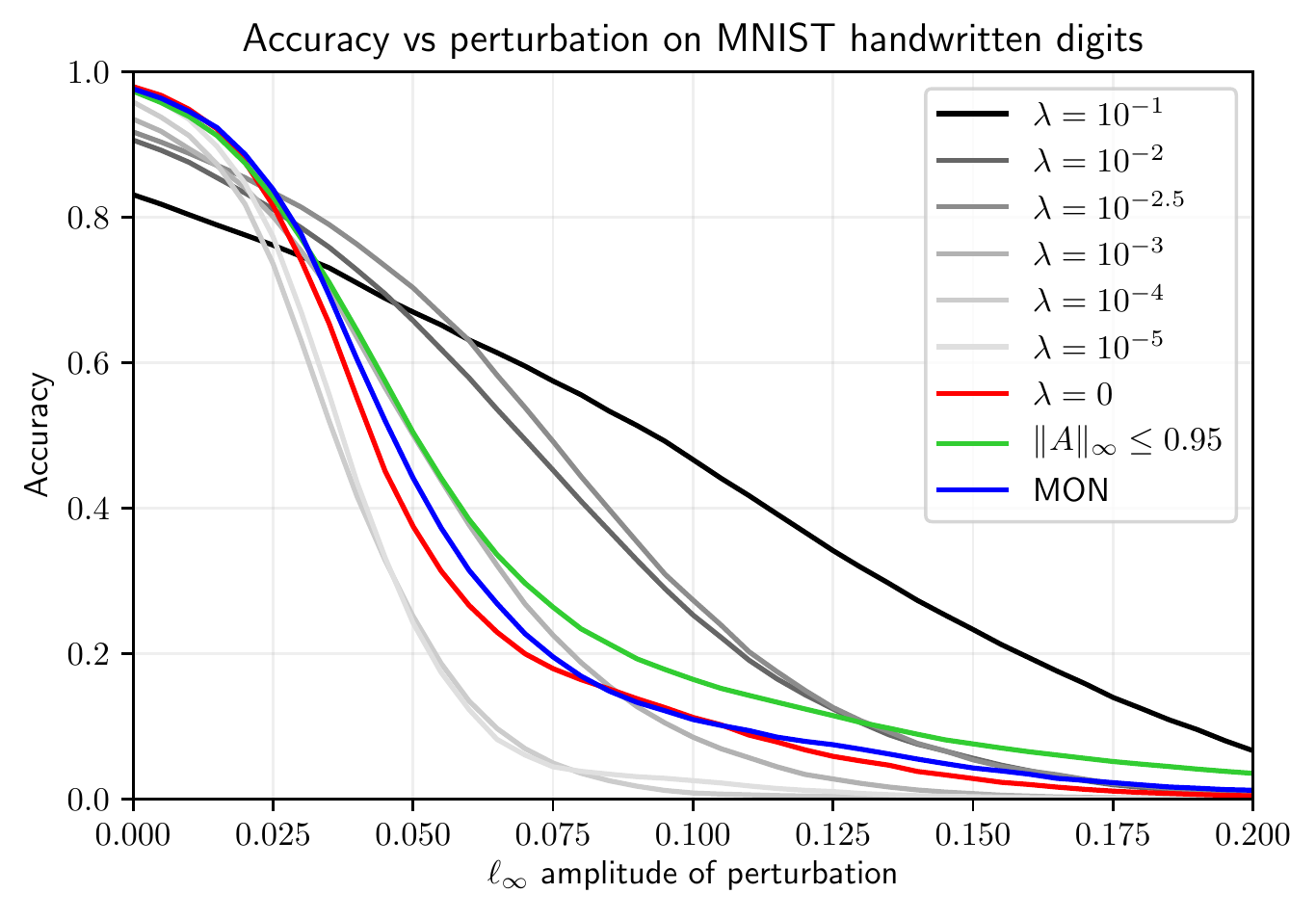}
	\caption{Plot of accuracy versus $\ell_\infty$ perturbation as
          generated by the FGSM for the NEMON model with $\mu_\infty(A) \leq
          0.95$, the implicit deep learning model with $\|A\|_\infty \leq
          0.95$, and MON with
          $I_n-\frac{1}{2}(A+A^{\top})\succeq 0.05I_n$. }\label{fig:FGSM}
\end{figure}
Given input data $U\in \real^{r\times m}$ and output labels $Y\in \real^{q\times m}$, the fast gradient sign method
(FGSM) generates adversarial inputs via the formula 
\begin{equation}
\subscr{U}{adversarial} = U + \varepsilon \sign\Big(\frac{\partial\mathcal{L}}{\partial U}(Y,CX+DU)\Big),
\end{equation}
where $\mathcal{L}$ is the loss function used to train the network and $\varepsilon$ provides the $\ell_{\infty}$ amplitude of the perturbation. Plots of perturbed MNIST images under the FGSM are shown in Figure~\ref{fig:perturbedFGSM}. Plots of accuracy versus $\ell_\infty$ perturbation under the FGSM are shown in Figure~\ref{fig:FGSM}.

\paragraph{Projected gradient descent method.}
\begin{figure}[ht]\centering
	\begin{tabular}{c}
		\includegraphics[width = 0.9\linewidth,clip]{images/Unperturbed} \\
		\includegraphics[width = 0.9\linewidth,clip]{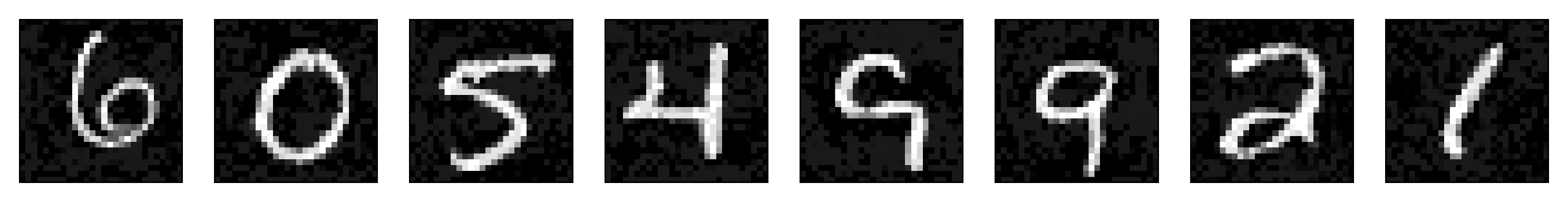} \\
		\includegraphics[width = 0.9\linewidth,clip]{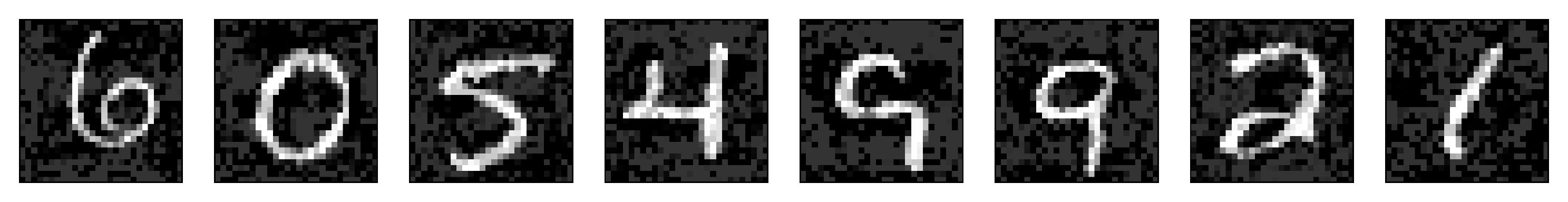} \\
		\includegraphics[width = 0.9\linewidth,clip]{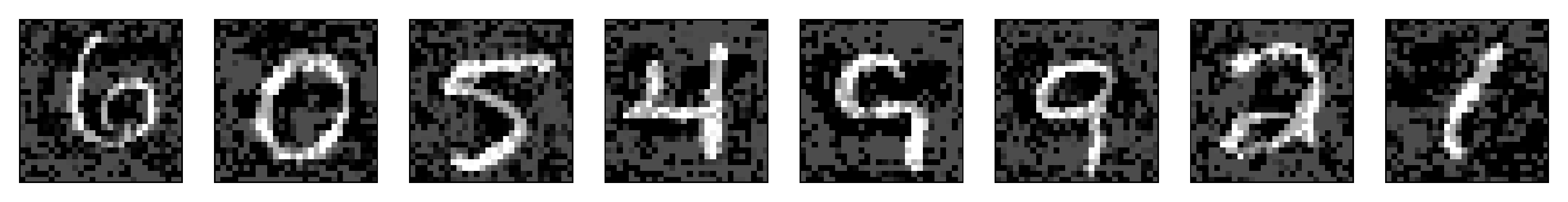} \\
		\includegraphics[width = 0.9\linewidth,clip]{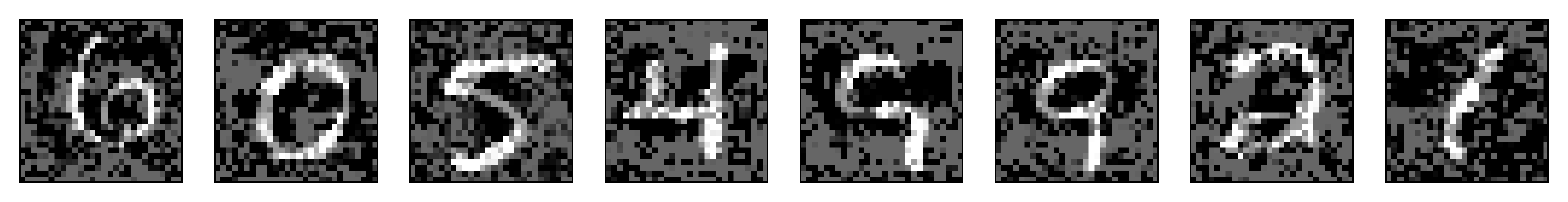} \\
	\end{tabular}
	\caption{Images of MNIST handwritten digits as perturbed by the PGDM. For $i\in
          \{1,\ldots,5\}$, row $i$ corresponds to an $\ell_{\infty}$ perturbation amplitude $\varepsilon=0.1 \times (i-1)$. In other words, the top row has unperturbed images, the second row has images that is perturbed by an $\ell_\infty$ amplitude $\varepsilon =0.1$, etc.}\label{fig:perturbedPGD}
\end{figure}
\begin{figure}[ht]
	\centering
	\includegraphics[width=0.75\linewidth]{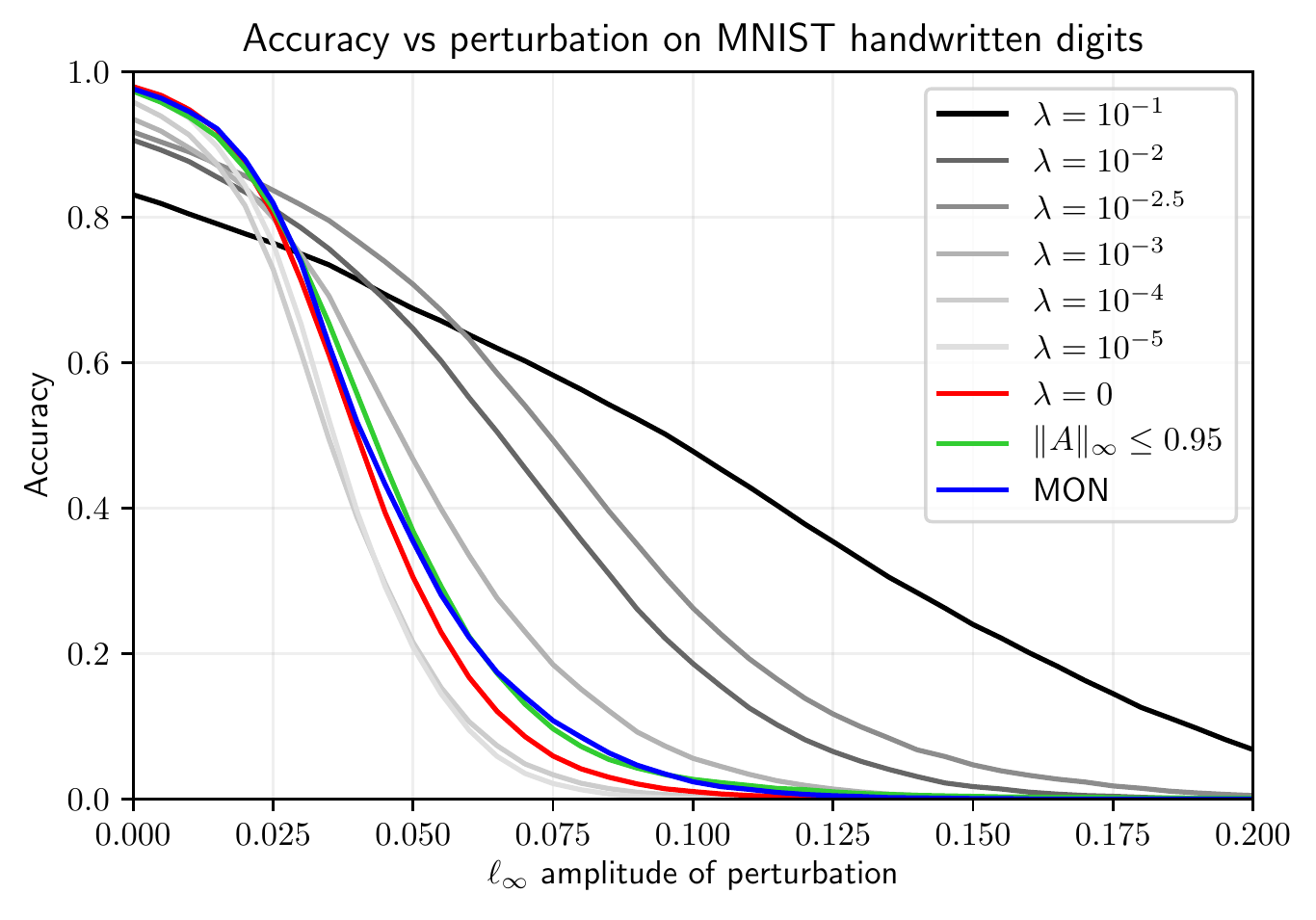}
	\caption{Plot of accuracy versus $\ell_\infty$ perturbation as
		generated by the PGDM for the NEMON model with $\mu_\infty(A) \leq
		0.95$, the implicit deep learning model with $\|A\|_\infty \leq
		0.95$, and MON with
		$I_n-\frac{1}{2}(A+A^{\top})\succeq 0.05I_n$. }\label{fig:PGDM}
\end{figure}
The projected gradient descent method (PGDM) can be thought of
as perturbing the input with several steps of the FGSM. The
PGDM attack can be defined for any norm, but for consistency, we reproduce it only for
  the $\ell_{\infty}$-norm. For the input data $U\in \real^{r\times
    m}$ and outputs $Y\in \real^{q\times m}$, PGDM defines the finite sequence of perturbations $\{\delta_k\}_{k=1}^M$ by 
\begin{equation}
\delta_{k+1} = \proj_{\overline{\mathcal{B}(\varepsilon)}}\Bigg(\delta_k + \alpha \sign\Big(\frac{\partial \mathcal{L}}{\partial U}(Y,CX + D(U + \delta_k)\Big)\Bigg),
\end{equation}
where $M$ is some prescribed maximum number of steps, $\alpha$ is a stepsize, and
$\proj_{\overline{\mathcal{B}(\varepsilon)}}$ is the $\ell_2$
  orthogonal projection operator onto the entrywise $\ell_{\infty}$ closed ball with radius
  $\varepsilon$. This projection operator corresponds to clipping each
  entry of the matrix so that it is in the range
  $[-\varepsilon,\varepsilon]$. Then, the perturbed input is simply
  \begin{equation*}
  \subscr{U}{adversarial} = U + \delta_{M}.
  \end{equation*}
  Plots of perturbed MNIST images under the PGDM are shown in Figure~\ref{fig:perturbedPGD}.
  Plots of accuracy versus $\ell_\infty$ perturbation under the PGDM are shown in Figure~\ref{fig:PGDM}.

\subsection{Other methods to decrease the $\ell_\infty$ Lipschitz constant}
Recall that the input-output Lipschitz constant of the model~\eqref{eq:inn} with both $\|\cdot\|_{\mcU}$ and $\|\cdot\|_{\mcY}$ equal to the $\ell_{\infty}$-norm is given by
\begin{equation*}
\Lip_{u \to y} = \frac{\norm{B}{(\infty,[\eta]^{-1}),(\infty)}
	\norm{C}{(\infty),(\infty,[\eta]^{-1})} }
{1-\mu_{\infty,[\eta]^{-1}}(A)_+} + \norm{D}{\infty,\infty}.
\end{equation*}
When input data, $U$, is perturbed, the perturbation is directly fed into the output $Y$ via the output equation
$Y = CX + DU$. For this reason, a simple change to attempt to minimize the effect of input perturbations on the output is to replace the $DU$ term in the output equation by a static bias, i.e.,
\begin{equation*}
Y = CX + b\vectorones[m]^\top,
\end{equation*}
where $b \in \real^q$. This simple modification to the model changes the input-output Lipschitz constant to 
\begin{equation*}
\Lip_{u \to y} = \frac{\norm{B}{(\infty,[\eta]^{-1}),(\infty)}
	\norm{C}{(\infty),(\infty,[\eta]^{-1})} }
{1-\mu_{\infty,[\eta]^{-1}}(A)_+}.
\end{equation*}

Finally, another degree of freedom is the parameter $\gamma<1$ in the
constraint $\mu_{\infty,[\eta]^{-1}}(A) \leq \gamma$. In all
previously shown experiments on MNIST, we selected $\gamma =
0.95$. From the expression for the input-output Lipschitz
constant of the network~\eqref{eq:InputOutputLipschitz},
$\mu_{\infty,[\eta]^{-1}}(A) = 0.95$ leads to a small denominator,
resulting in a relatively large input-output Lipschitz constant. A simple modification to moderate the Lipschitz constant is to impose $\mu_{\infty,[\eta]^{-1}}(A) \leq \epsilon$ for some small $\epsilon \geq 0$. This attempts to maximize the denominator in the expression for the Lipschitz constant. 

\begin{figure}[ht]
	\centering
	\includegraphics[width=0.75\linewidth]{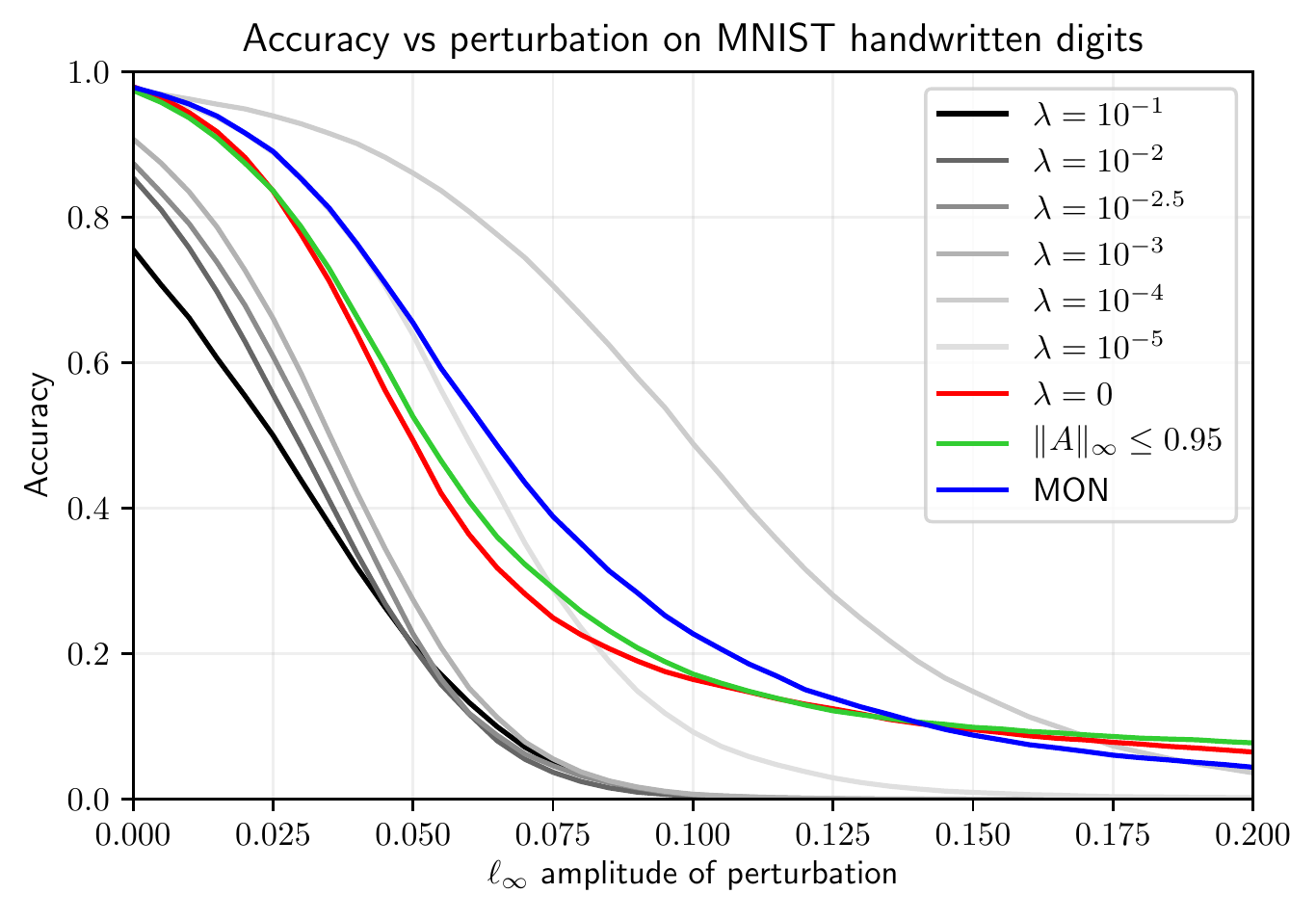}
	\caption{Plot of accuracy versus $\ell_\infty$ perturbation as
          generated by the FGSM for the NEMON model with $\mu_\infty(A) \leq
		0.05$, the implicit deep learning model with $\|A\|_\infty \leq
		0.95$, and MON with
		$I_n-\frac{1}{2}(A+A^{\top})\succeq 0.05I_n$. The output equation is $Y = CX + b\vectorones[m]^\top$.}\label{fig:bias}
\end{figure}

For these modifications to the models, plots of accuracy versus
$\ell_{\infty}$-perturbation generated by the FGSM are shown in
Figure~\ref{fig:bias}. In this figure, we set $\epsilon =0.05$ for
the NEMON models. For comparison, the well-posedness condition for MON is
set to be $\mu_2(A) \leq \epsilon$. We do not modify the condition $\|A\|_\infty \leq 0.95$ as imposing the constraint $\|A\|_\infty \leq \epsilon$ is overly restrictive and would result in a significant drop in accuracy. 

% \subsection{Comparison of robustness of implicit network models with
% respect to different attack models}

\subsection{Robustness of implicit neural networks on the MNIST dataset}

In this section, we compare the performance of the NEMON model with $\mu_\infty(A) \leq 0.95$
to the implicit deep learning model with $\|A\|_\infty \leq 0.95$ and
to the monotone operator equilibrium network (MON) with $I_n -
\frac{1}{2}(A+A^{\top}) \succeq 0.05I_n$
with respect to the attacks described in the previous subsection on
the MNIST dataset. 

For the continuous image inversion attack, Figure~\ref{fig:robustness}
shows the curves for accuracy versus $\ell_\infty$-amplitude of the
perturbation. We observe that, compared to the NEMON model, the implicit
deep learning model and the monotone operator equilibrium network (MON) have
larger drops in accuracy for small perturbations. For the NEMON model, as
$\lambda$ increases, the accuracy at zero perturbation
decreases. However, as $\lambda$ increases, the overall robustness of
NEMON improves as its accuracy does not decrease substantially
even for large amplitudes of perturbation.

For uniform additive $\ell_\infty$-noise, scatter plots with accuracy
versus $\ell_\infty$ amplitude of the perturbation are shown in
Figure~\ref{fig:uniformNoise}. We see that the NEMON model with
$\lambda = 0$, the implicit deep learning model, and the MON model all
perform comparably. The NEMON models with
$\lambda = 10^{-1}$ and $\lambda = 10^{-2.5}$ both see improved
robustness as their accuracy does not drop as noticeably with
$\ell_\infty$ amplitude of the perturbation. Surprisingly, the NEMON model
with $\lambda = 10^{-5}$ seems to be less robust than the NEMON model with
$\lambda = 0$.
 
For the FGSM, Figure~\ref{fig:FGSM} shows the curves for accuracy
versus $\ell_\infty$ amplitude of the perturbation. We see that the NEMON
models with $\lambda = 10^{-5}$ and $\lambda = 10^{-4}$ are the least
robust, followed by the NEMON model with $\lambda = 0$ and the MON. Only for
$\lambda \in \{10^{-2.5}, 10^{-2}, 10^{-1}\}$ do we see an improvement
in robustness for the NEMON model at the price of a decrease in nominal
accuracy. Note that for the FGSM, each model experiences different
perturbations.

For the PGDM, Figure~\ref{fig:PGDM} shows the curves for accuracy
versus $\ell_\infty$ amplitude of the perturbation. We see that the
results are comparable with the perturbation generated by the FGSM,
with the exception that the implicit deep learning model now performs
comparably with the monotone operator equilibrium model. Note that for
the PGDM, each model experiences different perturbations.

Finally, we compare the performance of the models with the
modification that the output equation is
$Y = CX + b\vectorones[m]^\top$. Figure~\ref{fig:bias} shows the
curves for accuracy versus $\ell_\infty$ amplitude of the FGSM
perturbation for the NEMON model with $\mu_\infty(A) \leq 0.05$, the
implicit deep learning model with $\|A\|_\infty \leq 0.95$, and the
monotone operator equilibrium model with $I_n -\frac{1}{2}(A+A^{\top})\succeq 0.05I_n$. For these modifications in the models, we see improvement in
overall accuracy compared to original models of implicit
networks~\eqref{eq:inn} shown in Figure~\ref{fig:FGSM}. Additionally, we observe comparable
performance in the NEMON model with $\lambda = 0$ and the implicit deep
learning model, with the MON performing slightly
better than both.  For the NEMON model with $\lambda = 10^{-4}$, the accuracy at zero
perturbation is comparable to the NEMON model with $\lambda = 0$ and the
overall robustness of the NEMON model to the FGSM attack is significantly 
improved. However, as $\lambda$ increases, we see that the nominal
accuracy and overall robustness of the NEMON models deteriorate.

\subsection{Robustness of implicit neural networks on the CIFAR-10
  dataset}

{\color{black}In this section, we compare the performance of the NEMON model with $\mu_\infty(A) \leq 0$
to MON with $I_n -
\frac{1}{2}(A+A^{\top}) \succeq I_n$
with respect to the FGSM attack described in the previous subsection on
the CIFAR-10 dataset.}

\begin{figure}[ht]
	\centering
	\includegraphics[width=0.75\linewidth]{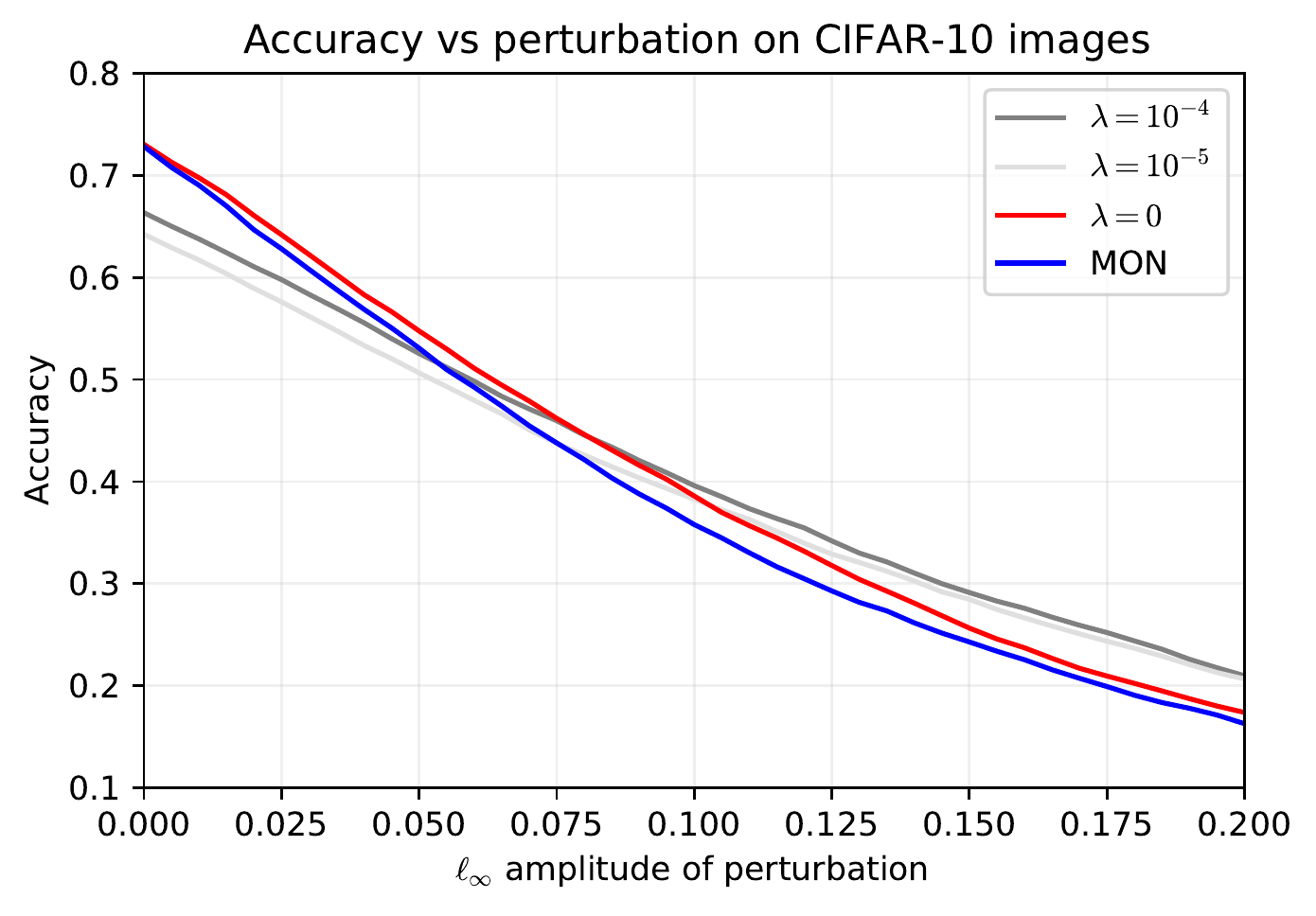}
	\caption{Plot of accuracy versus $\ell_\infty$ perturbation as
          generated by the FGSM for the NEMON model with
          $\mu_\infty(A) \leq 0$ and MON with
          $I_n-\frac{1}{2}(A+A^{\top})\succeq I_n$. }\label{fig:CIFAR-10-FGSMGS}
      \end{figure}

      {\color{black}For the FGSM attack on the CIFAR-10 dataset,
        Figure~\ref{fig:CIFAR-10-FGSMGS} shows the accuracy versus the
        $\ell_{\infty}$ amplitude of the perturbation for the regularized and
        un-regularized NEMON model and the MON model. We observe that
        un-regularized NEMON model is more accurate than MON for all amplitudes of perturbation. For
        example, at $\ell_{\infty}$-perturbation equal to $0.1$, the
        accuracy of un-regularized NEMON is $39 \%$ whereas the accuracy of MON
        at this attack amplitude is $35 \%$. Moreover, the regularized NEMON
        with the regularization parameter $\lambda =10^{-4}$ has a clean
        performance accuracy of $66 \% $ which is lower than the clean
        accuracy of both MON and the un-regularized NEMON. However,
        the regularized NEMON demonstrates a consistent improvement in
        accuracy for sizeable $\ell_{\infty}$-perturbations. For
        example, at an $\ell_{\infty}$-perturbation equal to $0.15$,
        the accuracy of the regularized NEMON model is $29 \%$ whereas the
        accuracy of MON at this attack amplitude is $24 \%$.}
            
 %. For the PGDM attack, we get comparable results to the FGSM attack with our un-regularized model outperforming MON for all perturbation and, for sizeable perturbations, our regularized model outperforming both MON and the un-regularized model.

\clearpage

\end{document}